\documentclass[11pt]{article}

\newif\ifpreprint   \preprinttrue
\newif\ifbiblatex   \biblatexfalse

\usepackage{amsthm,amsmath,amssymb,mathtools}
\usepackage{graphicx,booktabs}
\usepackage[dvipsnames]{xcolor}
\usepackage{microtype}

\ifpreprint
  \usepackage[margin=1in]{geometry}
  \usepackage{authblk}
  \usepackage{newtxtext,newtxmath}   
\else
  \providecommand{\affil}[2][]{}     
\fi

\ifbiblatex
\else
  \usepackage{natbib}
\fi

\usepackage{bm}
\usepackage{mathrsfs}
\usepackage{tikz}
\usepackage{tabularx}

\newcommand{\dv}[3][]{\ifx\\#1\\\frac{\mathrm{d}#2}{\mathrm{d}#3}\else\frac{\mathrm{d}^{#1}#2}{\mathrm{d}#3^{#1}}\fi}

\newtheorem{theorem}{Theorem}[section]
\newtheorem{lemma}[theorem]{Lemma}
\newtheorem{proposition}[theorem]{Proposition}
\newtheorem{corollary}[theorem]{Corollary}

\theoremstyle{definition}
\newtheorem{definition}[theorem]{Definition}
\newtheorem{example}[theorem]{Example}
\theoremstyle{remark}
\newtheorem{remark}[theorem]{Remark}

\DeclareMathOperator{\E}{\mathbb{E}}

\DeclareMathOperator*{\argmax}{argmax}
\DeclareMathOperator*{\argmin}{argmin}

\newcommand{\R}{\mathbb{R}}

\newcommand{\cut}[1]{}
\newcommand{\rf}{\mathsf{ref}}

\newcommand{\tv}{\mathsf{TV}}
\newcommand{\softmax}{\mathsf{s}}
\newcommand{\kl}{D_{\mathsf{KL}}}

\newcommand{\nl}{\mathsf{NL}}
\newcommand{\rl}{\mathsf{RL}}

\newcommand{\mc}{\mathsf{MC}}
\newcommand{\MP}{\mathsf{P}}

\DeclareMathOperator*{\essinf}{essinf}
\DeclareMathOperator*{\esssup}{esssup}
\DeclareMathOperator*{\essosc}{essosc}

\ifpreprint
  \usepackage[colorlinks=true,linkcolor=Maroon,
    urlcolor=MidnightBlue,citecolor=MidnightBlue]{hyperref}
\else
  \usepackage[colorlinks=false]{hyperref}
\fi

\usepackage[capitalise,noabbrev]{cleveref}
\Crefname{assumption}{Assumption}{Assumptions}

\title{A Markov Chain Approach to Preference Alignment}
\author{Takuya Koriyama}
\author{Tengyuan Liang}
\affil{The University of Chicago}
\date{\today}

\begin{document}

\maketitle
\begin{abstract}
We propose Markov Chain from Human Feedback (MCHF), an elementary approach for aligning generative models from pairwise human preferences. Unlike Reinforcement Learning from Human Feedback (RLHF), which reduces comparisons to a scalar reward, and Nash Learning from Human Feedback (NLHF), which preserves pairwise utilities through a KL-regularized minimax optimization, MCHF uses pairwise preferences directly to define a transition mechanism over model outputs. Given a pairwise utility $U(x,y)$, which quantifies human preference for $y$ over $x$,  and a reference probability distribution $\mu_{\mathsf{ref}}$, we define a Markov kernel $\mathsf{P}(x, dy)\propto \exp(U(x,y))\mu_{\mathsf{ref}}(dy)$, and take the Markov chain starting from $\mu_{\mathsf{ref}}$ as an iterative alignment procedure. We show that MCHF converges geometrically fast to the stationary distribution, with a convergence rate governed by the seminorm $\|U\|_\oplus=\inf_{g,f\in L^\infty(\mu_{\mathsf{ref}})}\|U-g\oplus f\|_\infty$, which quantifies the non-transitive structure of the pairwise utility. We further show that a mirror-descent algorithm for NLHF satisfies an analogous structure-adaptive convergence guarantee. Finally, through a perturbation analysis, we prove that when $\|U\|_\oplus$ is small, MCHF and NLHF agree up to first order around an RLHF solution, which yields a unified view of reward-based, game-theoretic, and Markovian approaches to alignment. 
In particular, for two natural algorithms that converge to the MCHF/NLHF equilibria, we show that the first step of MCHF and NLHF recovers the RLHF solution based on the column-sum reward $\hat{f}(y)=\int \mu_{\mathsf{ref}}(dx) U(x, y)$, and starting from the second iteration, both algorithms incorporate the same linear functional of the residual $U-(-\hat f)\oplus \hat f$, which captures the non-transitive structure of the pairwise utility $U$.
\end{abstract}

\tableofcontents

\section{Introduction}
Aligning AI systems with human preferences has become a central problem in modern machine learning, as large generative models are increasingly deployed in practical applications \citep{openai2023gpt4,geminiteam2025gemini,anthropic2024claude}. 
In many alignment pipelines, human feedback is provided in the form of pairwise comparisons: given two candidate outputs, a human evaluator indicates which one is preferred. Such preference data are often easier to collect and prove more reliable than absolute numerical scores, as argued in several widely used alignment methods. 

\subsection{Reinforcement Learning from Human Feedback (RLHF)}
In Reinforcement Learning from Human Feedback (RLHF), 
pairwise comparisons are typically used to fit a scalar reward model, often through a probabilistic preference model such as the Bradley--Terry or Luce (BTL) model \citep{bradley1952rank,luce1959individual}:
   \begin{align}\label{eq:assumption_bt}
      \mathcal{P}(x \prec y) =  \sigma_{\text{sigmoid}}(R(y)-R(x)) \quad \text{where} \quad  \sigma_{\text{sigmoid}}(t) \equiv (1+e^{-t})^{-1}. 
  \end{align}
  More precisely, given samples of pairwise comparisons $\{ x_i \prec y_i\}_{i=1}^n$, assuming that they are drawn from the BTL model, one estimates the reward function $R$ by maximum likelihood estimation. 
 The policy is then optimized against the learned reward $\hat{R}$, usually with a KL regularization term that keeps it close to a reference model $\mu_\rf$ \citep{christiano2017deep,ouyang2022training,bai2022training}:
   \begin{align*}
      {\mu}_{\rl} = \argmax_{\mu} \Bigl(\int \mu(dx)\hat{R}(x) - \kl(\mu \mid \mu_{\rf})\Bigr)  \quad \text{or} \quad \mu_\rl(dy) \propto \exp(\hat{R}(y)) \mu_\rf(dy)
  \end{align*}
  In practice, this problem is often solved approximately using reinforcement-learning algorithms such as proximal policy optimization (PPO) \citep{schulman2017proximal}. More recently, several variants and alternatives to the standard RLHF pipeline have been proposed; see \cite{rafailov2023direct, ethayarajh2024kto, meng2024simpo, azar2024general, huang2024correcting}, among others. 
  

Although a variety of RLHF-style methods have been proposed, their ultimate goal remains to aggregate pairwise preference information to a single scalar reward and then tilt the reference distribution by the reward. At its core, the reward modeling relies on the strong assumption that human preference is determined by the single reward, as in the BTL model assumption \eqref{eq:assumption_bt}, which imposes a strong structural restriction on the preference structure $\mathcal{P}(x\prec y)$. In particular, such a model cannot capture genuinely cyclic preference patterns such as
$\mathcal{P}(x\prec y)>1/2, \mathcal{P}(y\prec z)>1/2,  \mathcal{P}(z\prec x)>1/2$. 
Furthermore, RLHF-style methods assume the existence of a single reward function representing aggregated human preference. In practice, preferences may vary across human annotators, and therefore cannot always be faithfully summarized by a homogeneous reward model. This limitation has motivated pluralistic alignment in the community, which tries to account for heterogeneous preferences explicitly \citep{sorensen2024roadmap}.

\subsection{Nash Learning from Human Feedback (NLHF): a Game-Theoretic Approach}
There have been a variety of attempts at incorporating preference information directly, without introducing an explicit single reward model. Among them, one influential example is the Nash Learning from Human Feedback (NLHF) \citep{munos2024nash}. 
The key idea is to first learn the preference model $(x,y)\mapsto \mathcal{P}(x \prec y)$ directly from a dataset of comparisons $\{x_i\prec y_i\}_{i=1}^n$, for instance by solving
$$\max_{\theta} \sum_{i=1}^n \log \mathcal{P}_\theta(x_i \prec y_i)$$ 
where $\{\mathcal{P}_\theta:\mathcal{X}\times \mathcal{X}\to [0,1]\}$ is a parametric class, typically represented by $\mathcal{P}_\theta(x\prec y) = \text{sigmoid}(U_\theta(x, y))$ (or more generally $\mathcal{P}_\theta = \Phi\circ U_\theta$ for an increasing link function $\Phi$ mapping to $[0,1]$), where $U_\theta$ is a neural network. Notice that when we further parametrize $U_\theta$ in additive form,  $U_\theta(x,y) = -R_\theta(x) + R_\theta(y)$, then it is equivalent to the reward modelling by BTL model used in the RLHF pipeline. 

Given such a pairwise utility function $U=U_\theta$, NLHF defines the aligned distribution through the following minimax optimization problem
\begin{align}\label{eq:intro_nlhf}
{\mu}_{\nl} \in \argmax_{\mu} \min_{\nu} \iint \nu(dx)\mu(dy) U(x, y)-  \kl(\mu \mid \mu_{\rf})
+ \kl(\nu \mid \mu_{\rf}). 
\end{align}
Minimax formulations without the KL regularization terms have also been studied in several works \cite{wang2023rlhf, swamy2024minimaximalist, wu2024self}. More broadly, game-theoretic perspectives on alignment have led to formulations based on other solution concepts, such as Stackelberg games \citep{pasztor2026stackelberg,chu2025stackelberg} and multiplayer games \citep{wu2026multiplayer}.

The NLHF paper \citep{munos2024nash} argues that the Nash-equilibrium formulation is attractive because it can yield a more diverse distribution than reward-based alignment methods. A more refined analysis of this phenomenon has been developed through the lens of social choice theory (cf. \cite{golz2025distortion, shi2025fundamental}).

Nevertheless, because NLHF is formulated as a worst-case equilibrium against an adversarial opponent, it is intrinsically more pessimistic than reward-based alignment, and it is not always clear that such a pessimistic objective is the most appropriate one for fine-tuning generative models. Moreover, NLHF treats alignment as a simultaneous game, in which the two players, $\mu$ and $\nu$, optimize a regularized zero-sum objective at the same time. In contrast, preference refinement can be more naturally viewed as a non-simultaneous, leader-follower procedure: given a current output, one samples a new output conditionally on it, favoring alternatives that are preferred to the current one.

\begin{figure}[tb]
  \centering
  \includegraphics[width=0.75\linewidth]{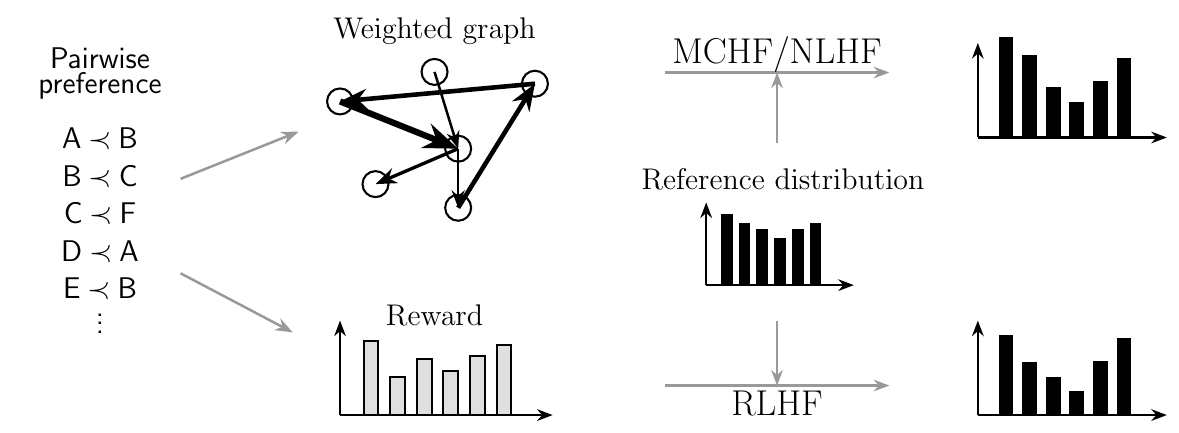}
  \caption{Workflow of alignment. We start from pairwise preference data, then we either construct a reward function (usually based on the BTL model) and run RLHF or construct a directed weighted graph and run MCHF (our proposal) or NLHF. 
  Here, unlike the scalar reward, the weighted graph can preserve non-transitive/pairwise interaction structure in human preferences, and therefore MCHF/NLHF provide more flexible alignment methods than RLHF.
  }
  \label{fig:workflow}
\end{figure}

\subsection{Our contribution: a Markov Chain Approach}
Motivated by this viewpoint, we propose Markov Chain from Human Feedback (MCHF), a Markovian alignment approach that uses pairwise preference utilities directly through such a non-simultaneous formulation. Concretely, given a utility $U(x,y)$, where larger values indicate that $y$ is preferred to $x$, and a reference model $\mu_\rf$, we define the Markov kernel $\MP$ by
$$
\MP(x,dy)=\frac{\exp(U(x,y))}{Z(x)}\mu_\rf(dy),
$$
where $Z(x)$ is the normalizing constant. Then, given a current distribution $\mu$, one step of MCHF updates it according to 
$$
\mu_{\text{new}}(dy) = \mu\MP(dy) = \int \mu(dx)\MP(x, dy).
$$
Here, the kernel $\MP$ favors transitions from a current output $x$ to outputs $y$ that are preferred relative to $x$, while the reference distribution $\mu_{\rf}$ preserves the structure, especially the support, of the reference distribution $\mu_\rf$.
Thus, in contrast to the simultaneous-game interpretation underlying NLHF, MCHF implements a non-simultaneous/leader-follower alignment procedure via the conditional sampling from the Markov kernel $\MP$.

This one-step update can also be iterated, starting from $\mu_\rf$:
$$
\mu_{\rf} \mapsto \mu_{\rf}\MP \mapsto \mu_{\rf}\MP^2
\mapsto \cdots \to \mu_\rf\MP^\infty = \mu_\mc
$$
Then, MCHF induces an interpolation path from the reference distribution to the stationary distribution $\mu_\mc$, which fully incorporates the pairwise preference utility $U$. This iterative alignment procedure
is inspired by classical spectral rank-aggregation methods \citep{page1999pagerank,negahban2012rank}, where one constructs a directed graph from pairwise comparisons and represents an aggregate ranking as the stationary distribution of a Markov chain on the graph (see \Cref{fig:workflow}).

Beyond introducing MCHF, our main theoretical contribution is to clarify its relationship with existing alignment methods such as NLHF and RLHF. 
A key observation is that, when the pairwise utility is written by an additive form 
$$
U = g\oplus f \quad \text{with} \quad (g\oplus f)(x, y)=g(x) + f(y)
$$
for some functions $g$ and $f$, 
then the MCHF and NLHF equilibria based on such $U$ both reduce to the RLHF solution with reward $f$:
$$
\mu_\rl(f) (dy) \propto \exp(f(y))\mu_\rf(dy).
$$
This motivates us to compare MCHF and NLHF via a common perturbation framework: what are their first-order behaviors when $U$ is close to an additive pairwise utility $g\oplus f$?

We address this question by introducing a seminorm $\|\cdot\|_\oplus$ on $L^\infty(\mu_\rf\otimes \mu_\rf)$:
\begin{align}\label{eq:seminorm_intro}
  \forall U\in L^\infty(\mu_\rf\otimes\mu_\rf), \quad \|U\|_\oplus \equiv \inf_{g, f\in L^\infty(\mu_\rf)} \|U-g\oplus f\|_\infty,
\end{align}
which quantifies the non-additive or non-transitive structure of the pairwise preference utility $U$. 
We then show that, when $\|U\|_\oplus$ is small, RLHF provides a natural baseline corresponding to the transitive component of the pairwise utility, and around this baseline, NLHF and MCHF exhibit the same first-order structure, despite arising from different motivations: NLHF from a game-theoretic equilibrium and MCHF from a stationary distribution of a preference-driven Markov chain. This allows us to compare the three methods, RLHF, NLHF, and MCHF, within a common perturbative framework and to quantify how their differences depend on the non-transitive component of the preference utility $U$. In this sense, we do not merely propose MCHF as an additional alignment algorithm, but also provide a unified view of reward-based, game-theoretic, and our Markovian approaches to preference alignment.
Below, we summarize technical contributions of this paper.

\paragraph{Structure-Adaptive Convergence for MCHF and NLHF}
For MCHF, we prove the geometric convergence to the stationary distribution $\mu_\mc$: 
\begin{align*}
  d_\tv(\mu_\rf\MP^t, \mu_\mc)\le d_\tv(\mu_\rf, \mu_\mc) \cdot c(\|U\|_\oplus)^t \quad \text{with} \quad  c(x) \equiv \min(1-e^{-2x}, x) \in [0,1), 
\end{align*}
where $\|U\|_\oplus$ is the seminorm defined by \eqref{eq:seminorm_intro}. By the definition of the seminorm, 
we have $\|U\|_\oplus\le \|U\|_\infty$, but importantly, for natural choices of utilities and underlying human preference structures, it holds that 
  $$
    \|U\|_\oplus \ll \|U\|_\infty
  $$
  (see Example \ref{example:logit_transform}-\ref{example:linear_transform}). 
  Since the contraction rate $c=c(\|U\|_\oplus)$ is an increasing function of $\|U\|_\oplus$, and the construction of our Markov kernel $\MP$ does not explicitly calculate the seminorm $\|U\|_\oplus$, 
this suggests that MCHF \emph{implicitly} adapts to the additive/transitive structure of $U$.

  For NLHF, we also show that a mirror descent algorithm converges exponentially fast to the NLHF equilibrium, with its convergence rate governed by $\|U\|_\oplus$ (not $\|U\|_\infty$), if the step-size is appropriately chosen. To the best of our knowledge, our theoretical guarantee provides a sharper convergence guarantee than prior works (see \cite{munos2024nash, rosset2024direct, tiapkin2025accelerating} and references therein), whose convergence rates usually depend on the naive norm $\|U\|_\infty$. 

\paragraph{Comparison between MCHF and NLHF.}
We next compare the iterative alignment dynamics induced by MCHF and NLHF from three perspectives: (1) computation, (2) coupling,  and (3) their asymptotic behavior when $\|U\|_\oplus$ is small. The third perspective is particularly useful for understanding what these dynamics learn in their early iterations.
Suppose $U\in L^\infty(\mu_\rf\otimes \mu_\rf)$ is an antisymmetric utility function, and consider the following two natural iterations to compute MCHF and NLHF equilibria:
$$
    \mu_\mc^t = \mu_\mc^{t-1}\MP, \qquad 
    \mu_\nl^t
    =
    \argmax_{\mu}
    \left\{
      \int \mu_\nl^{t-1}(dx)\mu(dy) U(x,y)
      -
      \kl(\mu|\mu_\rf)
    \right\},
$$
initialized at $\mu_\mc^0=\mu_\nl^0=\mu_\rf$. 
 We then show that they have the same first-order alignment dynamics in the small-$\|U\|_\oplus$ regime.
More precisely, for each $t\ge1$, let
$$
p_*^t(U)(y)=(d\mu_*^t/d\mu_\rf)(y)
$$ 
be the density for $*\in \{\mc, \nl\}$ (we highlight the dependence on $U$), and  
let $p_\rl(\hat{f})$ denote the RLHF density with reward given by the weighted column-sum of $U$:
$$
p_\rl(\hat{f})(y) = \frac{\exp(\hat{f}(y))}{\int\mu_\rf(dy') \exp(\hat{f}(y'))}, 
\qquad
\hat{f}(y) = \int\mu_\rf(dx) U(x, y).
$$
Then, we show that, for each $*\in \{\mc, \nl\}$, as $\|U\|_\oplus \to 0$, 
\begin{align*}
  \bigl\|p_*^1(U)-p_\rl(\hat{f})\bigr\|_{L^1(\mu_\rf)} &= o(\|U\|_\oplus),\\
  \sup_{t\ge 2} \ \Bigl\|p_*^t(U) - p_\rl(\hat{f})- {p_\rl(\hat{f})\odot\bigl[\bigl\{U-(-\hat{f})\oplus \hat{f}\bigr\}^\star p_\rl(\hat{f})\bigr]}\Bigr\|_{L^1(\mu_\rf)} &= o(\|U\|_\oplus), 
\end{align*}
where $\odot$ is the entrywise product, and $^\star$ is the adjoint operator defined as 
 $K^\star p(\cdot)=\int \mu_\rf(dx) K(x, \cdot)p(x)$.

This result gives a simple interpretation of the early-stage dynamics. The first iteration is asymptotically equivalent to RLHF with reward $\hat f$, which corresponds to the average preference for each response under the reference distribution. Starting from the second iteration, both dynamics incorporate the linear functional of the residual
$U-(-\hat f)\oplus \hat f$, 
which captures the non-transitive component of the preference. Moreover, this linear functional is independent of $t$, suggesting that the essential first-order alignment effect is already captured within the first two iterations. This provides a useful practical insight: when iterative alignment is computationally expensive, especially for large-scale models, the main first-order benefit may already be obtained after two alignment steps, as long as the non-transitive component of the utility $U$ measured by $\|U\|_\oplus$ is small.

\begin{figure}[htbp]
  \centering
  \includegraphics[width=0.7\linewidth]{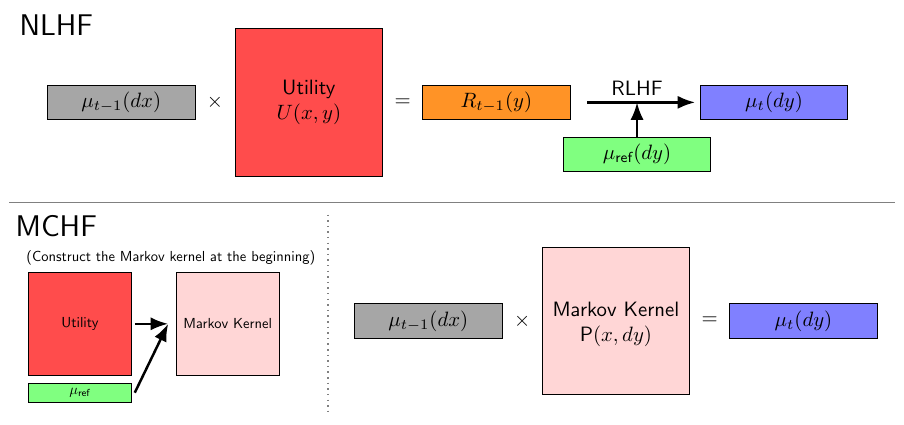}
  \caption{Comparison between MCHF iteration and NLHF iteration for antisymmetric $U$.}
  \label{fig:mchf_nlhf_update_compare}
\end{figure}

\paragraph{Notation}
Let $(\mathcal{X}, \mathcal{F}, \mu_{\rf})$ be a probability space, and let $\mathscr{P}_{\mu_\rf}$ denote the set of probability measures on $\mathcal{X}$ that are absolutely continuous with respect to $\mu_{\rf}$. We write $L^\infty = L^\infty(\mu_{\rf}\otimes \mu_{\rf})$ or $L^\infty = L^\infty(\mu_{\rf})$ when the meaning is clear from context, and similarly write $\|X\|_p = \|X\|_{L^p(\mu_{\rf})}$ for the $L^p$ norm. For any $\mathcal{F}$-measurable function $f$, we denote by $\esssup_x f(x)$ the essential supremum of $f(x)$ with respect to $\mu_\rf$, i.e., $\esssup_x f(x)=\inf\{M: \mu_{\rf}(f(x)>M) = 0\}$. For any multivariate function $F(x_1, \dots, x_p)$, 
$\esssup_{x_1, \dots x_p} F$ is with respect to the product measure $\mu_\rf \otimes \dots \mu_\rf$. {A more comprehensive list of notations is provided at the beginning of the appendix.}

\section{Definition of seminorm: measuring non-additive structure of utility}\label{sec:projection_utility}
Let $U(x, y)$ be a pairwise utility function
that encodes preference information as follows:
\begin{align*}
  \text{$U(x,y)$ is large}
  \quad \Longleftrightarrow \quad
  \text{``a human is more likely to prefer $y$ to $x$.''}
\end{align*}
Throughout this paper, we assume $U \in L^\infty(\mu_\rf\otimes \mu_\rf)$, i.e., its $L^\infty$ norm with respect to the product measure is bounded. 

We now define a seminorm on $L^\infty(\mu_\rf\otimes \mu_\rf)$, which serves as a fundamental complexity measure for the alignment dynamics of MCHF and NLHF throughout this paper.
\begin{definition}
Define the seminorm $\|\cdot\|_\oplus$ on $L^\infty(\mu_\rf\otimes \mu_\rf)$ by
  $$
  \|U\|_\oplus
  =
  \inf_{g,f\in L^\infty(\mu_\rf)}
  \|U-g\oplus f\|_\infty,
  $$
  where $(g\oplus f)(x, y)=g(x)+f(y)$. 
\end{definition}
This seminorm $\|U\|_\oplus$ captures the \textit{additive defect} of the utility; the larger $\|U\|_\oplus$ is, the more $U$ has a non-additive structure. 
We can also view this seminorm as the $L^\infty$-distance to the class of additive
functions $\{g\oplus f:g,f\in L^\infty(\mu_\rf)\}$. If the state space is finite, a similar yet 
more refined decomposition of pairwise comparison data, allowing for missing
entries, has been studied in \citet{jiang2011statistical}.

Note that $\|U\|_\oplus \le \|U\|_\infty$ by the definition.
In practice, the utility is often taken to be a possibly nonlinear transformation
of the preference model $\mathcal{P}(x\prec y)$, so the value of
$\|U\|_\oplus$ depends both on the nonlinear transformation and on the structure
of $\mathcal{P}(x\prec y)$. Importantly, the utility function need not be exactly
additive, but there are natural situations in which it is approximately additive.

\begin{example}[Logit transform]\label{example:logit_transform}
Suppose $U$ is the logit transform of the preference probability:
$$
U(x, y)=\sigma_{\text{sigmoid}}^{-1}\circ \mathcal{P}(x\prec y) \quad \text{where}\quad \sigma_{\mathrm{sigmoid}}^{-1}(t)=\log \frac{t}{1-t}
$$
and the preference model $\mathcal{P}(x\prec y)$ follows
$$
\mathcal{P}(x\prec y)= \Phi(R(y)-R(x)+E(x,y)),
$$
where $\Phi$ is a nonlinear link function and $E(x,y)$ is a pairwise interaction
term. In particular, if $\mathcal{P}(x\prec y)$ follows the BTL model (i.e., 
$\Phi=\sigma_{\mathrm{sigmoid}}$ and $E=0$), we have
$U(x,y)=R(y)-R(x)$ and hence $\|U\|_\oplus=0$. 
More generally, whenever $\Phi$ is close to $\sigma_{\mathrm{sigmoid}}$ and the pairwise
interaction term $E$ is small,  $U(x,y)$ is close to $R(y)-R(x)$, and hence $\|U\|_\oplus$ is small (see \Cref{fig:visualize_utility}).
\end{example}

\begin{figure}
  \centering
  \includegraphics[width=0.9\linewidth]{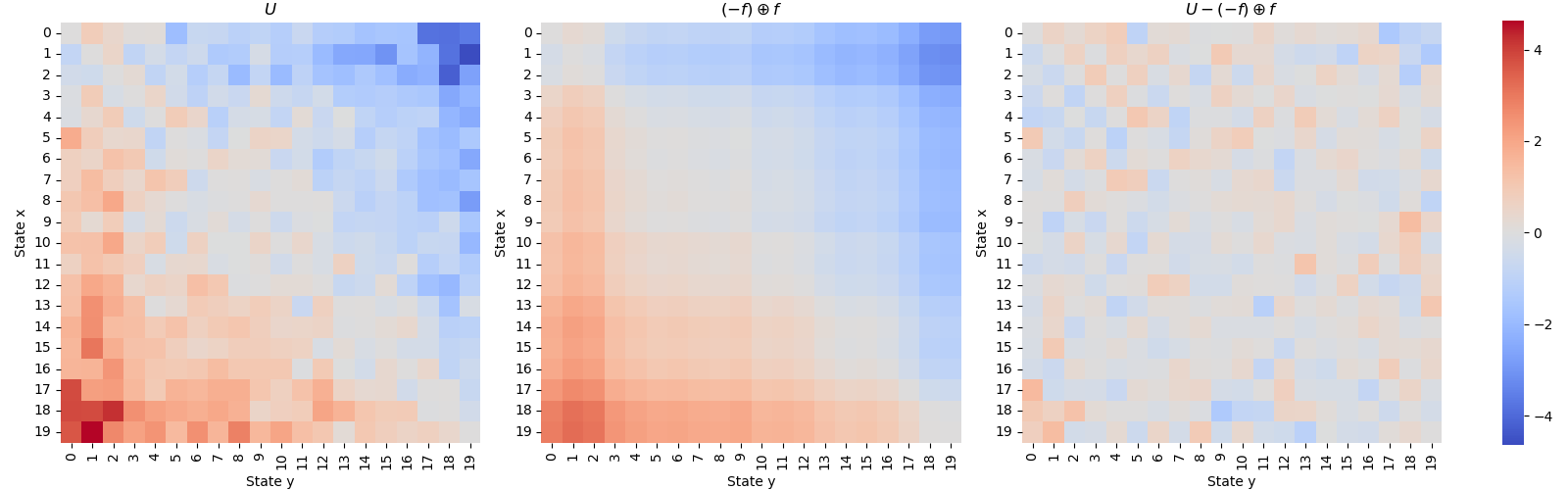}
  \caption{Comparison of antisymmetric utility $U$ generated by \eqref{eq:utility_DGP}, $(-\hat{f})\oplus \hat{f}$ with $\hat{f}(y)=\int\mu_\rf(dx)U(x, y)$, and residual $U-(-\hat{f})\oplus \hat{f}$. }
  \label{fig:visualize_utility}
\end{figure}

\begin{example}(Linear transform)\label{example:linear_transform}
Suppose 
$$
U(x,y)=\mathcal{P}(x\prec y)-1/2,
$$ 
where $\mathcal{P}(x\prec y)$ follows the BTL model $\mathcal{P}(x\prec y)=\sigma_{\mathrm{sigmoid}}(R(y)-R(x))$ for a reward $R$, and assume that $\|R\|_\infty$ is small. Then, using the linear approximation of the sigmoid, 
$\sigma_{\mathrm{sigmoid}}(x)=\frac{1}{2}+\frac{x}{4}+O(x^3)$
as $x\to 0$, one can show
\begin{align*}
  \|U\|_\oplus
=
O(\|(-R)\oplus R\|_\infty^3) \quad \text{and} \quad
\|U\|_\infty = 4^{-1} \|(-R)\oplus R\|_\infty + O(\|(-R)\oplus R\|_\infty^3)
\end{align*}
so that the ratio is significantly small: $\|U\|_\oplus/\|U\|_\infty=O(\|(-R)\oplus R\|_\infty^2) = O(\|R\|_\infty^2)\to 0$. 
\end{example}

We will see later in Theorem \ref{theorem:contraction} that MCHF implicitly adapts to the additive structure embedded in the pairwise utility $U$ in the sense that the convergence rate is bounded by an increasing function of $\|U\|_\oplus$.
 In contrast, we also provide a mirror-descent algorithm for NLHF whose convergence guarantee is governed by $\|U\|_\oplus$. However, achieving this adaptive convergence rate requires choosing a stepsize $\eta$ as $\eta=1/\|U\|_\oplus^2$, which requires the computation of $\|U\|_\oplus$.

Note that the exact computation of the seminorm $\|U\|_\oplus$ can be challenging, especially when the state space is non-discrete or when it is discrete but very large, since the problem is essentially an $L^\infty$ linear program.
However, we can estimate $\|U\|_\oplus$ up to a multiplicative constants in terms of some explicit quantities. 

\begin{proposition}\label{proposition:L2_estimate}
For any $U\in L^\infty(\mu_\rf\otimes \mu_\rf)$, let $\square(U)$ be the rectangle defect defined as 
  $
\square(U)
=
\esssup_{x,x',y,y'}
\left|U(x',y')-U(x,y')-U(x',y)+U(x,y)\right|,
$
and let $(\hat g,\hat f)$ be any solution to  the $L^2$ minimization $\min_{g,f}\|U-g\oplus f\|_2$. Then
$$
\square(U)
\asymp
\|U\|_\oplus
\asymp
\|U-\hat g\oplus \hat f\|_{\infty},
$$
or more precisely
$
\frac14 \square(U)
\le
\|U\|_\oplus
\le
\|U-\hat g\oplus \hat f\|_{\infty}
\le
\square(U).
$
\end{proposition}
We see that the rectangle defect $\square(U)$ is an explicit quantity by its definition, measuring the \textit{intransitivity} of the utility for any cycle consisting of $4$ points. The plug-in estimate $\|U-\hat g\oplus \hat f\|_\infty$ can also be computed explicitly, since 
 $(\hat{g}, \hat{f})$ is the minimizer of $\min_{g,f}\|U-g\oplus f\|_2$ and one of the solutions is given by the weighted row/column sums:
\begin{align}\label{eq:L2_min_solution}
  \hat g(x)
=
\int U(x,y)\mu_{\rf}(dy)-m,
\qquad
\hat f(y)
=
\int U(x,y)\mu_{\rf}(dx),
\end{align}
where
$m=\iint U(x,y)\mu_{\rf}(dx)\mu_{\rf}(dy)$. 
In particular, if $\|U\|_\oplus=0$, then we have $U = \hat{g} \oplus \hat f$. 

Finally, when $U$ is antisymmetric, it holds that $\hat{g}=-\hat{f}$, and we obtain a sharper characterization of the seminorm $\|U\|_\oplus$. 
\begin{proposition}\label{proposition:L2_estimate_antisymmetric}
Suppose $U\in L^\infty(\mu_\rf\otimes\mu_\rf)$ is antisymmetric, i.e.,
$U(x,y)=-U(y,x)$. Let $\Delta(U)$ be the triangle defect defined by $\Delta(U)=
  \esssup_{x,y,z}
  |U(x,y)+U(y,z)+U(z,x)|$ and let
$\hat f(y)
=\int \mu_\rf(dx) U(x,y)$ be the weighted column sum. 
Then, we have 
  $$
    \|U\|_\oplus
    \asymp
    \Delta(U)
    \asymp
    \|U-(-\hat f)\oplus \hat f\|_\infty,
  $$
or more precisely, 
$\frac13\Delta(U)
\le
\|U\|_\oplus
\le
\|U-(-\hat f)\oplus \hat f\|_\infty
\le
\Delta(U)$.
\end{proposition}
We see that $\Delta(U)$ measures the intransitivity of $U$ for any cycle consisting of $3$ points. Furthermore, if $\|U\|_\oplus=0$, then we have $U = (-\hat f)\oplus \hat f$ where $\hat{f}(y)=\int \mu_\rf(dx) U(x, y)$ is the weighted column sum. 

\section{Structure-Adaptive Convergence of MCHF and NLHF}
\subsection{Markov Chain from Human Feedback (MCHF)}\label{sec:mchf}
Based on the utility function $U\in L^\infty(\mu_{\rf}\otimes \mu_{\rf})$, we define the following Markov kernel $\MP$:
\begin{align}\label{eq:def_markov}
\MP(x, dy) = 
\frac{1}{Z(x)}{\exp ( U(x, y)) \mu_{\rf}(dy)}, \quad Z(x) \equiv \int \exp(U(x, y)) \mu_{\rf}(dy)
\end{align}
where $Z(x)$ is a normalizing constant, which is always strictly positive and bounded since $U\in L^\infty(\mu_{\rf}\otimes \mu_{\rf})$. 
 Note that the conditional distribution $Y|X \sim \MP(X, \cdot)$ will move mass towards preferred directions (where $U(x, \cdot)$ takes large values). Therefore, letting $\mu_{\rf}$ be some base reference measure, e.g., pre-trained generative model, the probability measure $\mu_{\star}$ induced by the one-step Markov chain:
 \begin{align*}
\mu_{\star} (dy) \equiv (\mu_{\rf} \MP)(dy) =  \int \mu_{\rf}(dx) \MP (x, dy)
\end{align*}
aligns $\mu_\rf$ with the preference information encoded by the utility function $U$.

Given the Markov kernel $\MP$ defined by \eqref{eq:def_markov}, let us consider the Markov chain
$\mu_\rf \mapsto \mu_\rf \MP \to \mu_\rf \MP^2 \cdots$
on the complete metric space $(\mathscr{P}_{\mu_\rf}, d_\tv)$ where $\mathscr{P}_{\mu_\rf}$ is the set of probability measures on $\mathcal{X}$ that are absolutely continuous with respect to $\mu_{\rf}$, and $d_\tv$ is the total variation distance.
Now we claim that the map $\mu \mapsto \mu \MP$ is a contraction mapping on this space.
\begin{theorem}\label{theorem:contraction}
For any $U\in L^\infty(\mu_\rf\otimes \mu_\rf)$, let $\MP$ be the Markov kernel defined by \eqref{eq:def_markov} with the utility $U$. Then, it holds that 
$$
\forall \mu, \nu \in \mathscr{P}_{\mu_\rf}, \quad d_{\tv}(\mu \MP, \nu \MP) \le  c \cdot d_{\tv}(\mu, \nu), 
$$
where the contraction rate $c$ is given by 
\begin{align}\label{eq:def_contraction}
  c=c(\|U\|_\oplus)=\min(1-e^{-2\|U\|_\oplus}, \|U\|_\oplus) \in [0,1). 
\end{align}
\end{theorem}

\begin{figure}[tpb]
  \centering
  \includegraphics[width=0.4\linewidth]{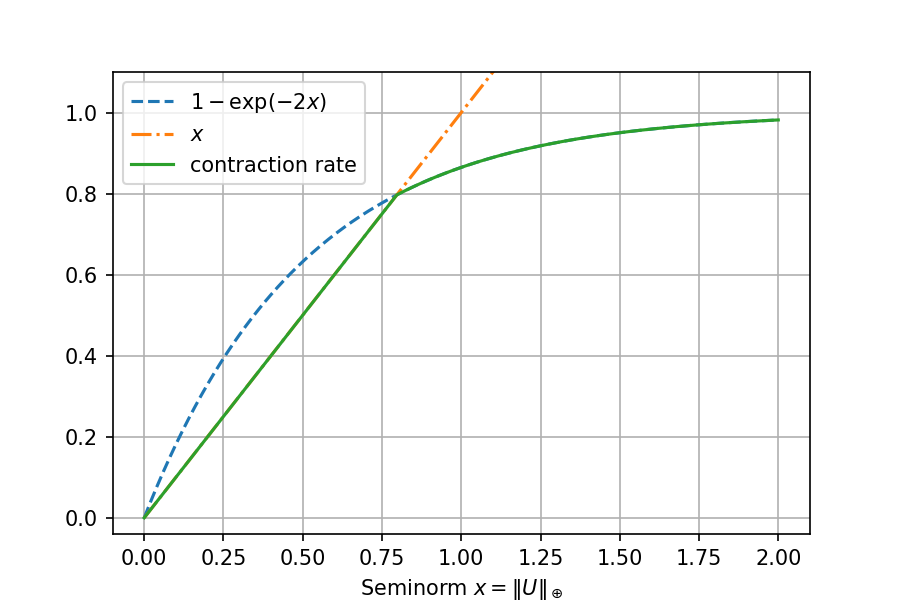}
  \caption{Plot of the contraction rate $c=\min(1-e^{-2\|U\|_\oplus}, \|U\|_\oplus)$ as a function of $\|U\|_\oplus$}
  \label{fig:compare_lipschitz_bound}
\end{figure}

Notice that the contraction rate $c(\|U\|_\oplus)$ is an increasing function of $\|U\|_\oplus$ taking values in $[0,1)$ (see Figure \ref{fig:compare_lipschitz_bound}). 
Theorem \ref{theorem:contraction} implies that the map $\mu \mapsto \mu \MP$ is a contraction mapping in the complete metric space $(\mathscr{P}_{\mu_\rf}, d_{\tv})$. Thus, by the Banach fixed-point theorem, we obtain the following result. 
\begin{corollary}\label{corollary:stationary_dist}
There exists a unique stationary distribution $\mu_\mc\in \mathscr{P}_{\mu_\rf}$ satisfying 
$\mu_\mc = \mu_\mc \MP$
and 
$$
d_{\tv}(\mu_\rf \MP^t, \mu_\mc)
\le
d_{\tv}(\mu_\rf, \mu_\mc)\cdot c^t,
$$
where $c$ is the contraction rate given by \eqref{eq:def_contraction}. 
\end{corollary}

Corollary \ref{corollary:stationary_dist} implies that the stationary distribution $\mu_\mc$ can be approximated by the $t$-step Markov chain $\mu_\rf\MP^t$ as $t\to+\infty$. 
Importantly, the contraction rate $c$ is bounded in terms of the seminorm $\|U\|_\oplus$, rather than the naive $L^\infty$-norm $\|U\|_\infty$. Since the construction of the Markov kernel and the operation $\mu\mapsto \mu\MP$ do not require prior knowledge of the value of $\|U\|_\oplus$, Corollary \ref{corollary:stationary_dist} suggests that MCHF \textit{implicitly} adapts to the additive structure of $U$.

\subsection{Nash Learning from Human Feedback (NLHF)}\label{sec:nlhf_algorithm}

We have seen in the previous section that MCHF implicitly adapts to the additive structure of $U$. In this section, we propose a mirror descent algorithm for NLHF; we show that if the step size is chosen carefully, the mirror descent enjoys the last-iterate convergence guarantee with a similar structure-adaptive convergence rate to MCHF. 

Before moving to the main result, let us introduce Nash Learning from Human Feedback (NLHF). For any $U\in L^\infty(\mu_\rf\otimes \mu_\rf)$, let $\mu_\nl$ be the solution to the following minimax optimization problem:
\begin{align}\label{eq:def_nlhf}
      {\mu}_{\nl} \in \argmax_{\mu} \min_{\nu} \Bigl(
\int \mu(dy)\nu(dx)U(x, y)
    - \kl(\mu \mid \mu_{\rf})
    + \kl(\nu \mid \mu_{\rf})
    \Bigr).
\end{align}
We refer to this solution $\mu_\nl$ as Nash Learning from Human Feedback (NLHF).
The original paper \cite{munos2024nash} corresponds to the choice
$U(x, y)=\lambda^{-1}(\mathcal{P}(x\prec y)-1/2)$, 
which is antisymmetric. The general case of an arbitrary utility $U(x,y)$ is discussed in \cite{shi2025fundamental} through the lens of social choice theory.
Thanks to the strong convexity induced by the KL regularization terms and the assumption $U\in L^\infty(\mu_\rf\otimes \mu_\rf)$, there exists a unique solution $(\mu_\nl, \nu_\nl)$, and both measures are absolutely continuous with respect to $\mu_\rf$.

There are various algorithms for computing $\mu_\nl$; see, for example, \cite{munos2024nash, rosset2024direct, tiapkin2025accelerating}. Existing algorithms either provide convergence guarantees for the average iterate $T^{-1} \sum_{t=1}^T \mu_t$, or obtain last-iterate convergence by adding a local regularization term $\kl(\cdot\mid \mu_t)$ in the update from $\mu_t$ to $\mu_{t+1}$.

Here, we provide a mirror descent algorithm with a last iterate convergence guarantee where the convergence rate is bounded in terms of $\|U\|_\oplus$. 
\begin{theorem}\label{theorem:convergence_nl_kl}
Define the iteration $(\nu_t, \mu_t)$ by
\begin{align}\label{eq:def_mirror_descent}
  \begin{split}
      \nu_t
      &=
      \argmin_\nu
      \int \nu(dx)\mu_t(dy) U(x, y)
      + \kl(\nu\mid\mu_\rf),\\
  \mu_{t+1}
      &=
      \argmax_\mu
      \int \nu_t(dx) \mu(dy) U(x, y)
      - \kl(\mu\mid\mu_\rf)
      - \eta^{-1} \kl(\mu\mid\mu_t),
  \end{split}
\end{align}
initialized at $\mu_0=\mu_\rf$. Then, for any step size $\eta$ such that $0 < \eta \le \|U\|_\oplus^{-2}$, it holds that
\begin{align*}
\kl(\mu_\nl\mid\mu_t)
  \le
  (1+\eta)^{-t}\cdot 
  \kl(\mu_\nl\mid\mu_\rf).
\end{align*}
\end{theorem}

If we take the step size $\eta=\|U\|_\oplus^{-2}$, we obtain
$
\kl(\mu_\nl\mid\mu_t)
\le
  (1+\|U\|_\oplus^{-2})^{-t}\cdot 
\kl(\mu_\nl\mid\mu_\rf).
$
A similar algorithm to \eqref{eq:def_mirror_descent} with a geometric convergence guarantee of the last iterate is given by previous works, such as \cite{tiapkin2025accelerating}, for the specific form 
$U(x, y)=\lambda^{-1}(\mathcal{P}(x\prec y)-1/2)$.
However, 
the convergence guarantee in previous work depends on $\|U\|_\infty$ rather than on $\|U\|_\oplus$.
Since $\|U\|_\oplus \le \|U\|_\infty$ always holds, and since $\|U\|_\oplus$ can be much smaller than $\|U\|_\infty$ when $U$ has a strong additive component as we discussed in Example \ref{example:logit_transform}-\ref{example:linear_transform}, Theorem \ref{theorem:convergence_nl_kl} provides a sharper convergence rate for NLHF than previous works.

However, to achieve this refined convergence rate for the last iterate, we still need to choose the step size as $\eta=\|U\|_\oplus^{-2}$. Although $\|U\|_\oplus$ can be estimated explicitly up to a multiplicative constant as in Proposition \ref{proposition:L2_estimate}-\ref{proposition:L2_estimate_antisymmetric}, its exact calculation may be challenging in practice.
Moreover, our algorithm still contains a local regularization term, which makes sampling from $\mu_t$ difficult. Indeed, letting $p_t=d\mu_t/d\mu_\rf$, the mirror descent update \eqref{eq:def_mirror_descent} can be written as
\begin{align*}
  \nu_t(dx)
  &\propto
  \exp\Bigl(-\int \mu_t(dy) U(x, y)\Bigr)
  \mu_\rf(dx),\\
  \mu_{t+1}(dy)
  &\propto
  \exp\Bigl(\frac{\eta}{\eta+1}\int \nu_t(dx) U(x, y)\Bigr)
  \tilde{\mu}_t^\eta(dy),
  \quad
  \tilde{\mu}_t^\eta(dy)
  \equiv
  \Bigl(\frac{d\mu_t}{d\mu_\rf}(y)\Bigr)^{\frac{1}{1+\eta}}
  \mu_\rf(dy).
\end{align*}
The update of $\nu_t$ from $\mu_t$ is a simple exponential tilt of the reference distribution. In contrast, the update of $\mu_{t+1}$ requires computing the geometric mixture $\tilde{\mu}_t^\eta$ of $\mu_\rf$ and $\mu_t$, due to the local regularization term $\kl(\mu\mid\mu_t)$.
Exact sampling from such a geometric mixture of probability distributions is challenging in practice; see \cite[Section F.1]{munos2024nash} for a practical algorithm to approximately sample from such mixtures.

To the best of our knowledge, such local regularization is necessary for the last iterate to converge exponentially fast for general $U$. Thus, the drawback above is not a limitation of our algorithm, but rather reflects an intrinsic computational difficulty of NLHF.


Finally,  we show that when $\|U\|_\oplus$ is smaller than $1$, 
the same mirror descent algorithm \eqref{eq:def_mirror_descent} without the local regularization term $\kl(\mu|\mu_t)$ (i.e., $\eta=+\infty$) converges to the equilibrium at a geometric convergence rate in TV distance.  
\begin{theorem}\label{theorem:convergence_nl_tv}
The mirror descent algorithm \eqref{eq:def_mirror_descent} with $\eta=+\infty$ satisfies 
  \begin{align*}
      d_\tv(\nu_t, \nu_\nl)
      \le
      \|U\|_\oplus d_\tv(\mu_t, \mu_\nl), \quad 
      d_\tv(\mu_{t+1}, \mu_\nl)
      \le
      \|U\|_\oplus d_\tv(\nu_t, \nu_\nl).
  \end{align*}
    Thus, if $\|U\|_\oplus<1$, we have 
$
  d_\tv(\mu_t, \mu_\nl)
  \le
  \|U\|_\oplus^{2t} \cdot d_\tv(\mu_\rf, \mu_\nl) \to 0. 
$
\end{theorem}

We emphasize that the convergence rate is bounded in terms of $\|U\|_\oplus$, rather than $\|U\|_\infty$. 
Thus, similarly to the MCHF update in Corollary \ref{corollary:stationary_dist}, this algorithm implicitly adapts to the additive structure of $U$, whenever $\|U\|_\oplus < 1$. 

\section{Comparison between MCHF and NLHF}\label{sec:compare_mchf_nlhf}

\subsection{Computational perspective}
There is a computational trade-off between the two approaches. MCHF requires sampling from the conditional distribution $\MP(x,\cdot)$, whereas NLHF repeatedly solves KL-regularized optimization problems.

As discussed in \Cref{sec:nlhf_algorithm}, last-iterate
convergence of NLHF typically requires adding a local regularization term $\eta^{-1}\kl(\mu\mid\mu_{t-1})$ to the update, and because of this, each iteration requires both
updating the reward function and sampling from a distribution involving a geometric mixture of $\mu_{t-1}$ and $\mu_{\rf}$. This additional sampling step can be practically challenging.

By contrast, MCHF requires implementing the conditional sampler
$\MP(x,dy)$, but this sampler needs to be implemented only once. Once such a sampling block is available, the alignment procedure simply consists of repeatedly applying the same Markov kernel. In this sense, MCHF can be viewed as attaching a preference-guided conditional sampling module on top of the existing architecture that generates samples from $\mu_{\rf}$; see Figure \ref{fig:mchf_nlhf_update_compare}.

\subsection{Coupling perspective and inference time refinement}\label{subsec:coupling_hitting_time}
Another key difference between NLHF and MCHF lies in their algorithms. For the sake of simplicity, we focus on the antisymmetric utility $U$. Suppose $\mu^0$ is a current distribution, and we update it via the following two different rules:
  \begin{align*}
    \mu_\mc^1
    =
    \mu^0\MP,
    \qquad
    \mu_\nl^1
    \in
    \argmax_{\mu}
    \int \mu^0(dx)\mu(dy) U(x,y)
    -
    \kl(\mu \mid \mu_\rf)
  \end{align*}
Note that $\mu_\mc$ is the stationary point of the first update, while $\mu_\nl$ is the equilibrium of the second update since we are assuming that $U$ is antisymmetric now. 

To compare these two updates, consider the couplings induced by MCHF and NLHF:
$$
\pi_\mc(dx, dy) = \mu^0(dx)\MP(x, dy), \quad \pi_\nl(dx, dy)=\mu^0(dx)\mu_\nl^1(dy).
$$
The second marginal of these couplings is precisely $\mu_\mc^1$ and $\mu_\nl^1$, respectively. Then, by the definition of Markov kernel $\MP$ and the update rule of $\mu_\nl^1$, $\pi_\mc$ and $\pi_\nl$ maximize the same objective function:
$$
   F(\pi) = \iint U(x, y)\pi(dx, dy) - \kl(\pi || \mu^0 \otimes \mu_\rf), 
$$
but over different feasible spaces of couplings:
 \begin{align*}
    \pi_\mc = \argmax_{\int \pi(\cdot, dy) = \mu^0(\cdot)} F(\pi), \quad \pi_\nl = \argmax_{\exists \mu, \pi = \mu^0\otimes \mu} F(\pi).
 \end{align*}
Thus, the MCHF coupling optimizes over a larger class of couplings than that of NLHF, and in this sense the coupling induced by MCHF can extract more utility from the preference structure than the product coupling induced by NLHF.

This coupling perspective also highlights a practical difference between the two methods. Since MCHF produces a conditional distribution $\MP(x,\cdot)$ given the current output $x$, it naturally supports an \textit{inference-time refinement} procedure: one may repeatedly refine a model output by sampling from $\MP(x,\cdot)$ until the user is satisfied. This is reminiscent of the inference-time refinement mechanism discussed in Stackelberg game-based alignment \citep{pasztor2026stackelberg}. By contrast, the NLHF update directly produces a marginal distribution and therefore does not provide the same conditional refinement mechanism. 

{Below we provide a simple example to highlight the benefit of MCHF (note that this example is used in \cite{pasztor2026stackelberg} to highlight the benefit of conditional alignment). 
 Suppose $\mu_\rf$ is a uniform distribution on $\{1, 2, 3\}$, and the pairwise preference data is given by the following three annotators
\begin{align*}
 \text{Annotator 1:} \quad 1 \prec 2 \prec 3, \quad 
 \text{Annotator 2:} \quad 3 \prec 1\prec 2, \quad 
 \text{Annotator 3:} \quad 2 \prec 3 \prec 1. 
\end{align*}
Note that each annotator has a consistent preference order. Suppose we take $U$ as $U(x, y)=\mathcal{P}(x\prec y)-1/2$ where $\mathcal{P}(x\prec y)$ is the ratio of annotators who prefer $y$ to $x$. In this case, $U$ can be written as 
\begin{align*}
  U =\begin{bmatrix}
    1/2 & 2/3 & 1/3\\
    1/3 & 1/2 & 2/3\\
    2/3& 1/3 & 1/2
  \end{bmatrix} - \frac{1}{2} \bm{1}_{3\times 3} = 
   \frac{1}{6}\begin{bmatrix}
    0 & 1 & -1\\
    -1 & 0 & 1\\
    1& -1 & 0
  \end{bmatrix}.
\end{align*}
In this case, the NLHF solution is the uniform distribution $[1/3, 1/3, 1/3]$, which is the same as $\mu_\rf$. Suppose the annotator 1 keeps sampling from NLHF (in this case, uniform) until the annotator gets the most preferred output `3'. The expected number of samples needed is $3$ since it is the expectation of the geometric distribution with success probability $1/3$. Meanwhile, suppose we restore $U$ as it is, and conditioning on the current output, say $x$, we choose the next output $y$ as $\argmax_y U(x, y)$, which corresponds to the conditional sampling from our Markov kernel $\MP_\gamma(x, dy)\propto \exp(\gamma U(x, y))\mu_\rf(dy)$ with $\gamma\to+\infty$. 
 In this case the expected number of total iterations is given by $1/3 + (1/3) \times 2 + (1/3) \times 3=2$, which is smaller than that of NLHF. 
}

In \Cref{sec:inference_time_refinement}, we extend the above rock-paper-scissors game to a general reference measure $\mu_\rf$ and general antisymmetric utility $U$, and show that we can reduce the number of samples needed to get a desired output by running our Markov kernel $\MP_\gamma(x, dy)\propto \exp(\gamma U(x, y))\mu_\rf(dy)$ for a finite $\gamma$. In particular, we show that for any measurable set $A$ such that $\mu_\rf(A)>0$, if there is enough preference ``energy'', $U^2(x, y)$, from the complement set $A^c$ toward $A$, the hitting time to the set $A$ can be reduced by taking a finite $\gamma>0$. The key idea is to analyze the linear system that the hitting time satisfies, and calculate the first and the second derivatives of the hitting time with respect to $\gamma$ using the implicit function theorem; see \Cref{sec:inference_time_refinement} for details.

\subsection{Alignment dynamics in the small $\|U\|_\oplus$ asymptotics}\label{sec:alignment_dynamics}
As we have seen in Proposition \ref{proposition:L2_estimate}, $\|U\|_\oplus=0$ is equivalent to 
$U(x,y)=\hat g(x)+\hat f(y)$, where $(\hat{g}, \hat{f})$ is a solution to the $L^2$ minimization problem $\min_{g,f}\|U-g\oplus f\|_2$, which can be written explicitly as the weighted row/column sum, as in \eqref{eq:L2_min_solution}. 
Moreover, by the definitions of MCHF and NLHF, if $U=g\oplus f$, then MCHF and NLHF equilibria both collapse to the RLHF solution with reward $f$, i.e., $\mu_\rl(f)(dy)\propto \exp(f(y))\mu_\rf(dy)$. We summarize the above discussion in the proposition below. 
\begin{proposition}\label{proposition:seminorm_zero}
For any $U\in L^\infty(\mu_\rf\otimes \mu_\rf)$, let $\mu_{\mc}(U)$ and $\mu_\nl(U)$ be MCHF and NLHF based on utility function $U$ respectively. 
Then it holds that 
  \begin{align*}
  \|U\|_\oplus=0
  \quad\Leftrightarrow\quad U = \hat g\oplus \hat f \quad \Rightarrow \quad 
  \mu_\mc(U)=\mu_\nl(U)=\mu_\rl(\hat f)
\end{align*}
where $\mu_\rl(\hat{f})(dy)\propto \exp(\hat{f}(y))\mu_\rf(dy)$ and $\hat{f}(y)=\int \mu_\rf(dx) U(x, y)$. 
\end{proposition}

We aim to provide a quantitative version of this claim: if $\|U\|_\oplus$ is small but not exactly zero, how far are $\mu_\mc(U)$ and $\mu_\nl(U)$ from $\mu_\rl(\hat f)$? This question is natural since we have seen in \Cref{sec:projection_utility} that $\|U\|_\oplus$ can be significantly small (relative to $\|U\|_\infty$) in natural human preference profiles. Below, we address this question by calculating the derivative of MCHF and NLHF with respect to $U$. 

\begin{theorem}[Uniform Fr\'echet differentiability]\label{theorem:uniform_frechet_differentiability}
  Let $p_*(U)=\frac{d\mu_*(U)}{d\mu_\rf}$ be the density for $*\in \{\mc, \nl\}$. Then, $U\mapsto p_*(U)$ is uniformly Fr\'echet differentiable in the following sense:
\begin{align*}
  \forall *\in \{\mc, \nl\}, \quad \forall R\ge 0, \quad 
  \lim_{\|E\|_\infty\to 0}\sup_{\|U\|_\oplus \le R} \frac{\|p_*(U+E)-p_*(U)- Dp_*(U)[E]\|_1}{\|E\|_\infty} = 0
\end{align*}
where $D p_*(U)[E]$ is a Fr\'echet derivative (see Theorems \ref{theorem:frechet_derivative_mc} and \ref{theorem:frechet_derivative_nl} for the explicit form for $*=\mc$ and $*=\nl$, respectively). 
\end{theorem}

Note that the derivatives $D p_\mc(U)$ and $D p_\nl(U)$ differ for general $U$. However, when $U=-f\oplus f$ and $E$ is antisymmetric, they collapse into the following simple form. 
\begin{proposition}\label{proposition:derivative_additive_antisymmetric}
  For any $f\in L^\infty$ and antisymmetric $E\in L^\infty(\mu_\rf\otimes \mu_\rf)$, for each $*\in \{\mc, \nl\}$, we have 
  $$
  D p_*((-f)\oplus f)[E] = p_\rl(f) \odot E^\star p_\rl(f)
  $$
  where $p_\rl(f)(y) = \frac{\exp(f(y))}{\int\mu_\rf(dy') \exp(f(y'))}$,  $\odot$ is the entrywise product, and $E^\star h(\cdot)=\int\mu_\rf(dx) E(x, \cdot) h(x)$. 
\end{proposition}
That is, when $U$ takes the additive form $U=-f\oplus f$, and for antisymmetric perturbation $E$, the derivatives of MCHF and NLHF coincide. 
Combining the above theorem and proposition, let us derive a quantitative version of Proposition \ref{proposition:seminorm_zero}; consider a sequence of antisymmetric utilities $U_1, U_2, \dots$ such that $\|U_n\|_\oplus\to 0$.  Let $\hat{f}_n(y)=\int \mu_\rf(dx) U_n(x, y)$ be the weighted column sum.  Then, by Proposition \ref{proposition:L2_estimate_antisymmetric}, we know 
$$
\|U_n-(-\hat{f}_n)\oplus \hat{f}_n\|_\infty \asymp \|U_n\|_\oplus.
$$
Then, applying Theorem \ref{theorem:uniform_frechet_differentiability} with $U=(-\hat{f}_n) \oplus \hat{f}_n$ and $E=U_n-(-\hat{f}_n)\oplus \hat{f}_n$, noting that $\|U\|_\oplus = 0$ in this case, 
using the derivative formula by Proposition \ref{proposition:derivative_additive_antisymmetric}, 
we obtain the following result. 
\begin{corollary}\label{corollary:asymptotics_equilibrium}
  For any sequence of antisymmetric utilities $U\in L^\infty(\mu_\rf\otimes \mu_\rf)$ such that $\|U\|_\oplus \to 0$, and for each $*\in \{\mc, \nl\}$, 
  \begin{align*}
  \bigl\|p_{*}(U) - p_{\rl}(\hat f) - p_\rl(\hat f) \odot \bigl(U-(-\hat f)\oplus \hat f\bigr)^\star p_{\rl}(\hat f)\bigr\|_{L^1(\mu_\rf)} = o(\|U\|_\oplus), 
  \end{align*}
  where $\hat{f}(y)=\int \mu_\rf(dx) U(x, y)$. 
\end{corollary}
This implies that MCHF and NLHF have the same leading term in the small-$\|U\|_\oplus$ asymptotics. 
Combined with the triangle inequality, we also get 
$$
  d_\tv(\mu_\mc(U), \mu_\nl(U)) = o(\|U\|_\oplus).
$$
We emphasize that the right-hand side is $o(\|U\|_\oplus)$, rather than merely $O(\|U\|_\oplus)$.

Let us verify Corollary \ref{corollary:asymptotics_equilibrium} by numerical simulation.
We set $\mu_\rf$ to be the uniform distribution on a discrete space with
$|\mathcal X|=20$, and generate utility $U$ as 
\begin{align}\label{eq:utility_DGP}
  U(x,y)
  =
\frac12 \log \frac{\mathcal{P}(x\prec y)}{1-\mathcal{P}(x\prec y)},
  \qquad
  \mathcal{P}(x\prec y)
  =
  \Phi\Bigl(R(y)-R(x)+\frac12E(x,y)\Bigr),
\end{align}
where $\Phi$ is the CDF of the standard normal distribution.
We generate $R\sim N(0,I_n)$, and take $E\in\mathbb R^{n\times n}$ as an antisymmetric matrix whose off-diagonal entries are drawn from $N(0,1)$.
The resulting pairwise utility $U$ is visualized in \Cref{fig:visualize_utility}, which suggests that
$\|U\|_\oplus\ll \|U\|_\infty$.
\Cref{fig:equilibrium_compare} plots MCHF, NLHF, RLHF, and the first-order approximation given by Corollary \ref{corollary:asymptotics_equilibrium}.  
We observe that MCHF and NLHF nearly coincide, and that their deviation from RLHF is accurately captured by the first-order approximation.\\

\begin{figure}[htpb]
  \centering
  \includegraphics[width=\linewidth]{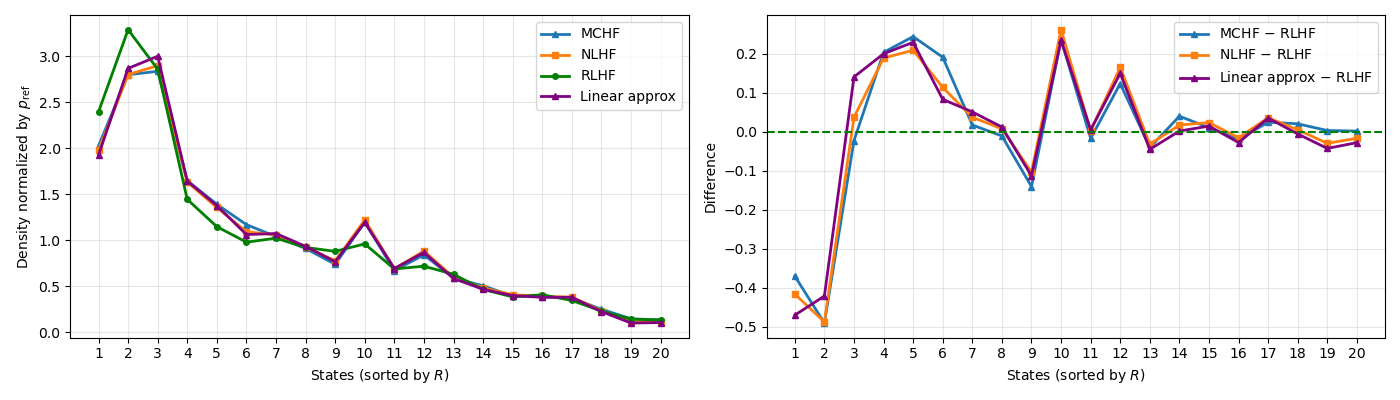}
  \caption{Comparison of MCHF, NLHF, RLHF, and the first-order approximation in Corollary \ref{corollary:asymptotics_equilibrium}.}
  \label{fig:equilibrium_compare}
\end{figure}

Next, we claim that MCHF and NLHF agree up to first order in their corresponding iterative dynamics as well. 
\begin{theorem}\label{theorem:asymptotics_iterations}
 For any antisymmetric utility $U\in L^\infty(\mu_\rf\otimes \mu_\rf)$, consider the following iterations:
  \begin{align*}
    \mu_\mc^t = \mu_\mc^{t-1}\MP, \quad 
    \mu_\nl^t = \argmax_{\mu} \int \mu_\nl^{t-1}(dx) \mu(dy) U(x, y) - \kl(\mu || \mu_\rf)
  \end{align*}
  initialized at $\mu_\mc^0=\mu_\nl^0=\mu_\rf$, where $\MP$ is the Markov kernel based on $U$. Let $p_*^t(U)=d\mu_*^t/d\mu_\rf$ be the corresponding density for $t\ge 1$ and $*\in\{\mc, \nl\}$. Then, for any sequence of antisymmetric utilities $U$ such that
$\|U\|_\oplus\to 0$, we have that 
for each $*\in \{\mc, \nl\}$, 
  \begin{align*}
 \begin{split}
  \bigl\|p_*^1(U) - p_\rl(\hat f)\bigr\|_{L^1(\mu_\rf)}
  &= o(\|U\|_\oplus), \\
  \sup_{t\ge 2} \bigl\|p_*^t(U) - p_\rl(\hat f) - p_\rl(\hat f)\odot \bigl(U -(-\hat f)\oplus \hat f\bigr)^\star p_\rl(\hat f)\bigr\|_{L^1(\mu_\rf)}
  &= o(\|U\|_\oplus). 
  \end{split}
  \end{align*}
\end{theorem}
From this theorem, we see that MCHF and NLHF agree at the iteration level up to
first order:
$$
 \sup_{t\ge 1}d_{\tv} (\mu_\nl^t(U), \mu_\mc^t(U)) = o(\|U\|_\oplus).
$$
Moreover, this theorem sheds light on the dynamics of alignment. At the first iteration $t=1$, both algorithms recover the RLHF update with reward $\hat{f}$, which is the weighted column sum of $U$. Then for $t\ge 2$, they incorporate the non-additive structure $U-(-\hat{f})\oplus \hat{f}$ as a first-order correction to the RLHF update. Furthermore, comparing this correction with the leading term of the equilibrium in Corollary \ref{corollary:asymptotics_equilibrium}, we see that the two algorithms capture the leading term within the first two iterations. 


\subsection{Stability under perturbation of utilities}
Finally, en route to the proof of Theorem \ref{theorem:uniform_frechet_differentiability}, we obtain a stability result for MCHF and NLHF equilibria under the perturbation of utilities. 
This stability property is important in practice, since the utility function $U$ may be updated as new preference data become available.
We may also consider a setting in which $U$ is taken to be an increasing transformation of the preference model $\mathcal{P}(x \prec y)$, which is itself estimated, for example, by logistic regression as in the original paper \cite{munos2024nash}. In this case, estimation error in $\mathcal{P}$ induces a perturbation in $U$. It is therefore important to understand how such perturbations affect the resulting equilibrium.

\begin{theorem}[Stability under perturbation of utilities]\label{theorem:stability}
  For any constant $R\ge 0$, let $B^\infty_\oplus(R) = \{U\in L^\infty(\mu_\rf\otimes \mu_\rf): \|U\|_\oplus \le R\}$ be the ball of radius $R$ with respect to the seminorm $\|\cdot\|_\oplus$. Then, there exists a constant $C_R$, depending only on $R$, such that for each $* \in \{\mc, \nl\}$, we have 
$$
\sup_{U, \hat{U}\in B^\infty_\oplus (R)} \frac{d_\tv(\mu_*(U), \mu_*(\hat{U}))}{\|U-\hat{U}\|_\infty} \le C_R. 
$$
\end{theorem}
See \Cref{sec:perturbation} for the precise dependence of $C_R$ on $R$.
This theorem implies that both maps $U\mapsto \mu_\mc(U)$ and
$U\mapsto \mu_\nl(U)$ are Lipschitz continuous on the ball
$B^\infty_\oplus(R)$.  Since
$\|U\|_\oplus \le \|U\|_\infty$, this ball includes the usual
$L^\infty$-ball of radius $R$. Furthermore, $B^\infty_\oplus(R)$ is invariant
under the addition of additive structure: if $U\in B^\infty_\oplus(R)$, then
$U+g\oplus f\in B^\infty_\oplus(R)$ for any $g,f\in L^\infty$. Thus, the result
allows perturbations around a larger class of utilities than a standard
$L^\infty$-ball would.

\section{Discussion}
In this paper, we proposed MCHF, an alternative framework for aligning generative models with human preferences by constructing a Markov kernel directly from pairwise preference information. We showed that the resulting Markov chain converges geometrically fast to its stationary distribution, with a contraction rate governed by the seminorm $\|U\|_\oplus$, which captures the non-transitive component of the pairwise utility. We further showed that the mirror-descent algorithm for NLHF, under an appropriate choice of step size, satisfies an analogous structure-adaptive convergence guarantee. Finally, through a perturbation analysis, we proved that when $\|U\|_\oplus$ is small, MCHF and NLHF agree up to first order around the RLHF solution. These results provide a unified perspective on reward-based, game-theoretic, and Markovian approaches to preference alignment.

It would also be interesting to investigate how the pairwise utility should be designed so that the aligned distribution induced by MCHF satisfies desirable axioms from social choice theory. Related questions have been studied for RLHF and NLHF; see, for example, \cite{liu2025statistical,golz2025distortion} and the references therein. We leave a systematic social-choice-theoretic analysis of MCHF to future work.

Our analysis in this paper was primarily theoretical, and we leave the practical implementation of MCHF for future work. To run the Markov chain in practice, one needs to sample from the conditional distribution
$
\MP(x,dy)\propto \exp(U(x,y))\mu_\rf(dy)
$
given the current output $x$. This problem can be viewed as a conditional analogue of sampling from an RLHF solution of the form
$
\mu_\rl(dy)\propto \exp(r(y))\mu_\rf(dy).
$
Recent developments for sampling from such reward-tilted or unnormalized distributions, including rejection-sampling-based approaches \cite{liu2025statistical}, SDE-based methods such as diffusion models (see \cite{uehara2024understanding} and the references therein), and a recent Monge--Ampère PDE sampler \cite{deb2025no}, provide promising tools for this purpose. Developing scalable implementations of MCHF by adapting these techniques to the Markov kernel $\MP(x,dy)$ is an important direction for future work.

\ifbiblatex
\else
  \bibliographystyle{abbrvnat}
  
  \bibliography{reference}

@article{deb2025no,
  title={{No-Regret Generative Modeling via Parabolic Monge--Amp{\`e}re PDE}},
  author={Deb, Nabarun and Liang, Tengyuan},
  journal={arXiv preprint arXiv:2504.09279},
  year={2025}
}

@article{rafailov2023direct,
  title={{Direct preference optimization: Your language model is secretly a reward model}},
  author={Rafailov, Rafael and Sharma, Archit and Mitchell, Eric and Manning, Christopher D and Ermon, Stefano and Finn, Chelsea},
  journal={Advances in neural information processing systems},
  volume={36},
  pages={53728--53741},
  year={2023}
}

@inproceedings{munos2024nash,
  title={{Nash Learning from Human Feedback}},
  author={Munos, R{\'e}mi and Valko, Michal and Calandriello, Daniele and Azar, Mohammad Gheshlaghi and Rowland, Mark and Guo, Zhaohan Daniel and Tang, Yunhao and Geist, Matthieu and Mesnard, Thomas and Fiegel, C{\^o}me and others},
  booktitle={Forty-first International Conference on Machine Learning},
  year={2024}
}

@inproceedings{golz2025distortion,
title={{Distortion of {AI} Alignment: Does Preference Optimization Optimize for Preferences?}},
author={Paul G{\"o}lz and Nika Haghtalab and Kunhe Yang},
booktitle={The Thirty-ninth Annual Conference on Neural Information Processing Systems},
year={2026},
}

@article{liu2025statistical,
  title={{Statistical Impossibility and Possibility of Aligning LLMs with
Human Preferences: From Condorcet Paradox to Nash Equilibrium}},
  author={Liu, Kaizhao and Long, Qi and Shi, Zhekun and Su, Weijie J and Xiao, Jiancong},
  journal={arXiv preprint arXiv:2503.10990},
  year={2025}
}

@article{shi2025fundamental,
  title={{Fundamental Limits of Game-Theoretic LLM Alignment: Smith Consistency and Preference Matching}},
  author={Shi, Zhekun and Liu, Kaizhao and Long, Qi and Su, Weijie J and Xiao, Jiancong},
  journal={arXiv preprint arXiv:2505.20627},
  year={2025}
}

@article{tiapkin2025accelerating,
  title={{Accelerating Nash Learning from Human Feedback via Mirror Prox}},
  author={Tiapkin, Daniil and Calandriello, Daniele and Belomestny, Denis and Moulines, Eric and Naumov, Alexey and Rasul, Kashif and Valko, Michal and Menard, Pierre},
  journal={arXiv preprint arXiv:2505.19731},
  year={2025}
}

@article{rosset2024direct,
title={{Direct Nash Optimization:
Teaching Language Models to Self-Improve with General Preferences}},
  author={Rosset, Corby and Cheng, Ching-An and Mitra, Arindam and Santacroce, Michael and Awadallah, Ahmed and Xie, Tengyang},
  journal={arXiv preprint arXiv:2404.03715},
  year={2024}
}

@inproceedings{wu2024self,
 author = {Wu, Yue and Sun, Zhiqing and Hughes, Rina and Ji, Kaixuan and Yang, Yiming and Gu, Quanquan},
 booktitle = {International Conference on Learning Representations},
 editor = {Y. Yue and A. Garg and N. Peng and F. Sha and R. Yu},
 pages = {91558--91582},
 title = {Self-Play Preference Optimization for Language Model Alignment},
 volume = {2025},
 year = {2025}
}

@article{wang2023rlhf,
  title={{Is RLHF More Difficult than Standard RL?
A Theoretical Perspective}},
  author={Wang, Yuanhao and Liu, Qinghua and Jin, Chi},
  journal={Advances in Neural Information Processing Systems},
  volume={36},
  pages={76006--76032},
  year={2023}
}

@article{bradley1952rank,
 author = {Ralph Allan Bradley and Milton E. Terry},
 journal = {Biometrika},
 number = {3/4},
 pages = {324--345},
 publisher = {[Oxford University Press, Biometrika Trust]},
 title = {{Rank Analysis of Incomplete Block Designs: I. The Method of Paired Comparisons}},
 urldate = {2026-04-21},
 volume = {39},
 year = {1952}
}

@InProceedings{ethayarajh2024kto,
  title = 	 {{KTO: Model Alignment as Prospect Theoretic Optimization}},
  author =       {Ethayarajh, Kawin and Xu, Winnie and Muennighoff, Niklas and Jurafsky, Dan and Kiela, Douwe},
  booktitle = 	 {Proceedings of the 41st International Conference on Machine Learning},
  pages = 	 {12634--12651},
  year = 	 {2024},
  volume = 	 {235},
  series = 	 {Proceedings of Machine Learning Research},
  publisher =    {PMLR},
}

@inproceedings{meng2024simpo,
 author = {Meng, Yu and Xia, Mengzhou and Chen, Danqi},
 booktitle = {Advances in Neural Information Processing Systems},
 pages = {124198--124235},
 publisher = {Curran Associates, Inc.},
 title = {{SimPO: Simple Preference Optimization with a Reference-Free Reward}},
 volume = {37},
 year = {2024}
}

@InProceedings{azar2024general,
  title = 	 {{A General Theoretical Paradigm to Understand Learning from Human Preferences}},
  author =       {Gheshlaghi Azar, Mohammad and Daniel Guo, Zhaohan and Piot, Bilal and Munos, Remi and Rowland, Mark and Valko, Michal and Calandriello, Daniele},
  booktitle = 	 {Proceedings of The 27th International Conference on Artificial Intelligence and Statistics},
  pages = 	 {4447--4455},
  year = 	 {2024},
  editor = 	 {Dasgupta, Sanjoy and Mandt, Stephan and Li, Yingzhen},
  volume = 	 {238},
  series = 	 {Proceedings of Machine Learning Research},
  publisher =    {PMLR},
}

@inproceedings{huang2024correcting,
 author = {Huang, Audrey and Zhan, Wenhao and Xie, Tengyang and Lee, Jason and Sun, Wen and Krishnamurthy, Akshay and Foster, Dylan},
 booktitle = {International Conference on Learning Representations},
 pages = {92647--92701},
 title = {{Correcting the Mythos of KL-Regularization: Direct Alignment without Overoptimization via Chi-Squared Preference Optimization}},
 volume = {2025},
 year = {2025}
}

@article{schulman2017proximal,
  title={Proximal policy optimization algorithms},
  author={Schulman, John and Wolski, Filip and Dhariwal, Prafulla and Radford, Alec and Klimov, Oleg},
  journal={arXiv preprint arXiv:1707.06347},
  year={2017}
}

@book{luce1959individual,
  title={{Individual Choice Behavior: A Theoretical Analysis}},
  author={Luce, R. Duncan},
  publisher={Wiley},
  year={1959}
}

@inproceedings{christiano2017deep,
  title={Deep reinforcement learning from human preferences},
  author={Christiano, Paul F. and Leike, Jan and Brown, Tom B. and Martic, Miljan and Legg, Shane and Amodei, Dario},
  booktitle={Advances in Neural Information Processing Systems},
  volume={30},
  year={2017}
}

@InProceedings{swamy2024minimaximalist,
  title = 	 {{A Minimaximalist Approach to Reinforcement Learning from Human Feedback}},
  author =       {Swamy, Gokul and Dann, Christoph and Kidambi, Rahul and Wu, Steven and Agarwal, Alekh},
  booktitle = 	 {Proceedings of the 41st International Conference on Machine Learning},
  pages = 	 {47345--47377},
  year = 	 {2024},
  volume = 	 {235},
  series = 	 {Proceedings of Machine Learning Research},
  publisher =    {PMLR},
}

@inproceedings{ouyang2022training,
  title={Training language models to follow instructions with human feedback},
  author={Ouyang, Long and Wu, Jeffrey and Jiang, Xu and Almeida, Diogo and Wainwright, Carroll L. and Mishkin, Pamela and Zhang, Chong and Agarwal, Sandhini and Slama, Katarina and Ray, Alex and others},
  booktitle={Advances in Neural Information Processing Systems},
  volume={35},
  pages={27730--27744},
  year={2022}
}

@article{bai2022training,
  title={Training a helpful and harmless assistant with reinforcement learning from human feedback},
  author={Bai, Yuntao and Jones, Andy and Ndousse, Kamal and Askell, Amanda and Chen, Anna and DasSarma, Nova and Drain, Dawn and Fort, Stanislav and Ganguli, Deep and Henighan, Tom and others},
  journal={arXiv preprint arXiv:2204.05862},
  year={2022}
}

@article{page1999pagerank,
title = {{The anatomy of a large-scale hypertextual Web search engine}},
journal = {Computer Networks and ISDN Systems},
volume = {30},
number = {1},
pages = {107-117},
year = {1998},
note = {Proceedings of the Seventh International World Wide Web Conference},
issn = {0169-7552},
author = {Sergey Brin and Lawrence Page},
}

@article{negahban2012rank,
  title={{Rank centrality: Ranking from pairwise comparisons}},
  author={Negahban, Sahand and Oh, Sewoong and Shah, Devavrat},
  journal={Operations Research},
  volume={65},
  number={1},
  pages={266--287},
  year={2017}
}

@article{jiang2011statistical,
  title={{Statistical ranking and combinatorial Hodge theory}},
  author={Jiang, Xiaoye and Lim, Lek-Heng and Yao, Yuan and Ye, Yinyu},
  journal={Mathematical Programming},
  volume={127},
  number={1},
  pages={203--244},
  year={2011},
  publisher={Springer}
}

@techreport{openai2023gpt4,
  title={GPT-4 Technical Report},
  author={{OpenAI}},
  institution={{OpenAI}},
  journal={arXiv preprint arXiv:2303.08774},
  year={2023}
}

@techreport{geminiteam2025gemini,
  title={{Gemini 2.5: Pushing the Frontier with Advanced Reasoning, Multimodality, Long Context, and Next Generation Agentic Capabilities}},
  author={{Gemini Team, Google}},
  institution={Google DeepMind},
  journal={arXiv preprint arXiv:2507.06261},
  year={2025}
}

@techreport{anthropic2024claude,
  title={{The Claude 3 Model Family: Opus, Sonnet, Haiku}},
  author={{Anthropic}},
  institution={Anthropic},
  year={2024}
}

@inproceedings{wu2026multiplayer,
title={{Multiplayer Nash Preference Optimization}},
author={Fang Wu and Xu Huang and Weihao Xuan and Zhiwei Zhang and Yijia Xiao and Guancheng Wan and Xiaomin Li and Bing Hu and Peng Xia and Jure Leskovec and Yejin Choi},
booktitle={The Fourteenth International Conference on Learning Representations},
year={2026},
}

@inproceedings{pasztor2026stackelberg,
title={{Stackelberg Learning from Human Feedback: Preference Optimization as a Sequential Game}},
author={Barna P{\'a}sztor and Thomas Kleine Buening and Andreas Krause},
booktitle={The Fourteenth International Conference on Learning Representations},
year={2026},
}

@inproceedings{chu2025stackelberg,
title={{Stackelberg Self-Annotation: A Robust Approach to Data-Efficient {LLM} Alignment}},
author={Xu Chu and Zhixin Zhang and Tianyu Jia and Yujie Jin},
booktitle={The Thirty-ninth Annual Conference on Neural Information Processing Systems},
year={2026},
}

@InProceedings{sorensen2024roadmap,
  title = {{Position: A Roadmap to Pluralistic Alignment}},
  author =    {Sorensen, Taylor and Moore, Jared and Fisher, Jillian and Gordon, Mitchell L and Mireshghallah, Niloofar and Rytting, Christopher Michael and Ye, Andre and Jiang, Liwei and Lu, Ximing and Dziri, Nouha and Althoff, Tim and Choi, Yejin},
  booktitle = 	 {Proceedings of the 41st International Conference on Machine Learning},
  pages = 	 {46280--46302},
  year = 	 {2024},
  volume = 	 {235},
  series = 	 {Proceedings of Machine Learning Research},
  month = 	 {21--27 Jul},
  publisher =    {PMLR},
}

@article{dobrushin1956central,
author = {Dobrushin, R. L.},
title = {{Central Limit Theorem for Nonstationary Markov Chains. I}},
journal = {Theory of Probability \& Its Applications},
volume = {1},
number = {1},
pages = {65-80},
year = {1956},
}

@article{rosenthal1995minorization,
author = {Jeffrey S. Rosenthal},
title = {{Minorization Conditions and Convergence Rates for Markov Chain Monte Carlo}},
journal = {Journal of the American Statistical Association},
volume = {90},
number = {430},
pages = {558--566},
year = {1995},
publisher = {Taylor \& Francis},
}

@book{krantz2013implicit,
  title     = {{The Implicit Function Theorem: History, Theory, and Applications}},
  author    = {Krantz, Steven G. and Parks, Harold R.},
  series    = {Modern Birkh{\"a}user Classics},
  publisher = {Birkh{\"a}user New York},
  address   = {New York, NY},
  year      = {2013},
  doi       = {10.1007/978-1-4614-5981-1}
}

@article{haihao2018relatively,
author = {Lu, Haihao and Freund, Robert M. and Nesterov, Yurii},
title = {Relatively Smooth Convex Optimization by First-Order Methods, and Applications},
journal = {SIAM Journal on Optimization},
volume = {28},
number = {1},
pages = {333-354},
year = {2018},
}

@article{uehara2024understanding,
  title={Understanding reinforcement learning-based fine-tuning of diffusion models: A tutorial and review},
  author={Uehara, Masatoshi and Zhao, Yulai and Biancalani, Tommaso and Levine, Sergey},
  journal={arXiv preprint arXiv:2407.13734},
  year={2024}
}
\fi


\appendix
\section*{Notation}
\addcontentsline{toc}{section}{Notation}
\begin{center}
\small
\begin{tabularx}{\textwidth}{lX}
\toprule
Notation & Meaning \\
\midrule
$\mu_\rf$ & Reference distribution, e.g. the pre-trained model. \\
$U(x,y)$ & Pairwise utility; larger values mean that $y$ is preferred to $x$. \\
$g\oplus f$ & Additive function defined by $(g\oplus f)(x,y)=g(x)+f(y)$. \\
$\|U\|_\infty$ & Essential supremum norm on $L^\infty(\mu_\rf\otimes\mu_\rf)$. \\
$\|U\|_\oplus$ & Additive-defect seminorm:
$\inf_{g,f\in L^\infty(\mu_\rf)}\|U-g\oplus f\|_\infty$. \\
$\hat f$ & Column-sum projection:
$\hat f(y)=\int U(x,y)\mu_\rf(dx)$. \\
$\hat g$ & Row-sum projection:
$\hat g(x)=\int U(x,y)\mu_\rf(dy)-m$ where 
$m=\int U(x,y)\mu_\rf(dx)\mu_\rf(dy)$. \\
\midrule
$\MP(x,dy)$ & MCHF Markov kernel:
$\MP(x,dy)=Z(x)^{-1}\exp(U(x,y))\mu_\rf(dy)$ \\
$\mu\MP$ & One-step update of $\mu$ by the Markov kernel $\MP$. \\
$\mu_\mc(U)$ & Stationary distribution of the MCHF Markov kernel based on $U$ \\
$\mu_\nl(U)$ & NLHF equilibrium distribution  based on $U$\\
$\nu_\nl(U)$ & Adversarial/opponent equilibrium distribution in NLHF  based on $U$ \\
$\mu_\rl(f)$ & RLHF distribution with reward $f$:
$\mu_\rl(f)(dy)\propto \exp(f(y))\mu_\rf(dy)$. \\
$p_\ast(U)$ & Density $d\mu_\ast(U)/d\mu_\rf$ for $\ast\in\{\mc, \nl\}$. \\
$p_{\rl}(f)$ & Density $d\mu_{\rm RL}(f)/d\mu_\rf$. \\
\midrule
$K h$ & Integral operator $Kh (x)=\int K(x, y) h(y)$.\\
$K^\star h$ & Adjoint operator:
$K^\star h(y)=\int K(x,y)h(x)\mu_\rf(dx)$. \\
$\odot$ & Entrywise product \\
\midrule
$L^p$ &Set of $L^p(\mu_\rf)$-measurable functions\\
$L^p_0$ &Set of $L^p(\mu_\rf)$-measurable function with mean zero: $L^p_0 = \{f\in L^p: \int \mu_\rf(dy) f(y) = 0\}$  \\
$\Delta^\infty$ & Density simplex: $\Delta^\infty = \{p\in L^\infty: p(\cdot)\ge 0, \int \mu_\rf(dz) p(z)=1\}$\\
$\softmax$ & Softmax operator $\softmax: L^\infty \to \Delta^\infty$, $f\mapsto \softmax(f)(x) = \frac{\exp(f(x))}{\int \mu_{\rf}(dx') \exp(f(x'))} $
\\
$J_p$ & Demean operator $J_{p}: L^1 \to L^1_0$ for each $p\in \Delta^\infty$: $f \mapsto J_p f  = p\odot (f - \int \mu_{\rf}(dz') p(z') f(z'))$\\
\bottomrule
\end{tabularx}
\end{center}

\section{Proof for \Cref{sec:projection_utility}}
\subsection{Proof of Proposition \ref{proposition:L2_estimate}}\label{proof:proposition:L2_estimate}
Take $(\hat{g}, \hat{f})\in \argmin_{g,f\in L^2}\|U-g\oplus f\|_2$ as 
$$
\hat g(x)
=
\int U(x,y')\mu_{\rf}(dy')-m,
\quad
\hat f(y)
=
\int U(x',y)\mu_{\rf}(dx'), \quad m=\int U(x',y)\mu_{\rf}(dx')\mu_{\rf}(dy')
$$
Here, we have $\hat{g}, \hat{f}\in L^\infty$ by $U\in L^\infty(\mu_\rf\otimes\mu_\rf)$, and hence 
$\|U\|_\oplus \le \|U-\hat{g}\oplus \hat{f}\|_\infty$ holds 
by the definition of $\|U\|_\oplus$. Thus, it suffices to show $4^{-1}\square(U)\le \|U\|_\oplus$ and $\|U-\hat{g}\oplus \hat{f}\|_\infty\le \square(U)$. 

\paragraph{Proof of $4^{-1}\square(U)\le \|U\|_\oplus$}
Fix arbitrary $g,f\in L^\infty$ and let $U^{g\oplus f}=U-g\oplus f$. 
For every $x,x',y,y'$, noting that the $g$- and $f$-terms cancel, 
\begin{align*}
U(x',y')-U(x,y')-U(x',y)+U(x,y) =
U^{g\oplus f}(x',y')-U^{g\oplus f}(x,y')-U^{g\oplus f}(x',y)+U^{g\oplus f}(x,y),
\end{align*}
Taking $\esssup_{x,y,x',y'}$ on both sides, the LHS becomes $\square(U)$, while using the triangle inequality for the RHS, we get 
$\square(U) \le 4 \|U^{g\oplus f}\|_\infty$. 
Taking $\inf_{g,f\in L^\infty}$ on the RHS, we get $\square(U)\le 4\|U\|_\oplus$ and complete the proof.

\paragraph{Proof of $\|U-\hat{g}\oplus \hat{f}\|_\infty\le \square(U)$}
By the definition of $\hat{g}$ and $\hat{f}$, for all $(x, y)$, 
\begin{align*}
(U-\hat g\oplus \hat f)(x,y) 
&=\int \Bigl(
U(x,y)-U(x',y)-U(x,y')+U(x',y')
\Bigr)\mu_{\rf}(dx')\mu_{\rf}(dy').
\end{align*}
Hence, taking $\esssup_{x, y}|\cdot|$ on both sides, 
\begin{align*}
\|U-\hat g\oplus \hat f\|_\infty 
&\le\esssup_{x, y, x', y'} \bigl|U(x,y)-U(x',y)-U(x,y')+U(x',y')\bigr|= \square(U),
\end{align*}
so the proof is complete. 

\subsection{Proof of Proposition \ref{proposition:L2_estimate_antisymmetric}}\label{proof:proposition:L2_estimate_antisymmetric}
We first claim $\|U\|_\oplus = \inf_{f}\|U-(-f)\oplus f\|_\infty$. To prove this, it suffices to show $\inf_{g, f} \|U- g\oplus f\|_\infty \ge \inf_{f} \|U- (-f)\oplus f\|_\infty$. Let us fix $g, f\in L^\infty$ and let $h=2^{-1}(-g + f)$. Note 
  $$
  \bigl((-h)\oplus h\bigr)(x,y) = 2^{-1} \bigl({g(x)-f(x) -g(y)+f(y)}\bigr) = 2^{-1} \bigl(g\oplus f - f \oplus g\bigr)(x, y)
  $$
  and the $\|U+f\oplus g\|_\infty = \|-U + g\oplus f\|_\infty$ since $U$ is antisymmetric. Therefore, by the triangle inequality, 
  \begin{align*}
    \inf_{f} \|U- (-f)\oplus f\|_\infty\le \|U-(-h)\oplus h\|_\infty \le 2^{-1} \|U-g\oplus f\|_\infty +2^{-1} \|U + f\oplus g\|_\infty =  \|U-g\oplus f\|_\infty. 
  \end{align*}
Taking $\inf_{g, f}$ on the RHS, we obtain the claim. 

By $\|U\|_\oplus = \inf_{f}\|U-(-f)\oplus f\|_\infty$, we have $\|U\|_\oplus\le \|U-(-\hat{f})\oplus \hat{f}\|_\infty$ for $\hat{f}(y)=\int\mu_\rf(dx) U(x, y)$. Furthermore, using $\|U\|_\oplus = \inf_{f}\|U-(-f)\oplus f\|_\infty$ and the same argument as that of \Cref{proof:proposition:L2_estimate}, one can show $3^{-1}\Delta(U)\le \|U\|_\oplus$ and 
$\|U-(\hat{f})\oplus \hat{f}\|_\infty \le \Delta(U)$. Thus, we complete the proof.

\section{Proof for \Cref{sec:mchf}}
\subsection{Proof of Theorem \ref{theorem:contraction}}\label{proof:theorem:contraction}
We prove $d_\tv(\mu\MP, \nu\MP)\le (1-e^{-2\|U\|_\oplus}) d_\tv(\mu, \nu)$ and $d_\tv(\mu\MP, \nu\MP)\le \|U\|_\oplus d_\tv(\mu, \nu)$ separately. 

\subsubsection{Proof of $d_\tv(\mu\MP, \nu\MP)\le (1-e^{-2\|U\|_\oplus}) d_\tv(\mu, \nu)$}
Let $L^1=L^1(\mu_{\rf})$ and $L_0^1 = \{h\in L^1: \int \mu_{\rf}(dy) h(y)=0\}$. We denote the $L_p$ norm as $\|f\|_p = \|f\|_{L^p}$. We also write $L^\infty = L^\infty(\mu_{\rf}\otimes\mu_{\rf})$ and $\|U\|_\infty = \|U\|_{L^\infty(\mu_{\rf}\otimes \mu_{\rf})}$ for any $U\in L^\infty(\mu_{\rf}\otimes \mu_{\rf})$ when the context is clear.
For any $U\in L^\infty(\mu_{\rf}\otimes \mu_{\rf})$, define the Markov transition density $P\in L^\infty(\mu_\rf\otimes\mu_\rf)$ as 
\begin{align}\label{eq:def_markov_density}
  P(x, y) = \frac{\exp(U(x, y))}{\int \mu_{\rf}(dy') \exp(U(x, y'))}.
\end{align}
For all $\mu, \nu \in \mathscr{P}_{\mu_\rf}$, we can write the TV distance as 
\begin{align*}
  d_\tv(\mu, \nu) = \frac12 \Bigl\|\frac{d\mu}{d \mu_\rf} - \frac{d\nu}{d \mu_\rf}\Bigr\|_1, \quad
  d_{\tv}(\mu\MP, \nu \MP) = \frac12 \Bigl\||P^\star  \Bigl(\frac{d\mu}{d \mu_\rf}- \frac{d\nu}{d \mu_\rf}\Bigr)\Bigr\|_{1}
\end{align*}
where $P^\star h (y) = \int \mu_\rf (dx) h(x) P(x, y)$ is the adjoint integral operator. 
Since $\frac{d\mu}{d \mu_\rf} - \frac{d\nu}{d \mu_\rf}\in L_0^1 = \{h\in L^1: \int \mu_{\rf}(dy) h(y)=0\}$, it suffices to show the contraction rate of $P^\star$ on the Banach space $(L_0^1, \|\cdot\|_1)$.  
\begin{lemma}\label{lemma:minorization}
Let $P$ be the Markov density given by \eqref{eq:def_markov_density}. Then, we have
    \begin{align*}
   \sup_{h\in L_0^1: \|h\|_1\le 1} \|P^\star h\|_{1} \le 1- \exp(-2\|U\|_\oplus)
  \end{align*}
  where $L_0^1 = \{h\in L^1: \int \mu_{\rf}(dy) h(y)=0\}$.
\end{lemma}
\begin{proof}
We prove this lemma using the standard minorization condition \citep{rosenthal1995minorization}.
Let us fix $g, f\in L^\infty$ and let $U^{g\oplus f}= U - g\oplus f$. 
 Define $p_f\in L^\infty$ as 
$$
p_f(y) = \frac{\exp(f(y))}{\int \mu_\rf(dy')\exp(f(y'))}
$$
so that $p_f$ is a valid density with respect to $\mu_\rf$, i.e., $p_f(\cdot)\ge 0$ and $\int \mu_\rf(dy) p_f(y)=1$. Since the factor $\exp(g(x))$ cancels between the numerator and the denominator in the definition of $P$ (see \eqref{eq:def_markov_density}), the Markov density $P$ can be rewritten as
\begin{align}
  P(x, y) =\frac{\exp(U(x, y)-g(x)-f(y)) e^{f(y)}}{\int \mu_{\rf}(dy') \exp(U(x, y')-g(x)-f(y'))e^{f(y')}}
          &= \frac{\exp(U^{g\oplus f}(x, y)) p_f(y)}{\int \mu_{\rf}(dy') \exp(U^{g\oplus f}(x, y')) p_f(y')} \label{eq:density_U_gf}
\end{align}
Then, for almost every $(x, y)$ with respect to the product measure $\mu_{\rf}\otimes \mu_{\rf}$, using $\int \mu_\rf(dy')p_f(y')=1$, 
\begin{align*}
  P(x, y) \ge \frac{\exp(-\|U^{g\oplus f}\|_\infty)}{\exp(\|U^{g\oplus f}\|_\infty)} \cdot \frac{p_f(y)}{\int \mu_{\rf}(dy')p_f(y')} = \delta^{g\oplus f}  \cdot p_f(y)
\end{align*}
where $\delta^{g\oplus f} = \exp(-2\|U^{g\oplus f}\|_\infty) \in (0,1]$. 
If we define $\tilde{P}$ as
$$
\tilde{P}(x, y) = \frac{
  P(x, y) - \delta^{g\oplus f} \cdot p_f(y)
}{1-\delta^{g\oplus f}}
$$
then $\tilde{P}$ satisfies $\tilde{P}(x, \cdot)\ge 0$ and $\int \mu_{\rf}(dy)\tilde{P}(x, y)=1$, that is, $\tilde{P}$ is a valid Markov transition density with respect to $\mu_{\rf}$. Rearranging this,
$$
P(x, y) = \delta^{g\oplus f} \cdot p_f(y)  + (1-\delta^{g\oplus f}) \tilde{P}(x, y). 
$$
For any $h\in L_0^1$, noting $\int\mu_\rf(dx) h(x)=0$, 
\begin{align*}
  P^\star h(y) &= \int \mu_\rf(dx) P(x, y) h(x) = (1-\delta^{g\oplus f}) \tilde{P}^\star h (y)
\end{align*}
Using the Data Processing Inequality: $\|P^\star h\|_1\le \|h\|_1$, which follows by 
$$
\|\tilde{P}^\star h\|_1 = \int\mu_\rf(dy) \Bigl|
  \int \mu_\rf(dx) \tilde{P}(x, y)  h(x)
\Bigr| \le \int \mu_\rf(dx)\mu_\rf(dy) \tilde{P}(x, y) |h(x)| = \int \mu_\rf(dx) |h(x)|
$$ 
where we have used the Fubini's theorem and the fact that $\tilde{P}$ is the Markov transition probability, we get 
$$
\sup_{h\in L_0^1: \|h\|_1 \le 1 } \|P^\star h\|_1  \le (1-\delta^{g\oplus f}) \sup_{h\in L_0^1: \|h\|_1 \le 1 } \|\tilde P^\star h\|_1 \le (1-\delta^{g\oplus f}).
$$
Since $\delta^{g\oplus f}=\exp(-2\|U-g\oplus f\|_\infty)$, 
taking $\inf_{g,f}$ on the RHS, we complete the proof. 
\end{proof}

\subsubsection{Proof of $d_\tv(\mu\MP, \nu\MP)\le \|U\|_\oplus d_\tv(\mu, \nu)$}
It suffices to show 
$\sup_{h\in L_0^1: \|h\|_1\le 1}  \|P^\star h\|_{1} \le \|U\|_\oplus$. 
We prove this using the following two standard lemmas.
\begin{lemma}\label{lemma:dobrushin-coefficient}
For any $P\in L^\infty(\mu_\rf \otimes \mu_\rf)$, 
  $$
    \sup_{h\in L_0^1: \|h\|_1\le 1}  \|P^\star h\|_{1} = \esssup_{x,x'}\frac{1}{2} \|P(x, \cdot)-P(x', \cdot)\|_{1}
  $$
where $\esssup_{x,x'}$ is taken with respect to the product measure $\mu_\rf\otimes \mu_\rf$. 
\end{lemma}
\begin{proof}
  This identity, when $\mu_\rf$ is a simplex over finite space, is known as classical Dobrushin ergodicity coefficient \citep{dobrushin1956central}. We prove this result for general probability measure $\mu_\rf$ using a standard argument; see \Cref{proof:lemma:dobrushin-coefficient}. 
\end{proof}

\begin{lemma}\label{lemma:lipschitz_softmax}
Fix a probability measure $\mu$ and define the map $\softmax:L^\infty(\mu)\to L^\infty(\mu)$ as 
$$
\softmax(g) \equiv \frac{\exp(g(y))}{\int \mu(dy') \exp(g(y'))}.
$$
Then, for all $g, \tilde{g}\in L^\infty(\mu)$, we have 
$
  \|\softmax(g)-\softmax(\tilde{g})\|_{L^1(\mu)} \le \|g-\tilde{g}\|_{L^\infty(\mu)}.
$
\end{lemma}
\begin{proof}
This is essentially the Lipschitz continuity of the softmax function. See \Cref{proof:lemma:lipschitz_softmax} for the proof. 
\end{proof}

Applying Lemma \ref{lemma:dobrushin-coefficient}, we have
\begin{align*}
    \sup_{h\in L_0^1: \|h\|_1\le 1}  \|P^\star h\|_{1} = \esssup_{x,x'}\frac{1}{2} \|P(x, \cdot)-P(x', \cdot)\|_{1}. 
\end{align*}
By \eqref{eq:density_U_gf} $P(x, y)$ can be written as
$$
P(x,y) = \frac{\exp(U^{g\oplus f}(x, y)) p_f(y)}{\int \mu_{\rf}(dy') \exp(U^{g\oplus f}(x, y')) p_f(y')}, \quad p_f(y) = \frac{\exp(f(y))}{\int \mu_\rf(dy')\exp(f(y'))}
$$
for any $g, f\in L^\infty$, 
where $U^{g\oplus f}=U-g\oplus f$. 
Letting $\mu_f(dy) = \mu_\rf(dy) p_f(y)$ be the exponentially tilted probability measure by $f$, using Lemma \ref{lemma:lipschitz_softmax} $\mu=\mu_f$ and $g(y)=U^{g\oplus f}(x, y)$ and $\tilde{g}(y)=U^{g\oplus f}(x', y)$, for each $x, x'$, 
\begin{align*}
  \|P(x, \cdot)-P(x', \cdot)\|_{1}\le \|U^{g\oplus f}(x, \cdot)-U^{g\oplus f}(x', \cdot)\|_{L^\infty(\mu_f)} = \|U^{g\oplus f}(x, \cdot)-U^{g\oplus f}(x', \cdot)\|_{L^\infty(\mu_\rf)},
\end{align*}
where the second equation follows from the fact that $\mu_f$ and $\mu_\rf$ are absolutely continuous with respect to each other. Thus, by the definition: $\esssup_y |f(y)| = \|f\|_{L^\infty(\mu_\rf)}$, combining all together, 
\begin{align*}
    \sup_{h\in L_0^1: \|h\|_1\le 1}  \|P^\star h\|_{1} &\le \esssup_{x,x'}\frac{1}{2} \esssup_y |U^{g\oplus f}(x, y)-U^{g\oplus f}(x', y)|\le
\|U^{g\oplus f}\|_\infty
\end{align*}
where we used the triangle inequality for the last inequality. 
Since $U^{g\oplus f}=U-g\oplus f$, taking $\inf_{g,f}$ on the RHS, we obtain $\sup_{h\in L_0^1: \|h\|_1\le 1}  \|P^\star h\|_{1} \le \|U\|_\oplus$ and the proof is complete.
Below, we prove the intermediate lemmas we used. 
\subsubsection{Proof of Lemma \ref{lemma:dobrushin-coefficient}}\label{proof:lemma:dobrushin-coefficient}
Now we claim the following standard identity:
\begin{lemma}\label{lemma:essosc}
  For all $g\in L^\infty$, we have 
\begin{align*}
  \sup_{h\in L_0^1: \|h\|_1\le 1} \Bigl|\int h(x) g(x) \mu_\rf(dx)\Bigr| = \frac12 \essosc(g)
\end{align*}
where $\essosc(g) = \esssup_{x} g(x) - \essinf_x g(x) = \esssup_{x,y}|g(x)-g(y)|$
\end{lemma}
\begin{proof}
  Fix $g\in L^\infty$ and a constant $c\in \R$. Note
  $\int h(x) g(x) \mu_\rf(dx) = \int h(x) (g(x)-c) \mu_\rf(dx)$ for $h\in L_0^1$. Then, 
  \begin{align*}
    \sup_{h\in L_0^1: \|h\|_1\le 1}\Bigl| \int h(x) g(x) \mu_\rf(dx)\Bigr| =   \sup_{h\in L_0^1: \|h\|_1\le 1}\Bigl| \int h(x) (g(x)-c) \mu_\rf(dx)\Bigr| \le \|g-c\bm{1}\|_\infty
  \end{align*}
  where $\|g-c\bm{1}\|_\infty = \esssup_x |g(x)-c|$. If we take $c=c_g=2^{-1}(\esssup_x g(x) + \essinf_x g(x))$, we have 
  $$
\|g-c\bm{1}\|_\infty = 2^{-1}(\esssup_x g(x)-\essinf_x g(x)) = 2^{-1} \essosc(g)
  $$
  which completes the proof of the upper bound. 

  To show a lower bound, let $M=\esssup_x g(x)$ and $m=\essinf_x g(x)$ so that $\essosc g = M-m$.  
  By the definition of $\esssup$ and $\essinf$, for any $\epsilon>0$, the following two events $\Omega_\epsilon^+$ and $\Omega_\epsilon^-$ have a positive measure under $\mu_\rf$:
  $$
    \Omega_\epsilon^+ = \{x: g(x) > M-\epsilon\}, \quad \Omega_\epsilon^- = \{x: g(x) < m+\epsilon\}. 
  $$
  Now we define the measurable function $h_\epsilon$ as 
  $
  h_\epsilon = \frac{\bm{1}\{\Omega_\epsilon^+\}}{2\mu_\rf(\Omega_\epsilon^+)}  - \frac{\bm{1}\{\Omega_\epsilon^-\}}{2\mu_\rf(\Omega_\epsilon^-)}.
  $
  Notice that this is well-defined since $\mu_\rf(\Omega_\epsilon^{\pm})>0$, and satisfies $h_\epsilon\in L_0^1$, $\|h_\epsilon\|_1 \le 1$. Then, 
  \begin{align*}
      \int h_\epsilon(x) g(x) \mu_\rf(dx) &= \frac{1}{2\mu_\rf(\Omega_\epsilon^+)}\int \bm{1}\{\Omega_\epsilon^+\} g(x) \mu_\rf(dx) - \frac{1}{2\mu_\rf(\Omega_\epsilon^-)}\int \bm{1}\{\Omega_\epsilon^-\} g(x) \mu_\rf(dx)\\
      &\ge 2^{-1} (M-\epsilon) - 2^{-1}(m+\epsilon)\\
      &= 2^{-1} (M-m) - \epsilon.
  \end{align*}
  Taking $\epsilon\to 0$, we get the lower bound. 
\end{proof}

Let us prove Lemma \ref{lemma:dobrushin-coefficient} using Lemma \ref{lemma:essosc}. By the duality of $L^1$-$L^\infty$ norms, 
  \begin{align*}
\|P^\star h\|_{1}=  \sup_{\|f\|_\infty \le 1} \Bigl|\int\mu_\rf(dy) f(y) P^\star h(y) \Bigr|. 
  \end{align*}
 Using Fubini's theorem, for all $f\in L^\infty$, 
  \begin{align*}
\int \mu_\rf(dy) f(y) P^\star h(y)&=    \int \mu_\rf(dy) f(y) \int \mu_\rf(dx) P(x, y)h(x)
   =\int \mu_\rf(dx) h(x) P f(x)
  \end{align*}
  where $P f(x) \equiv \int\mu_\rf(dy) P(x, y) f(y)$. Then, applying Lemma \ref{lemma:essosc} with $g=Pf$, swapping the order of $\sup_{h\in L_0^1: \|h\|_1=1}$ and $\sup_{\|f\|_\infty\le 1}$, 
  \begin{align*}
    \sup_{h\in L_0^1: \|h\|_1\le 1}\|P^\star h\|_{1} = \sup_{\|f\|_\infty\le 1}  \sup_{h\in L_0^1: \|h\|_1\le 1} \Bigl|\int\mu_\rf(dy) f(y) P^\star h(y) \Bigr| = \frac12 \sup_{\|f\|_\infty \le 1}\essosc (Pf)
  \end{align*}
  where 
  $$
    \sup_{\|f\|_\infty \le 1}\essosc (Pf) =  \sup_{\|f\|_\infty \le 1}  \esssup_{x,x'} \Bigl| \int \mu_\rf(dy) f(y) \bigl(P(x, y) - P(x', y)\bigr)\Bigr|
  $$
Swapping the order of $\sup_{\|f\|_\infty\le 1}$ and $\esssup_{x,x'}$, using the duality of $L^1$-$L^\infty$ norms, we get
\begin{align*}
  \sup_{\|f\|_\infty \le 1}\essosc (Pf) =  \esssup_{x,x'} \|P(x,\cdot)-P(x',\cdot)\|_1,
\end{align*}
so the proof is complete. 

\subsubsection{Proof of Lemma \ref{lemma:lipschitz_softmax}}\label{proof:lemma:lipschitz_softmax}
  Let $h(x) = g(x)-\tilde{g}(x)$ and define the interpolator $g_t\in L^\infty$ for $t\in[0,1]$ as $g_t \equiv \tilde{g} + t h$
  so that $g_0 = \tilde{g}$ and $g_1=g$. Now we define the tilted probability measure $\mu_t$ as 
  $$
  \mu_t(dy) \equiv \softmax(g_t)(y) \mu(dy).
  $$
  Note that $\mu_t$ and $\mu$ are absolutely continuous with respect to each other. 
Then, 
    \begin{align*}
  \|\softmax(g)-\softmax(\tilde{g})\|_{1} = \sup_{\|f\|_{L^\infty(\mu)}\le 1} \Bigl|
      \int \mu_\rf(dx) f(x) (\softmax(g_1)(x)-\softmax(g_0)(x))
     \Bigr| = \sup_{\|f\|_{L^\infty(\mu)}\le 1} \Bigl|
    \phi_f(1) - \phi_f(0)
     \Bigr|
  \end{align*}
where 
  $$
  \forall t\in [0,1], \quad 
  \phi_f(t)  \equiv \int f(x)  \mu_t(dx)  \equiv \int \mu(dx) f (x) \frac{\exp(g_t(x))}{\int \mu(dx') \exp(g_t(x'))}. 
  $$
  Here, since $g_t = \tilde{g} + th$ with $\tilde{g}, h\in L^\infty$, $\phi_f$ is $C^1$ on $[0,1]$ and the derivative is given by 
  \begin{align*}
    \dot{\phi}_f(t) 
    &= \int \mu_t(dx) f(x) h(x) - \Bigl(\int \mu_t(dx) f(x)\Bigr)\cdot \Bigl(\int \mu_t(dx) h(x)\Bigr)= \text{Cov}_{X\sim \mu_t}(f(X), h(X))
  \end{align*}
  Using the Cauchy--Schwarz inequality $|\text{Cov}(X, Y)|\le \sqrt{\text{Var}(X) \text{Var}(Y)}\le \|X\|_{\infty} \|Y\|_{\infty}$, noting that $\mu_t$ and $\mu$ are absolutely continuous with respect to each other, 
  we obtain the uniform upper bound:
  $$
    |\dot{\phi}_f(t)|\le \|f\|_{L^\infty(\mu_t)} \cdot \|h\|_{L^\infty(\mu_t)} = \|f\|_{L^\infty(\mu)} \|h\|_{L^\infty(\mu)}. 
  $$
  Therefore, by the fundamental theorem of calculus, 
  $$
  \|\softmax(g)-\softmax(\tilde{g})\|_{1} = \sup_{\|f\|_{L^\infty(\mu)}\le 1} \Bigl|
    \phi_f(1) - \phi_f(0)
     \Bigr| \le \sup_{\|f\|_{L^\infty(\mu)}\le 1}  \int_0^1 |\dot{\phi}_f(t)| dt \le \|h\|_{L^\infty(\mu)}. 
  $$
  This completes the proof.

\section{Proof for \Cref{sec:nlhf_algorithm}}
\subsection{Proof of Theorem \ref{theorem:convergence_nl_kl}}\label{proof:theorem:convergence_nl_kl}
We use the following standard convergence guarantee for proximal mirror descents. 
\begin{lemma}\label{lemma:mirror_descent}
Let $X$ be a convex subset of a Hilbert space $(\mathcal{H}, \langle\rangle)$ and consider the minimization problem
$$
\min_{p\in X} \ell(p) + w(p),
$$
where both $\ell$ and $w$ are differentiable and convex, and furthermore $\ell$ is relatively $L$-smooth with respect to $w$ in the sense that 
\begin{align*}
\forall p,q\in X, \qquad D_\ell(p\|q) \le L D_w(p\|q),
\end{align*}
where $D_h(p\|q) = h(p)-h(q)-\langle \nabla h(q), p-q\rangle$
denotes the Bregman divergence associated with a convex function $h\in\{\ell,w\}$. Now we consider the (proximal) mirror descent iteration
\begin{align*}
  p_{t+1}
  = \argmin_{p\in X}
  \Bigl\{
    \langle \nabla \ell(p_t), p\rangle
    + w(p)
    + \eta^{-1} D_w(p\|p_t)
  \Bigr\}.
\end{align*}
Then for any $\eta \le 1/L$, letting $p_*\in \argmin \ell(p)+w(p)$ be the minimizer, we have 
\begin{align*}
  D_w(p_*\|p_t)
  \le
  (1+\eta)^{-t} D_w(p_*\|p_0).
\end{align*}
\end{lemma}
\begin{proof}
The proof relies on the relative smoothness framework introduced in \cite{haihao2018relatively}. To keep our paper self-contained, we prove this lemma in  \Cref{proof:lemma:mirror_descent}.  
\end{proof} 
Let us show Theorem \ref{theorem:convergence_nl_kl} using Lemma \ref{lemma:mirror_descent}. Let 
$p_\nl = \frac{d \mu_\nl}{d\mu_\rf}$ and $q_\nl = \frac{d \nu_\nl}{d\mu_\rf}$ be the densities of the NLHF solutions. 
Consider the density simplex
$
\Delta^\infty = \{f\in L^\infty: f(\cdot)\ge 0, \ \int \mu_\rf(dy)f(y)=1 \},
$ 
which is a convex subset of the Hilbert space $L^2$ with the usual inner product $\langle f,g\rangle = \int \mu_\rf(dx) g(x) f(x)$, and define the linear operator $h\mapsto Uh$ by $Uh(x) = \int \mu_\rf(dy) U(x, y) h(y)$, and let $U^\star$ be its adjoint operator. Then, $(p_\nl, q_\nl)$ solves the following minimax optimization problem:
\begin{align*}
\argmax_{p\in \Delta^\infty} \min_{q\in \Delta^\infty} \langle q, U p\rangle + w(q) - w(p) \quad \text{where} \quad w(p)=\int \mu_\rf(dy) p(y)\log p(y)
\end{align*}
If we define the convex function $\ell:\Delta^\infty\to\R$ as $ \ell(p) = - (\min_{q\in \Delta^\infty} \langle q, U p\rangle + w(q))$
then $p_\nl$ is the solution to the following convex optimization problem:
$$
p_\nl \in \argmin_{p\in \Delta^\infty} \ell(p) + w(p). 
$$
Now we claim that $\ell$ is relatively $\|U\|_\oplus^2$-smooth with respect to $w$. 
\begin{lemma}\label{lemma:ell_p_smoothness}
Fix $U\in L^\infty(\mu_\rf\otimes\mu_\rf)$, and for each $p\in \Delta^\infty$, let $\ell(p)=- (\min_{q\in \Delta^\infty} \langle q, U p\rangle + w(q))$ and $q_p\in \argmin_{q\in \Delta^\infty} \langle q, U p\rangle + w(q)$. Then, 
  \begin{enumerate}
    \item $\|q_p-q_{p'}\|_1\le \|U\|_\oplus \|p-p'\|_1$ for all $p, p'\in \Delta^\infty$. 
    \item $\ell$ is $\|U\|_\oplus^2$-smooth with respect to $w$, i.e., $D_\ell(p||q)\le \|U\|_\oplus^2 D_w(p||q)$
  \end{enumerate}
  where $\|U\|_\oplus = \inf_{g,f\in L^\infty}\|U-g\oplus f\|_\infty$. 
\end{lemma}
\begin{proof}
 See \Cref{proof:lemma:ell_p_smoothness} for the proof. 
\end{proof}
Thus, combined with Lemma \ref{lemma:ell_p_smoothness}-(2), 
 we may apply Lemma \ref{lemma:mirror_descent} with
$$
L=\|U\|_\oplus^2,
\qquad
w(p)=\int \mu_\rf(dy) p(y)\log p(y), 
\qquad
\ell(p) = - \Bigl(\min_{q\in \Delta} \langle q, U p\rangle + w(q)\Bigr)
$$
and conclude that for any $\eta \le 1/\|U\|_\oplus^2$, the iteration $p_t$ defined by
\begin{align*}
  p_{t+1}
  &=  \argmin_p
  \Bigl\{
    \langle \nabla \ell(p_t), p\rangle
    + w(p)
    +
\eta^{-1}D_w(p\|p_t)
  \Bigr\}
\end{align*}
satisfies $D_w(p_*\|p_t)
  \le
 (1+\eta)^{-t}\cdot D_w(p_*\|p_0)$. Here, by the envelope theorem, $\nabla \ell(p_t)=- U^\star q_t$ where 
 $q_t=
  \argmin_q \langle q, U p_t\rangle + w(q)$. Thus, the total algorithm can be written as 
  $$
q_t=
  \argmin_q \Bigl\{ \langle q, U p_t\rangle + w(q) \Bigr\}, \quad p_{t+1} =
  \argmin_p
  \Bigl\{
    - \langle U^\star q_t, p\rangle
    + w(p)
    +
   \eta^{-1}D_w(p\|p_t)
  \Bigr\}
  $$
Finally, by the change of variable $(p, q)\mapsto (\mu, \nu)$ where
$\mu(dy)=p(y)d\mu_\rf(dy)$ and $\nu(dx)=q(x)\nu_\rf(dx)$, 
noting that $w(p) = \kl(\mu|\mu_\rf)$ and $D_w(p|q)=\kl(\mu|\nu)$, we obtain Theorem \ref{theorem:convergence_nl_kl}. 

\subsubsection{Proof of Lemma \ref{lemma:mirror_descent}}\label{proof:lemma:mirror_descent}
Let $G(p)=\ell(p)+w(p)$ and define the surrogate objective
$$
G_t(p)
=
\ell(p_t)
+
\langle \nabla \ell(p_t), p-p_t\rangle
+
\eta^{-1} D_w(p\|p_t)
+
w(p).
$$
Then $p_{t+1}\in \argmin_{p\in X} G_t(p)$.
By the definition of Bregman divergence,
\begin{align*}
  G(p)
  &=\ell(p)+w(p) \\
  &=\ell(p_t)+
  \langle \nabla \ell(p_t), p-p_t\rangle+
  D_\ell(p\|p_t)
  +
  w(p) && \text{by }D_\ell(p\|p_t) = \ell(p)-\ell(p_t)-
   \langle \nabla \ell(p_t), p-p_t\rangle
\\
  &=
  G_t(p) + D_\ell(p\|p_t) - \eta^{-1} D_w(p\|p_t).
\end{align*}
Using the assumption $D_\ell(p\|p_t)\le L D_w(p\|p_t)$ and the non-negativity of the Bregman divergence $D_\ell(p\|p_t)\ge 0$, the approximation error $G(p)-G_t(p)$ can be controlled as 
\begin{align}
\label{eq:sandwich}
-\eta^{-1} D_w(p\|p_t)
\le
G(p)-G_t(p)
\le
-(\eta^{-1}-L)D_w(p\|p_t).
\end{align}
Now, since $p_{t+1}\in \argmin G_t(p)$ and $G_t$ is convex and differentiable, we have
\begin{align*}
  \langle \nabla G_t(p_{t+1}), q-p_{t+1}\rangle \ge 0
  \qquad \forall q\in X.
\end{align*}
Equivalently, by the definition of Bregman divergence, 
\begin{align*}
  D_{G_t}(q\|p_{t+1})
  \le
  G_t(q)-G_t(p_{t+1}),
  \qquad \forall q\in X.
\end{align*}
Applying this with $q=p_*$, we get
\begin{align*}
  (\eta^{-1}+1) D_w(p_*\|p_{t+1})
  &=
  D_{G_t}(p_*\|p_{t+1})
  \qquad\text{since $G_t(p)=(\eta^{-1}+1)w(p)+\text{linear terms}$} \\
  &\le
  G_t(p_*)-G_t(p_{t+1}) \\
  &\le
  \Bigl(F(p_*)+\eta^{-1}D_w(p_*\|p_t)\Bigr)
  -
  \Bigl(F(p_{t+1})+(\eta^{-1}-L)D_w(p_{t+1}\|p_t)\Bigr)
  \qquad\text{by \eqref{eq:sandwich}} \\
  &\le
  \eta^{-1} D_w(p_*\|p_t)
  -
  (\eta^{-1}-L)D_w(p_{t+1}\|p_t)
  \qquad\text{since $F(p_*)\le F(p_{t+1})$} \\
  &\le
  \eta^{-1} D_w(p_*\|p_t),
\end{align*}
where the last step uses $\eta^{-1}\ge L$ and $D_w(\cdot|\cdot)\ge 0$. Dividing both sides by $\eta^{-1}+1$ 
and iterating completes the proof.

\subsubsection{Proof of Lemma \ref{lemma:ell_p_smoothness}}\label{proof:lemma:ell_p_smoothness}
Fix $g, f\in L^\infty$ and let $U^{g\oplus f}=U-g\oplus f$. Then, we can rewrite $\ell(p)$ as 
$$
\ell(p)= -\Bigl(\min_{q\in \Delta^\infty}
\langle q, U^{g\oplus f}p\rangle + \langle q, g\rangle + \langle p,f\rangle + w(q)
\Bigr).
$$
By the envelope theorem, $\ell$ is differentiable with its derivative given by 
\begin{align}\label{eq:envelope}
\nabla \ell(p) = -(U^{g\oplus f})^\star q_p + f, 
\end{align}
Here, $q_p$ is the solution of $\min_q \langle q, U^{g\oplus f}p\rangle + \langle q, g\rangle + \langle p,f\rangle + w(q)$, 
which can be written explicitly as 
\begin{align}\label{eq:q_p}
  q_p(x) = \frac{\exp(-U^{g\oplus f} p (x)) e^{-g(x)}}{\int \mu_\rf(dx') \exp(-U^{g\oplus f} p (x')) e^{-g(x')}}.
\end{align}
Then, for any $p, p'\in \Delta^\infty$, using $\|Uf\|_\infty \le \|U\|_\infty \|f\|_1$ by the Cauchy--Schwarz inequality and Lemma \ref{lemma:lipschitz_softmax} with $\mu_{-g}(dx) \propto e^{-g(x)}\mu_\rf(dx)$, noting that $\mu_{-g}$ and $\mu_\rf$ are absolutely continuous with respect to each other, 
\begin{align*}
  \|q_p-q_{p'}\|_1 &\le \|-U^{g\oplus f} p + U^{g\oplus f}p'\|_{L^\infty(\mu_{-g})}  &&\text{by \eqref{eq:q_p} and Lemma \ref{lemma:lipschitz_softmax} with $\mu_{-g}(dx) \propto e^{-g(x)}\mu_\rf(dx)$}\\
                  &=\|-U^{g\oplus f} p + U^{g\oplus f}p'\|_{\infty} && \|\cdot\|_{L^\infty(\mu_{-g})}=\|\cdot\|_{L^\infty(\mu_\rf)}\\
                  &\le \|U^{g\oplus f}\|_\infty \|p-p'\|_1 && \text{by Hölder's inequality}.
\end{align*}
Since $g,f\in L^\infty$ are arbitrary, taking $\inf_{g,f\in L^\infty}$ on the RHS, we complete the proof of $\|q_p-q_{p'}\|_1\le \|U\|_\oplus \|p-p'\|_1$. 

Now, combined with the derivative form \eqref{eq:envelope},
\begin{align*}
  \|\nabla \ell(p)-\nabla\ell(p')\|_\infty &= \|-(U^{g\oplus f})^\star q_p + (U^{g\oplus f})^\star q_{p'}\|_\infty && \text{by \eqref{eq:envelope}}\\
  &\le \|(U^{g\oplus f})^\star\|_\infty \|q_p-q_{p'}\|_1 && \text{by Hölder's inequality}\\
  &\le \|U^{g\oplus f}\|_\infty\|U\|_\oplus \|p-p'\|_1
\end{align*}
Again, since $g,f\in L^\infty$ are arbitrary, taking $\inf_{g,f\in L^\infty}$ on the RHS, we get  
\begin{align}\label{eq:derivative_ell_lipschitz}
  \forall p, p' \in \Delta^\infty, \quad
    \|\nabla \ell(p)-\nabla\ell(p')\|_\infty\le \|U\|_\oplus^2 \|p-p'\|_1.
\end{align}
Therefore, by the fundamental theorem of calculus, the Bregman divergence $D_\ell(p\|q)=
  \ell(p)-\ell(q)-\langle \nabla \ell(q), p-q\rangle$ is bounded from above as 
\begin{align*}
  D_\ell(p\|q)
  &=
  \int_0^1
  \langle \nabla \ell(tp+(1-t)q)-\nabla \ell(q), p-q\rangle
  dt \\
  &\le
  \int_0^1  \|U\|_\oplus^2 \|t(p-q)\|_1 \|p-q\|_1 dt  && \text{by H\"{o}lder's inequality and \eqref{eq:derivative_ell_lipschitz}}\\
  &=
  \frac12 \|U\|_\oplus^2 \|p-q\|_1^2.
\end{align*}
Finally, by Pinsker's inequality $\tv\le \sqrt{\kl/2}$, which is translated to $2^{-1}\|p-q\|_1\le \sqrt{D_w(p||q)/2}$ in our setup, the last term is bounded from above by $\|U\|_\oplus^2 D_w(p||q)$. Thus, we obtain $D_\ell(p\|q)\le \|U\|_\oplus^2 D_w(p||q)$ and complete the proof.

\subsection{Proof of Theorem \ref{theorem:convergence_nl_tv}}
Let $p_t=\frac{d\mu_\nl^t}{d\mu_\rf}$, $q_t=\frac{d\nu_\nl^t}{d\mu_\rf}$, $p_\nl = \frac{d\mu_\nl}{d\mu_\rf}$, and $q_\nl = \frac{d\nu_\nl}{d\mu_\rf}$ so that 
\begin{align*}
  q_t &= \argmin_{q\in \Delta^\infty} \langle q, U p_t\rangle + w(q), &&q_{\nl}= \argmin_{q\in \Delta^\infty} \langle q, U p_\nl \rangle + w(q),\\
  p_{t+1} &= \argmax_{p\in \Delta^\infty} \langle q_t,  U p\rangle - w(p), &&  p_\nl = \argmax_{p\in \Delta^\infty} \langle q_\nl,  U p\rangle - w(p)
\end{align*}
By Lemma \ref{lemma:ell_p_smoothness}-(1) and the invariance $\|U^\star\|_\oplus=\|U\|_\oplus$, we have $\|q_t - q_\nl\|_1 \le \|U\|_\oplus \|p_t-p_\nl\|_1$ and $\|p_{t+1} - p_\nl\|_1 \le \|U\|_\oplus \|q_t-q_\nl\|_1$. Hence, the proof is complete. 

\section{Fr\'echet Differentiability and Stability}
This section provides the technical ingredients needed for the analysis of the alignment dynamics in \Cref{sec:alignment_dynamics}, in particular Proposition \ref{proposition:derivative_additive_antisymmetric} and Theorem \ref{theorem:asymptotics_iterations}. Our main goal is to establish the uniform Fr\'echet differentiability result stated in Theorem \ref{theorem:uniform_frechet_differentiability}, together with an explicit derivative formula. To this end, we organize this section as follows:
\begin{itemize}
\item \Cref{app:frechet_differentiability}: Fr\'echet differentiability and derivative formulas for MCHF and NLHF; see Theorems \ref{theorem:frechet_derivative_mc} and \ref{theorem:frechet_derivative_nl}, respectively.
\item \Cref{sec:perturbation}: The stability estimate introduced in Theorem \ref{theorem:stability}.
\item \Cref{app:uniform_frechet_differentiability}: Uniform Fr\'echet differentiability; namely, the proof of Theorem \ref{theorem:uniform_frechet_differentiability} for MCHF and NLHF.
\end{itemize}
Readers primarily interested in the consequences for the alignment dynamics may read the statements of Theorems \ref{theorem:frechet_derivative_mc} and \ref{theorem:frechet_derivative_nl}, skip the other theorems and technical proofs, and proceed directly to \Cref{sec:asymptotics_general_U}.

\subsection{Fr\'echet differentiability}\label{app:frechet_differentiability}

\subsubsection{Fr\'echet differentiability of MCHF}
For each $U\in L^\infty(\mu_\rf\otimes\mu_\rf)$, define $P_U\in L^\infty(\mu_\rf\otimes \mu_\rf)$ as 
$$
P_U(x,y)
=
\frac{\exp(U(x,y))}
{\int \mu_\rf(dy')\exp(U(x,y'))}
$$
Let $P_U^\star:L^1(\mu_\rf)\to L^1(\mu_\rf)$ be the adjoint operator of $P_U$, that is, 
$
P_U^\star f(y)
=
\int \mu_\rf(dx)f(x)P_U(x,y).
$
Let $\mu_\mc(U)$ be the unique stationary distribution of the Markov kernel $\MP_U(x,dy)=P_U(x,y)\mu_\rf(dy)$ (the existence and uniqueness follow from Corollary \ref{corollary:stationary_dist})
and let
$p_\mc(U)=\frac{d\mu_\mc(U)}{d\mu_\rf}$
be its density. Then, for each $U\in L^\infty$, $p=p_\mc(U)$ is the unique solution to the system:
$$
p\in L^1, \qquad 
P_U^\star p = p, 
\quad
\int \mu_\rf(dy)p(y)=1.
$$
Now we claim that $U\mapsto p_\mc(U)$ is Fr\'echet differentiable.
\begin{theorem}[Fr\'echet differentiability of MCHF]\label{theorem:frechet_derivative_mc}
The map
$U\mapsto p_\mc(U)$
is continuously Fr\'echet differentiable as a map from $L^\infty(\mu_\rf\otimes \mu_\rf)$ to $L^1$. 
  More precisely, for
each $U\in L^\infty(\mu_\rf\otimes\mu_\rf)$, for any sequence of $E\in L^\infty(\mu_\rf\otimes \mu_\rf)$ such that $\|E\|_\infty\to 0$, 
$$
\lim_{\|E\|_\infty\to 0}\frac{\|p_\mc(U+E)-p_\mc(U)-Dp_\mc(U)[E]\|_1}{\|E\|_\infty} = 0,
$$
where the Fr\'echet derivative
$Dp_\mc(U)[E]\in L_0^1$ is given by
\begin{align*}
  &Dp_\mc(U)[E] \\
  &=
\bigl((I-P_U^\star)\mid_{L_0^1}\bigr)^{-1} 
\Bigl(
y\mapsto
\int \mu_\rf(dx)
p_\mc(U)(x)P_U(x,y)
\bigl[
E(x,y)
-
\int \mu_\rf(dy')P_U(x,y')E(x,y')
\bigr]
\Bigr),
\end{align*}
where
$L_0^1
\equiv
\{
f\in L^1(\mu_\rf):
\int \mu_\rf(dy)f(y)=0
\}$.
\end{theorem}
\begin{proof}
Denote $p(U)=p_\mc(U)$, and define
$f(U)=p(U)-\bm 1$. 
Then $f(U)\in L_0^1$ and $f=f(U)$ is the unique solution to the system
$$
f\in L_0^1, \quad F(f, U) = 0 
$$
where $F:L_0^1\times L^\infty(\mu_\rf\otimes\mu_\rf)\to L_0^1$ is defined by
$$
F(f,U)
=
(I-P_U^\star)(f+\bm 1). 
$$
Here, the image of $F$ is included in $L_0^1$ since $\int \mu_\rf(dy)P_U(x,y)=1$ for every $x$.


We first show that $F$ is continuously Fr\'echet differentiable. The key
point is that the softmax map
$U\mapsto P_U$
is continuously Fr\'echet differentiable from
$L^\infty(\mu_\rf\otimes\mu_\rf)$ to
$L^\infty(\mu_\rf\otimes\mu_\rf)$. Its derivative at $U$ in the direction
$E\in L^\infty(\mu_\rf\otimes\mu_\rf)$ is
\begin{align}\label{eq:derivative_P_U}
  DP_U[E](x,y)
=
P_U(x,y)
\Bigl[
E(x,y)
-
\int \mu_\rf(dy')P_U(x,y')E(x,y')
\Bigr].
\end{align}
Indeed, this follows by differentiating the numerator and denominator in
the normalized exponential map. 
Consequently, $F$ is continuously Fr\'echet differentiable with
derivatives given by 
\begin{align}\label{eq:D_f_D_U}
  \begin{split}
      \forall h \in L_0^1, \quad D_fF(f,U)[h]
&=
(I-P_U^\star)h\\
\forall E \in L^\infty(\mu_\rf\otimes \mu_\rf),\quad D_U F(f, U)[E]
&=-(DP_U[E])^\star(f+\bm 1).
  \end{split}
\end{align}

Next we show that $D_fF(f(U),U)$ has a bounded inverse on $L_0^1$. By
Lemma \ref{lemma:minorization}, 
$$
\forall h\in L_0^1, \quad 
\|P_U^\star h\|_1
\le
(1-e^{-2\|U\|_\oplus})\cdot \|h\|_1
$$
Since $(1-e^{-2\|U\|_\oplus})\in [0,1)$, $(I-P_U^\star)\mid_{L_0^1}$
has a bounded inverse given by the Neumann series
$((I-P_U^\star)\mid_{L_0^1})^{-1}
=\sum_{t=0}^\infty (P_U^\star)^t$, with its operator norm bounded as 
\begin{align}\label{eq:I_P_inverse_operator_norm}
  \sup_{h\in L_0^1: \|h\|_1\le 1}
\Bigl\|
\bigl((I-P_U^\star)\mid_{L_0^1}\bigr)^{-1} h
\Bigr\|_1\le 
\sum_{t=0}^\infty (1-e^{-2\|U\|_\oplus})^t = e^{2\|U\|_\oplus} <+\infty.
\end{align}
Therefore, by the implicit function theorem for the Banach space (cf. \cite[Theorem 3.4.10]{krantz2013implicit}), there
exists a neighborhood of $U$ in
$L^\infty(\mu_\rf\otimes\mu_\rf)$ such that $U\mapsto f(U)$ is continuously Fr\'echet
differentiable on the neighborhood. Since $f(U)=p(U)-1$, $U\mapsto p(U)$ is also continuously Fr\'echet differentiable at $U$.

It remains to compute the derivative. Differentiating the identity
$F(f(U),U)=0$
in the direction $E$ gives
\begin{align*}
  Df(U)[E]
&=
-
\Bigl(D_fF(f(U),U)\mid_{L_0^1}\Bigr)^{-1}
D_UF(f(U),U)[E].
\end{align*}
Substituting \eqref{eq:D_f_D_U} and \eqref{eq:derivative_P_U} to the above equation, noting $f(U)+\bm 1=p(U)$, 
we have 
\begin{align*}
&Df(U)[E]\\
&=
\Bigl((I-P_U^\star)\mid_{L_0^1}\Bigr)^{-1}
\Bigl(
y\mapsto
\int \mu_\rf(dx)
p(U)(x)P_U(x,y)
\Bigl[
E(x,y)
-
\int \mu_\rf(dy')P_U(x,y')E(x,y')
\Bigr]
\Bigr).
\end{align*}
Finally, since $p(U)=f(U)+\bm 1$, we have
$Dp(U)[E]=Df(U)[E]$. 
This proves the claimed formula for the Fr\'echet derivative of $p_\mc(U)$. 
\end{proof}

\subsubsection{Fr\'{e}chet differentiability of NLHF}
Let $\Delta^\infty \subset L^\infty$ be the density simplex defined as 
$$
\Delta^\infty = \{p\in L^\infty: p(\cdot)\ge 0, \int \mu_\rf(dz) p(z)=1\}. 
$$ 
Notice that for any $p\in \Delta^\infty$, it induces a tilted probability measure $\mu(dy)=p(y) \mu_{\rf}(dy)$. We define the operator $\softmax$ taking values in $\Delta^\infty$ as
$$
\softmax: L^\infty \to \Delta^\infty, \quad f\mapsto \softmax(f)(x) = \frac{\exp(f(x))}{\int \mu_{\rf}(dx') \exp(f(x'))} 
$$
For each $p \in \Delta^\infty$, define the linear operator $J_{p}: L^2 \to L^2_0$ as $J_p=\text{Diag}(p)-pp^\star$ or more precisely
  \begin{align*}
      J_{p}: L^2 \to L^2_0, \quad &&f \mapsto J_p f(z)  = p(z) f(z) - p(z)\int \mu_{\rf}(dz') p(z') f(z')
  \end{align*}
where 
$L_0^2=\{h\in L^2: \int \mu_\rf(dz)h(z)=0\}\subset L^2$

\begin{lemma}
  For any $p\in \Delta^\infty$,  $J_{p}$ is bounded, self-adjoint, and positive semidefinite operator on $L^2$ such that 
  $$
  \|J_p\|_{2\to 2} = \lim_{\|h\|_2\le 2} \|J_p h\|_2 \le \|p\|_\infty
  $$
\end{lemma}
\begin{proof}
  For any $f, g\in L^2$, letting $\mu_p(dz)=p(z)\mu_{\rf}(dz)$ be the tilted probability measure by the density $p\in \Delta^\infty$, one can show $    \langle f, J_p g\rangle
    =  \text{Cov}_{Z\sim \mu_p}(f(Z), g(Z))$. This implies 
 $J_{p}$ is self-adjoint and positive semidefinite operator.   Moreover, the operator norm $\|J_{p}\|_{L^2\to L^2}$ is bounded by $\|p\|_\infty$ because the Cauchy--Schwarz inequality yields
  \begin{align*}
      |\langle f, J_p g\rangle| 
    \le \sqrt{\textrm{Var}_{Z\sim \mu_p} f(Z)} \cdot \sqrt{\textrm{Var}_{Z\sim \mu_p} g(Z)}
    \le \|f\|_{L^2(\mu_p)}\cdot \|g\|_{L^2(\mu_p)}
    \le \|p\|_\infty \cdot \|f\|_{L^2(\mu_{\rf})} \cdot \|g\|_{L^2(\mu_{\rf})}
  \end{align*}
  where the last inequality follows from 
  $
  \|f\|_{L^2(\mu_p)}\le \|\frac{d\mu_p}{dp_{\rf}}\|_{L^\infty(\mu_{\rf})}^{1/2}\cdot \|f\|_{L^2(\mu_{\rf})} = \|p\|_\infty^{1/2} \|f\|_2
  $. 
\end{proof}
 Therefore, by the spectral theorem, there exists a bounded self-adjoint positive semidefinite operator $T$ such that  $T\cdot T= J_p$. We denote such $T$ by $J_p^{1/2}$.  

\begin{theorem}[Fr\'echet differentiability of NLHF]\label{theorem:frechet_derivative_nl}
Let $U\in L^\infty(\mu_\rf\otimes\mu_\rf)$, and let $(p_\nl(U),q_\nl(U))$ be the
density pair, which is the unique solution to the NLHF fixed-point equations
$$
p, q\in  \Delta^\infty, \quad 
p=\softmax(U^\star q),
\qquad
q=\softmax(-Up).
$$
Then the map
$
U\mapsto p_\nl(U)
$
is continuously Fr\'echet differentiable as a map from
$L^\infty(\mu_\rf\otimes\mu_\rf)$ to $L^2(\mu_\rf)$.
More precisely, for each $U\in L^\infty(\mu_\rf\otimes \mu_\rf)$, and for any sequence of $E\in L^\infty(\mu_\rf\otimes \mu_\rf)$ such that $\|E\|_\infty\to 0$, we have 
$$
\lim_{\|E\|_\infty\to 0}
\frac{
\|p_\nl(U+E)-p_\nl(U)-Dp_\nl(U)[E]\|_2}{\|E\|_\infty} 
=0
$$
where the
Fr\'echet derivative
$Dp_\nl(U)[E]$
belongs to $L_0^2$ and is given by
\begin{align*}
Dp_\nl(U)[E]
&=
\bigl(I+J_{p(U)}U^\star J_{q(U)}U\bigr)^{-1}
\bigl(
J_{p(U)}E^\star q(U)
-
J_{p(U)}U^\star J_{q(U)}E p(U)
\bigr),
\end{align*}
where $(p(U), q(U))=(p_\nl(U), q_\nl(U))$ on the RHS.  
Here $(I+J_{p(U)}U^\star J_{q(U)}U)$ has a bounded inverse on $L^2$. 
\end{theorem}

\begin{proof}
Let
$x(U)=p(U)-\bm 1$ and $y(U)=q(U)-\bm 1$. 
Then $x(U),y(U)\in L_0^2$. Define the Implicit function
$
F:L_0^2\times L_0^2\times L^\infty(\mu_\rf\otimes\mu_\rf)
\to
L_0^2\times L_0^2
$
by
$$
F(x,y,U)
=
\begin{pmatrix}
F_1(x,y,U)
\\
F_2(x,y,U)
\end{pmatrix}
\equiv
\begin{pmatrix}
(x+\bm 1)-\softmax(U^\star(y+\bm 1))
\\
(y+\bm 1)-\softmax(-U(x+\bm 1))
\end{pmatrix}
$$
so that 
$F(x(U),y(U),U)=0$. 
Moreover, $F$ indeed takes values in $L_0^2\times L_0^2$, since
$\softmax$ maps into the density simplex $\Delta^\infty$. 

We first verify that $F$ is continuously Fr\'echet differentiable. Note that the map
$\softmax:L^\infty(\mu_\rf)\to L^\infty(\mu_\rf)$
is continuously Fr\'echet differentiable, with derivative
$
D\softmax(f)[h]
=
J_{\softmax(f)}h.
$
Furthermore, for $U\in L^\infty(\mu_\rf\otimes\mu_\rf)$ and
$v\in L^2$,
$
\|Uv\|_\infty
\le
\|U\|_\infty\|v\|_1
\le
\|U\|_\infty\|v\|_2,
$
and similarly
$
\|U^\star v\|_\infty
\le
\|U\|_\infty\|v\|_2.
$
Thus, the two bilinear maps,
$(U,v)\mapsto Uv$ and $(U,v)\mapsto U^\star v$, 
are bounded from
$L^\infty(\mu_\rf\otimes\mu_\rf)\times L^2(\mu_\rf)$ to
$L^\infty(\mu_\rf)$. Therefore, by the chain rule, $F$ is continuously
Fr\'echet differentiable.

Let us compute its derivatives. For perturbations
$(\delta x,\delta y)\in L_0^2\times L_0^2$, we have
\begin{align*}
D_{x,y}F_1(x,y,U)[\delta x,\delta y]
&=
\delta x
-
J_{\softmax(U^\star(y+\bm 1))}U^\star\delta y,
\\
D_{x,y}F_2(x,y,U)[\delta x,\delta y]
&=
\delta y
+
J_{\softmax(-U(x+\bm 1))}U\delta x.  
\end{align*}
In particular, at the solution $(x(U),y(U))$ to $F(x, y, U)=0$, 
we obtain
\begin{align}\label{eq:D_xy_F}
  D_{x,y}F(x(U),y(U),U)
=
\begin{bmatrix}
I & -J_pU^\star
\\
J_qU & I
\end{bmatrix}
\end{align}
as a bounded operator on $L_0^2\times L_0^2$, where we denote $J_p = J_{p(U)}$ and $J_q = J_{q(U)}$ for simplicity.  

Next, for a perturbation $E\in L^\infty(\mu_\rf\otimes\mu_\rf)$, the
Fr\'echet derivative of $F$ with respect to $U$ is
\begin{align*}
    D_UF_1(x,y,U)[E]
&=
-
J_{\softmax(U^\star(y+\bm 1))}E^\star(y+\bm 1),
\\
D_UF_2(x,y,U)[E]
&=
J_{\softmax(-U(x+\bm 1))}E(x+\bm 1).
\end{align*}
Therefore, at the solution,
\begin{align}\label{eq:D_U}
D_UF(x(U),y(U),U)[E]
=
\begin{pmatrix}
-J_pE^\star q(U)
\\
J_qE p(U)
\end{pmatrix}
\end{align}
where $J_p = J_{p(U)}$ and $J_q = J_{q(U)}$. 

We now show that $  D_{x,y}F(x(U),y(U),U)$ given by \eqref{eq:D_xy_F}
is a bounded isomorphism on $L_0^2\times L_0^2$. For any
$(r_1,r_2)\in L_0^2\times L_0^2$, solving
$$
\begin{bmatrix}
I & -J_pU^\star
\\
J_qU & I
\end{bmatrix}
\begin{pmatrix}
\delta x
\\
\delta y
\end{pmatrix}
=
\begin{pmatrix}
r_1
\\
r_2
\end{pmatrix},
$$
eliminating $\delta y$ first, we have 
$$
\delta y=r_2-J_qU\delta x.
$$
Substituting into the first equation yields
$$
\bigl(I+J_pU^\star J_qU\bigr)\delta x
=
r_1+J_pU^\star r_2.
$$
Since $J_q U$ and $J_p U^\star$ are both bounded linear operators on $L^2$, 
it is enough to show that
$I+J_pU^\star J_qU$
has a bounded inverse on $L_0^2$.
On $L_0^2$, write
$$
I+J_pU^\star J_qU
=
I+AB,
\qquad
A=J_p^{1/2},
\qquad
B=J_p^{1/2}U^\star J_qU.
$$
Then
$$
BA
=
J_p^{1/2}U^\star J_qU J_p^{1/2}.
$$
This operator is self-adjoint and positive semidefinite on $L^2$ 
since $J_q$ is positive semidefinite. Thus, by the Lax--Milgram theorem, $I+BA$ has a bounded inverse on $L^2$,
with
$$
\sup_{h\in L^2:\|h\|_2=1}\|(I+BA)^{-1} h\|_2 \le 1.
$$
Thus, $I+J_pU^\star J_qU = I+AB$ is also invertible on $L_0^2\subset L^2$, with inverse
$
(I+AB)^{-1}
=
I-A(I+BA)^{-1}B.
$
Hence, 
$D_{x,y}F(x(U),y(U),U)$
has a bounded inverse on $L_0^2\times L_0^2$.

Therefore, by the Implicit Function Theorem on Banach spaces, there
exists a neighborhood of $U$ in
$L^\infty(\mu_\rf\otimes\mu_\rf)$ such that $U\mapsto (x(U), y(U))$ is continuously Fr\'echet differentiable on the neighborhood. By
$x(U)=p(U)-\bm 1$ and $y(U)=q(U)-\bm 1$, the map 
$U\mapsto (p(U),q(U))$
is also continuously Fr\'echet differentiable at $U$.

It remains to compute the derivative. Differentiating
$F(x(U),y(U),U)=0$
in the direction $E$ gives
$$
D_{x,y}F(x(U),y(U),U)
\begin{pmatrix}
Dx(U)[E]
\\
Dy(U)[E]
\end{pmatrix}
+
D_UF(x(U),y(U),U)[E]
=
0.
$$
Since $p(U)=x(U)+\bm 1$ and $q(U)=y(U)+\bm 1$, we have
$Dp(U)[E]=Dx(U)[E]$ and $Dq(U)[E]=Dy(U)[E]$. Substituting \eqref{eq:D_U} and \eqref{eq:D_xy_F}, we get 
$$
\begin{bmatrix}
I & -J_pU^\star
\\
J_qU & I
\end{bmatrix}
\begin{pmatrix}
Dp(U)[E]
\\
Dq(U)[E]
\end{pmatrix}
=
\begin{pmatrix}
J_pE^\star q
\\
- J_qE p
\end{pmatrix}
$$
where $J_p = J_{p(U)}$ and $J_q = J_{q(U)}$. Eliminating $D q(U)[E]$ gives 
$$
Dp(U)[E]
=
\bigl(I+J_pU^\star J_qU\bigr)^{-1}
\bigl(
J_pE^\star q
-
J_pU^\star J_qE p
\bigr).
$$
This completes the proof.
\end{proof}

\subsection{Stability under perturbations of utility}\label{sec:perturbation}
\subsubsection{Stability of MCHF}
Let
$c(\|U\|_\oplus) = \min(1-e^{-2\|U\|_\oplus}, \|U\|_\oplus)<1$
be the contraction rate  of the map $\mu\mapsto \mu \MP$ given by Theorem \ref{theorem:contraction}. The next theorem shows how the perturbation error of $U$ propagates along the Markov chain. 
\begin{theorem}\label{theorem:stability_iteration_mc}
For any $\hat{U}, U\in L^\infty(\mu_{\rf}\otimes \mu_{\rf})$, let $\MP$ and $\hat\MP$ be the Markov kernels \eqref{eq:def_markov} constructed from $U$ and $\hat{U}$, respectively. Then,
  \begin{align*}
    \forall t\ge 1, \quad d_\tv(\mu_\rf \MP^t, \mu_\rf \hat\MP^t)
     \le
    \frac{1}{2}\|U-\hat{U}\|_\infty
    \cdot
    \frac{1-\bigl(c(\|U\|_\oplus)\wedge c(\|\hat{U}\|_\oplus)\bigr)^t}
    {1-c(\|U\|_\oplus)\wedge c(\|\hat{U}\|_\oplus)}
  \end{align*}
  where  $c(\|U\|_\oplus) = \min(1-e^{-2\|U\|_\oplus}, \|U\|_\oplus)<1$. 
\end{theorem}
\begin{proof}
  Let $c$ and $\hat{c}$ be the contraction coefficient given by Theorem \ref{theorem:contraction} for $\mu\mapsto \MP\mu$ and $\mu\mapsto \hat{\MP}\mu$, respectively. 
Consider the two iterations, 
$\mu_t = \mu_{t-1}\MP$ and 
$\hat{\mu}_t = \hat{\mu}_{t-1}\hat{\MP}$, 
initialized with $\mu_0 = \hat{\mu}_0 \in \mathscr{P}_{\mu_\rf}$. 
Using the triangle inequality, we have 
\begin{align*}
    d_{\tv}(\mu_t, \hat{\mu}_t) &\le d_{\tv}(\mu_{t-1}\MP, \hat{\mu}_{t-1}\MP)
  + d_{\tv}(\hat{\mu}_{t-1}\MP, \hat{\mu}_{t-1}\hat{\MP})\le c \cdot d_{\tv}(\mu_{t-1}, \hat{\mu}_{t-1})
  + d_{\tv}(\hat{\mu}_{t-1}\MP, \hat{\mu}_{t-1}\hat{\MP}).
\end{align*}
Let us bound the second term. 
By the definition of TV distance, using Fubini's lemma,
\begin{align*}
    d_{\tv}(\hat{\mu}_{t-1}\MP, \hat{\mu}_{t-1}\hat{\MP}) &= \frac12 \int \mu_\rf(dy)\Bigl|
         \int \hat{\mu}_{t-1}(dx)
    \bigl(P(x,y)-\hat{P}(x,y)\bigr)
    \Bigr|\\
    &\le \frac12 \int \mu_\rf(dy)\hat{\mu}_{t-1}(dx) |P(x, y)-\hat{P}(x, y)|\\
    &= \frac12 \int \hat{\mu}_{t-1}(dx) \|P(x, \cdot)-\hat{P}(x, \cdot)\|_1  && \text{Fubini's lemma} \\
  &\le
  \frac12
  \int \hat{\mu}_{t-1}(dx)
  \bigl\|U(x,\cdot)-\hat{U}(x,\cdot)\bigr\|_\infty
  && \text{by Lemma \ref{lemma:lipschitz_softmax}} \\
  &\le
  \frac12 \|U-\hat{U}\|_\infty
\end{align*}
where we used Lemma \ref{lemma:lipschitz_softmax} with $\mu=\mu_\rf, g(\cdot)=U(x,\cdot)$ and $\tilde{g}(\cdot)=\hat{U}(x, \cdot)$. The last inequality follows from the fact that $\hat\mu_{t-1}$ is absolutely continuous with respect to $\mu_\rf$. 

Thus, we obtain 
$$
d_{\tv}(\mu_t, \hat{\mu}_t)
  \le
  c \cdot d_{\tv}(\mu_{t-1}, \hat{\mu}_{t-1})
  + \frac12 \|U-\hat{U}\|_\infty.
$$
Iterating this inequality and using the initial condition $\mu_0=\hat{\mu}_0$, we get
$d_{\tv}(\mu_t, \hat{\mu}_t)
  \le
  \frac12 \|U-\hat{U}\|_\infty
  \sum_{s=0}^{t-1} c^s.
$
By symmetry, the same upper bound also holds with $c$ replaced by $\hat{c}$. Thus, we complete the proof.
\end{proof}

Now, letting $t\to\infty$ in Theorem \ref{theorem:stability_iteration_mc}, noting that $\mu_\rf \MP^t\to \mu_\mc$ as $t\to+\infty$ in TV distance (see Corollary \ref{corollary:stationary_dist}), we obtain the following corollary.
\begin{corollary}[Stability of MCHF equilibrium]\label{corollary:stability_mc}
  For any $U, \hat{U} \in L^\infty(\mu_{\rf}\otimes \mu_{\rf})$, letting $\mu_\mc(U)$ and $\mu_\mc(\hat{U})$ denote the stationary distributions of $\MP$ and $\hat\MP$, respectively,
  $$
  d_\tv(\mu_\mc(U), \mu_\mc(\hat{U}))
  \le
  \frac{1}{2}
  \cdot
  \frac{\|U-\hat{U}\|_\infty}
  {1-c(\|U\|_\oplus)\wedge c(\|\hat{U}\|_\oplus)}
  $$
    where  $c(\|U\|_\oplus) = \min(1-e^{-2\|U\|_\oplus}, \|U\|_\oplus)<1$.
\end{corollary}
Since $c(\|U\|_\oplus)$ is always strictly less than $1$, we conclude that the map $U \mapsto \mu_\mc(U)$ is locally Lipschitz over $L^\infty(\mu_\rf\otimes \mu_\rf)$, and the local Lipschitz constant is controlled by the seminorm $\|U\|_\oplus$. 

\subsubsection{Stability of NLHF}
Next, we claim the stability of the NLHF solution $\mu_\nl$. 
\begin{theorem}[Stability of NLHF equilibrium]\label{theorem:stability_nl}
There exists an absolute constant $C$ such that, for any
$U,\hat U\in L^\infty(\mu_{\rf}\otimes \mu_{\rf})$,
  $$
  d_{\tv}\bigl(\mu_{\nl}(U), \mu_{\nl}(\hat{U})\bigr)
  \le
  C
  \bigl(1+\|U\|_\oplus^6+\|\hat{U}\|_\oplus^6\bigr)
  \|U-\hat{U}\|_\infty.
  $$
\end{theorem}
\begin{proof}
The proof of this theorem requires a different technique from that of Theorem \ref{theorem:asymptotics_iterations} since the fixed point equation of the NLHF is not always contractive under the TV distance. 

Fix $U, \hat{U}\in L^\infty(\mu_\rf\otimes \mu_\rf)$. 
For each $t\in [0,1]$, let $U_t$ be the linear interpolation
$$
U_t = t\hat{U} + (1-t) U = U + t E \quad \text{where} \quad E = \hat{U}-U
$$
and 
let $(\mu_{t}, \nu_t)$ be the solution to the NLHF based on the pairwise utility $U_t$, and let $p_t= \frac{d\mu_{t}}{d\mu_{\rf} }$ and $q_t=\frac{d\nu_t}{d\mu_\rf}$ be the density for each $t\in [0,1]$, so that our goal is to bound 
$d_{\tv}(\mu_1, \mu_0) =  \frac12 \|p_1-p_0\|_1$. 

By Theorem \ref{theorem:frechet_derivative_nl}, $t\mapsto p_t$ is $C^1$ from $\R$ to $L^2$ in the sense of
$$\lim_{\epsilon\to 0} \|\epsilon^{-1}(p_{t+\epsilon}-p_t) - \dot{p}_t\|_2 = 0$$
where the derivative $\dot{p}_t$ is given by
$$
\dot{p}_t =  (I + J_{p_t} (U_t)^\star J_{q_t} U_t)^{-1} (J_{p_t} E^\star q_t - J_{p_t} U_t^\star J_{q_t} E p_t).
$$
Here, since $\|\cdot\|_1\le\|\cdot\|_2$, $p_t$ is in particular $C^1$ from $\R$ to $L^1$. Thus, applying the fundamental theorem of calculus on the Banach space $L^1$ to $d_{\tv}(\mu_1, \mu_0) = \frac12 \|p_1-p_0\|_1$, using Fubini's lemma, 
\begin{align*}
    d_{\tv}(\mu_1, \mu_0)
  = \frac12  \int \mu_{\rf}(dy)\Bigl|
  \int_0^1 \dot{p}_t(y) dt
  \Bigr|\le \frac12  \int \mu_{\rf}(dy)
  \int_0^1 |\dot{p}_t(y)| dt = \frac12 \int_0^1 dt \|\dot{p}_t\|_1.
\end{align*}
Next, we claim that there exists an absolute constant $C$ such that 
\begin{align}\label{eq:derivative_bound}
  \forall t\in [0,1], \quad 
\|\dot{p}_t\|_1 \le C \cdot (1 + \|U_t\|_\oplus^6) \cdot  \|E\|_{\infty}. 
\end{align}
If this claim holds, by the triangle inequality for the seminorm $\|U_t\|_\oplus = \|t\hat{U} + (1-t)U\|_\oplus$, one can show $\int_0^1 dt \|U_t\|_\oplus^6 \le C' (\|U\|_\oplus^6 + \|\hat{U}\|_\oplus^6)$ for an absolute constant $C'$. Consequently, we obtain
$$
d_{\tv}(p_{\nl}(U), p_{\nl}(\hat{U})) \le \frac12 \int_0^1 C (1+\|U_t\|_\oplus^6) \|E\|_\infty \le \frac12 C C' \cdot  \bigl(1 + {\|\hat{U}\|_\oplus^6 + \|U\|_\oplus^6}\bigr)  \cdot {\|E\|_{\infty}}, 
$$
and  the proof of Theorem \ref{theorem:stability_nl} is complete. Thus, the rest of the goal is to show \eqref{eq:derivative_bound}.
Below, we prove it using the following two lemmas. 
\begin{lemma}\label{lemma:LJ_p_infty_1_norm}
  For any $p\in \Delta^\infty$ and any $f\in L^\infty$, 
  $\|J_p f\|_1 \le \|f\|_\infty$.   
\end{lemma}
\begin{proof}
Letting $\mu(dz)=p(z)\mu_{\rf}(dz)$ be the tilted measure induced by density $p\in \Delta^\infty$, for any $f\in L^\infty$, we have $      \|J_{p} f\|_{1} 
= \E_{\mu} [|f-\E_{\mu}[f]|]$ where $\E_\mu [h] = \int \mu(dz) h(z)$. Using 
$$
\E_\mu[|f-\E_\mu[f]|]\le \sqrt{\operatorname{Var}_\mu(f)}\le \|f\|_{L^2(\mu)}\le \|f\|_{L^\infty(\mu)}
$$ and $\|f\|_{L^\infty(\mu)}\le \|f\|_{L^\infty(\mu_\rf)}$ since $\mu$ is absolutely continous with respect to $\mu_\rf$, we complete the proof. 
\end{proof}

\begin{lemma}\label{lemma:Schur_complement_l1_norm}
There exists a universal constant $C$ such that for any $U\in L^\infty(\mu_\rf\otimes \mu_\rf)$, $g\in L^2$, and $p, q\in \Delta^\infty$, 
  $$
  \|(I + J_p U^\star J_q U)^{-1} g\|_1 \le C(1+\|U\|_\infty^5) \|g\|_1. 
  $$
\end{lemma}
\begin{proof}
 The proof requires a careful analysis of the operator norm of $(I + J_p U^\star J_q U)^{-1}$ measured by $L^1$ norm (rather than $L^2$), which is of independent interest. 
  See \Cref{proof:lemma:Schur_complement_l1_norm}. 
\end{proof}

Let us finish the proof of \eqref{eq:derivative_bound}. 
Fix $f,g \in L^\infty$ and let $U_t^{g\oplus f}=U_t - g\oplus f$.  For any $h\in L_0^2$, using $J_{q_t} \bm{1}=0$ and $\langle \bm{1}, h\rangle=0$, we have 
  $$
  J_{q_t} U_t h = J_{q_t} (U_t-g\oplus f) h + J_{q_t} \bm{1} \langle f, h\rangle +  J_{q_t} g \langle \bm{1}, h\rangle = J_{q_t} U_t^{g\oplus f} h
  $$
  and hence 
  $J_{q_t} U_t$ operates on $L_0^2$ as 
  $$
  J_{q_t} U_t|_{L_0^2} = J_{q_t} U_t^{g\oplus f}.
  $$
  By the same argument, we have 
  $$
    J_{p_t} U_t^\star |_{L_0^2} = J_{p_t} (U_t^{g\oplus f})^\star.
  $$
Since the image of $J_{p_t}$ and $J_{q_t}$ are both $L_0^2$, we can rewrite $\dot p_t$ as 
\begin{align*}
  \dot p_t
&=\bigl(I+J_{p_t}(U_t^{g\oplus f})^\star J_{q_t}U_t^{g\oplus f}\bigr)^{-1}
\bigl(
J_{p_t}E^\star q_t
-
J_{p_t}(U_t^{g\oplus f})^\star J_{q_t}E p_t
\bigr). 
\end{align*}
Applying Lemma \ref{lemma:LJ_p_infty_1_norm} with $\|p_t\|_1=\|q_t\|_1=1$ and H\"{o}lder's inequality $\|K h\|_\infty \le \|K\|_\infty \|h\|_1$, 
 we get 
\begin{align*}
\bigl\|
J_{p_t}E^\star q_t
-
J_{p_t}(U_t^{g\oplus f})^\star J_{q_t}E p_t
\bigr\|_1 &\le \|E\|_\infty + \|U_t^{g\oplus f}\|_\infty \|E\|_\infty. 
\end{align*}
Applying Lemma \ref{lemma:Schur_complement_l1_norm} with $U_t=U_t^{g\oplus f}$ and $(p,q)=(p_t, q_t)$, combined with the above $L^1$-bound, 
we have 
\begin{align*}
  \|\dot p_t\|_1 &\le C (1+\|U_t^{g\oplus f}\|_\infty^5) (1 + \|U_t^{g\oplus f}\|_\infty) \|E\|_\infty \le C'(1+\|U_t^{g\oplus f}\|_\infty^6) \|E\|_\infty
\end{align*}
where $C$ and $C'$ are absolute constants. Recalling $U_t^{g\oplus f}=U_t-g\oplus f$ and $g, f\in L^\infty$ are arbitrary, taking $\inf_{g,f}$ on the RHS, with $\|U\|_\oplus = \inf_{g,f}\|U-g\oplus f\|_\infty$, we complete the proof of \eqref{eq:derivative_bound}. 
\end{proof}

\begin{remark}[Other stability bounds for NLHF]
Using the formulation of NLHF as the solution to a strongly convex optimization as in \Cref{proof:theorem:convergence_nl_kl}, 
one can easily show the bound under KL:
$$
\kl\bigl(\mu_{\nl}(U)\|\mu_{\nl}(\hat U)\bigr)
\le
2\|U-\hat U\|_\infty .
$$
Combined with Pinsker's inequality, $d_{\tv}\le \sqrt{\kl/2}$, this yields global H\"older-$1/2$ continuity in TV.

On the other hand, 
using the fact that the Lipschitz constant of the map $p \mapsto \argmin_{q\in \Delta^\infty} \langle q, Up \rangle + w(q)$ is $\|U\|_\oplus$-Lipschitz under the $L^1(\mu_\rf)$ metric (see Lemma \ref{lemma:ell_p_smoothness}-(1)), an argument similar to the proof of Theorem \ref{theorem:stability_iteration_mc} yields
$$
d_{\tv}\bigl(\mu_{\nl}(U),\mu_{\nl}(\hat U)\bigr)
\le
\frac12
\cdot
\frac{\|U-\hat U\|_\infty}
{1-(\|U\|_\oplus\wedge \|\hat U\|_\oplus)}
$$
whenever $\|U\|_\oplus\wedge \|\hat U\|_\oplus<1$. 

However, neither of these bounds is sufficient for the proof of uniform Fr\'echet differentiability (Theorem \ref{theorem:uniform_frechet_differentiability}) for NLHF, which requires Lipschitz continuity on every subset bounded in the seminorm $\|\cdot\|_\oplus$.
\end{remark}

\subsubsection{Proof of Lemma \ref{lemma:Schur_complement_l1_norm}}\label{proof:lemma:Schur_complement_l1_norm}
  Let $f=(I + J_p U^\star J_q U)^{-1} g$. We aim to bound $\|f\|_1$ by $\|g\|_1$. 
If we define $h=J_{q} U f$, the triple $(f, h, g)$ satisfies
  \begin{align*}
    f + J_{p} U^\star h = g, \qquad
    -J_{q} U f + h = 0. 
  \end{align*}
  Now, for any $p \in \Delta^\infty$, we rewrite the operator $J_p$ as $  J_{p} = D_{\sqrt{p}} \cdot D_{\sqrt{p}}  - p p^\star$
  where $D_{\sqrt{\nu}}$ and $\nu\nu^\star$ are bounded linear operators defined as 
  $
  D_{\sqrt{p}} f(z) = \sqrt{p(z)} f(z)$ and $p p^\star f(z) = p (z) \cdot \langle p, f \rangle. 
  $
  Substituting this into the previous system, letting $u_f=D_{\sqrt{q}} U f$, $u_h = D_{\sqrt{p}} U^\star h$, we are left with
  \begin{align}\label{eq:system_f_g_h}
    \begin{split}
          f +  D_{\sqrt{p}} u_h - p \cdot c_h = g, \quad 
    -D_{\sqrt{q}} u_f + q \cdot c_f + h = 0
    \end{split}
  \end{align}
  where we defined the scalars $c_h, c_f$ as 
  \begin{align*}
  c_h &= \langle p, U^\star h\rangle = \langle \sqrt{p}, D_{\sqrt{p}} U^\star h \rangle = \langle \sqrt{p}, u_h\rangle, \\
  c_f &= \langle q, U f\rangle = \langle \sqrt{q}, D_{\sqrt{q}}Uf\rangle  = \langle \sqrt{q}, u_f\rangle  
  \end{align*}
  Below, we repeatedly use 
  \begin{align*}
    \forall p\in \Delta^\infty, \quad 
    \|D_{\sqrt{p}} f\|_2 \le \|f\|_\infty, \quad \|D_{\sqrt{p}} f\|_1 \le \|f\|_2
  \end{align*}
  which can be shown by using H\"{o}lder's inequality and $p\in \Delta^\infty$, i.e., $p()\ge 0$ and $\int \mu_\rf(dx) p(x)=1$. 

  Here, from the first equation of \eqref{eq:system_f_g_h}, taking $L^1$ norm and using triangle inequality, we see 
\begin{align}\label{eq:upper_bound_f_by_uh_ch}
  \|f\|_1 &\le \|D_{\sqrt{p}} u_h\|_1 + c_h \|p\|_1 + \|g\|_1 \le \|u_h\|_2 + |c_h| + \|g\|_1
\end{align}
Thus, it is sufficient to bound the $\|u_h\|_2$ and $|c_h|=|\langle\sqrt{p}, u_h\rangle|$.

  Now, multiplying the first equation in \eqref{eq:system_f_g_h} by $D_{\sqrt{q}} U$ and the second equation by $D_{\sqrt{p}} U^\star$, letting 
  $A=D_{\sqrt{q}} U D_{\sqrt{p}}$,  
  noting $D_{\sqrt{p}} \sqrt{p}=p$ and $D_{\sqrt{q}} \sqrt{q}=q$, 
  we have 
  \begin{align}\label{eq:system_uf_uh}
          u_f +  A u_h - c_h  A \sqrt{p} = D_{\sqrt{q}}U g, \qquad
    - A^\star u_f + u_h + c_f   A^\star \sqrt{q} = 0.
  \end{align}
Now we claim $\|A\|_{2\to 2}\le \|U\|_\infty$. Indeed, 
$$
\forall f\in L^2, \quad 
\|Af\|_{2} = \|D_{\sqrt{q}} U D_{\sqrt{p}}f\|_{2}\le \|U D_{\sqrt{p}} f\|_{\infty} \le \|U\|_\infty \cdot \|D_{\sqrt{p}} f\|_1 \le \|U\|_\infty \|f\|_2
$$
where we used $\|D_{\sqrt{p}} f\|_2 \le \|f\|_\infty$, $\|Uf\|_\infty \le \|U\|_\infty \|f\|_1$, and 
$\|D_{\sqrt{p}} f\|_1 \le \|f\|_2$.

Below, we solve the system \eqref{eq:system_uf_uh} for $(u_f, u_h)$ and obtain explicit upper bounds of $(c_f, c_h)$. Substituting the second equation $u_h =  A^\star u_f - c_f   A^\star \sqrt{q}$ in \eqref{eq:system_uf_uh} to the first equation, rearranging it, we get
  \begin{align*}
   &(I+ A  A^\star) u_f = c_f  AA^\star \sqrt{q} + c_h  A\sqrt{p} + D_{\sqrt{q}} Ug. 
  \end{align*}
  Now we claim that $(I+ A  A^\star)$ has a bounded inverse on $L^2$.  Indeed,
  $(I+ AA^\star)$ is coercive since $ A  A^\star$ is a self adjoint positive semidefinite linear operator on $L^2$. Moreover, the operator norm is bounded as 
  $$
  \| A   A^\star \|_{2\to 2} \le \| A\|_{2\to 2}^2 \le \|U\|_\infty^2. 
  $$
  Thus, by the Lax-Milgram theorem, $(I+ A  A^\star)$ has a bounded inverse, and letting 
  $
  T = (I+ A   A^\star)^{-1}, 
  $ 
  the spectrum of $T$ is controlled as 
  \begin{align}\label{eq:spectrum_T}
   \forall f\in L^2, \quad (1 + \|U\|_{\infty}^2)^{-1} \|f\|_2^2 \le \langle f, Tf\rangle \le \|f\|_2^2.
  \end{align}
  Then, $u_f$ can be solved as 
  \begin{align*}
    u_f = T\Bigl(c_f  AA^\star \sqrt{q} + c_h  A\sqrt{p} + D_{\sqrt{q}} Ug\Bigr) = c_f (I-T) \sqrt{q} + c_h T A \sqrt{p} + T D_{\sqrt{q}} Ug 
  \end{align*}
  where we used $TAA^\star = 1-T$, which follow from the definition $T=(I+ AA^\star)^{-1}$. 
  
  By the same argument, $(I+ A^\star A)$ has a bounded inverse on $L^2$, and letting 
  $
   S=(I+ A^\star A)^{-1},
  $ 
  the spectrum of $S$ is controlled as 
  \begin{align}\label{eq:spectrum_S}
    \forall f\in L^2, \quad (1 + \|U\|_{\infty}^2)^{-1} \|f\|_2^2 \le \langle f,  Sf\rangle \le \|f\|_2^2
  \end{align}
  and $u_h$ can be solved explicitly as 
  $$
    u_h = c_h (I- S) \sqrt{p} -c_f \cdot  S A^\star \sqrt{q} +  S A^\star D_{\sqrt{q}}Ug
  $$
Putting all together, we get 
\begin{align}\label{eq:system_uf_uh_S_T}
  \begin{split}
      u_f &= c_f (I-T) \sqrt{q} + c_h T A \sqrt{p} + T D_{\sqrt{q}} Ug \\
   u_h &= c_h (I- S) \sqrt{p} -c_f  S A^\star \sqrt{q} +  S A^\star D_{\sqrt{q}}Ug.
  \end{split}
\end{align}
Taking the inner products $\langle \sqrt{q}, \cdot \rangle$ and  $\langle \sqrt{p}, \cdot \rangle$ on the first equation and the second equation respectively, recalling $c_f = \langle \sqrt{q}, u_f\rangle$
and $c_h = \langle \sqrt{p}, u_h\rangle$, with $\langle \sqrt{p}, \sqrt{p}\rangle= \langle \sqrt{q}, \sqrt{q}\rangle = 1$ since $p, q\in \Delta^\infty$, we are left with 
\begin{align*}
  c_f &= c_f (1 - \langle \sqrt{q}, T\sqrt{q}\rangle) + c_h \langle\sqrt{q}, T A \sqrt{p} \rangle + \langle \sqrt{q},  T D_{\sqrt{q}} Ug  \rangle\\
  c_h &= c_h (1-\langle \sqrt{p},   S \sqrt{p}\rangle) -c_f \langle \sqrt{p},   S A^\star \sqrt{q} \rangle + \langle \sqrt{p},  S A^\star D_{\sqrt{q}}Ug\rangle 
\end{align*}
Noting that $c_f$ and $c_h$ are cancelled out, using the identity $T A=  A S$ from the definition of $ S$ and $T$, we are left with the linear system of $(c_f, c_h)$:
\begin{align*}
 \begin{bmatrix}
  \alpha & -\gamma\\
  \gamma & \beta 
 \end{bmatrix}\begin{bmatrix}
  c_f\\
  c_h
 \end{bmatrix} = \begin{bmatrix}
  \langle \sqrt{q},  T D_{\sqrt{q}} Ug  \rangle\\
  \langle \sqrt{p},  S A^\star D_{\sqrt{q}}Ug\rangle 
 \end{bmatrix} \quad \text{where} \quad \begin{cases}
   \alpha &=  \langle \sqrt{q}, T\sqrt{q}\rangle\\
   \beta  &= \langle \sqrt{p},   S \sqrt{p}\rangle\\
   \gamma &= \langle\sqrt{q}, T A \sqrt{p} \rangle
 \end{cases}
\end{align*}
Since $\|\sqrt{p}\|_2 = \|\sqrt{q}\|_2=1$, using the established bound of the spectrum of $ S$ and $T$ \eqref{eq:spectrum_T}-\eqref{eq:spectrum_S}, we know 
\begin{align}\label{eq:alphabetagamma_estimate}
(1+\|U\|_{\infty}^2)^{-1} \le \alpha \le 1, \quad (1+\|U\|_{\infty}^2)^{-1}\le \beta \le 1, \quad |\gamma|\le 2^{-1}
\end{align}
Here $|\gamma|\le 2^{-1}$ follows from the fact that $\|T A\|_{2\to 2} =\sqrt{\|(T A) (T A)^\star\|_{2\to 2}}$ where $(T A) (T A)^\star =  AA^\star (I+ A  A^\star)^{-2}$, whose spectrum is given by $\sigma(1+\sigma)^{-2}$ where $\sigma$ is the spectrum of $AA^\star$. Since $\sup_{x\ge 0} x/(1+x)^2 = 1/4$, taking square root, we get $\|T A\|_{2\to 2}\le 1/2$. By the same argument, $\| S A^\star\|_{2\to 2}\le 1/2$ holds. 

By the lower bound of $\alpha$ and $\beta$ in \eqref{eq:alphabetagamma_estimate}, the determinant of the matrix is strictly positive:
\begin{align*}
  \text{det} \begin{bmatrix}
  \alpha & -\gamma\\
  \gamma & \beta 
 \end{bmatrix} = \alpha\beta+\gamma^2 \ge \frac{1}{(1+\|U\|_\infty^2)^2}. 
\end{align*}
Therefore, this matrix is invertible, and we obtain 
 \begin{align*}
\begin{bmatrix}
  c_f\\
  c_h
 \end{bmatrix} =   \frac{1}{\alpha\beta + \gamma^2} \begin{bmatrix}
 \beta  & \gamma\\
  -\gamma & \alpha
 \end{bmatrix} \begin{bmatrix}
  \langle \sqrt{q},  T D_{\sqrt{q}} Ug  \rangle\\
  \langle \sqrt{p},  S A^\star D_{\sqrt{q}}Ug\rangle
 \end{bmatrix}
 \end{align*}
Taking the absolute value and using the estimate of $|\alpha|, |\beta|, |\gamma|$ from \eqref{eq:alphabetagamma_estimate}, noting that $\|T\|_{2\to 2}\le 1$ and $\|SA^\star\|_{2\to 2}\le 1/2$ give 
$$
|\langle \sqrt{q},  T D_{\sqrt{q}} Ug  \rangle|\le\|D_{\sqrt{q}} Ug\|_2, \qquad |\langle \sqrt{p},  S A^\star D_{\sqrt{q}}Ug\rangle|\le 2^{-1}\|D_{\sqrt{q}}Ug\|_2,
$$ 
where 
$\|D_{\sqrt{q}} Ug\|_2 \le \|Ug\|_\infty \le \|U\|_\infty \|g\|_1$, we obtain 
 the following inequality for each coordinate:
\begin{align}\label{eq:upper_bound_cf_ch}
  \begin{bmatrix}
  |c_f|\\
  |c_h|
 \end{bmatrix} \le  (1+\|U\|_\infty^2)^2 \begin{bmatrix}
 1  & 2^{-1}\\
2^{-1} & 1
 \end{bmatrix} \begin{bmatrix}
 1\\
 2^{-1}
 \end{bmatrix} \cdot \|U\|_\infty \|g\|_1 = (1+\|U\|_\infty^2)^2 \|U\|_\infty \|g\|_1 \begin{bmatrix}
  5/4\\
  1
 \end{bmatrix}
\end{align}

Let us finish the proof. By the second equation of \eqref{eq:system_uf_uh_S_T}, 
\begin{align*}
   u_h &= c_h (I- S) \sqrt{p} -c_f  S A^\star \sqrt{q} +  S A^\star D_{\sqrt{q}}Ug. 
\end{align*}
Taking $\|\cdot\|_2$, using $\|I- S\|_{2\to 2}\le 1$ from \eqref{eq:spectrum_S}, $\| S A^\star\|_{2\to 2}\le 1/2$ and $  \|D_{\sqrt{q}} Ug\|_2\le \|U\|_\infty \|g\|_1$, we obtain
\begin{align*}
  \|u_h\|_2\le |c_h| + 2^{-1} |c_f| + 2^{-1} \|U\|_\infty \|g\|_1.
\end{align*}
Combined with the upper bound $  \|f\|_1 \le \|u_h\|_2 + |c_h| + \|g\|_1$ from \eqref{eq:upper_bound_f_by_uh_ch}, 
$$
\|f\|_1\le |c_h| + 2^{-1}|c_f| + 2^{-1} \|U\|_\infty \|g\|_1 + |c_h| + \|g\|_1 = [2^{-1}, 2]\begin{bmatrix}
  |c_f|\\
  |c_h|
\end{bmatrix} + (2^{-1}\|U\|_\infty  + 1)\|g\|_1
$$
Finally, substituting the upper bounds of $[|c_f|, |c_h|]$ in \eqref{eq:upper_bound_cf_ch}, with $[2^{-1}, 2] \cdot [5/4, 1]^\top = 21/8$, we obtain
$$
\|f\|_1 \le \Bigl(\frac{21}{8}(1+\|U\|_\infty^2)^2 \|U\|_\infty + 1 + \frac12 \|U\|_\infty\Bigr) \|g\|_1.
$$
Expanding the RHS, one can find an absolute constant $C$ such that $\|f\|_1 \le C (1+\|U\|_\infty^5)\|g\|_1$. Recalling $f=(I + J_p U^\star J_q U)^{-1} g$, the proof is complete. 


\subsection{Uniform Fr\'echet differentiability}\label{app:uniform_frechet_differentiability}
We have seen in \Cref{app:frechet_differentiability} that MCHF and NLHF are both Fr\'echet differentiable; there exists derivatives $D p_*(U)$ for $*\in \{\mc, \nl\}$ such that
\begin{align*}
  \forall U\in L^\infty(\mu_\rf\otimes\mu_\rf), \quad 
 \lim_{\|E\|_\infty\to 0}\frac{\|p_*(U+E)-p_*(U)-Dp_*(U)[E]\|_1}{\|E\|_\infty} = 0.
\end{align*}
Now we aim to make this convergence result uniform for all $U$ such that $\|U\|_\oplus\le R$ for a fixed constant $R$, i.e., we aim to show the above convergence with $\sup_{\|U\|_\oplus\le R}$ inside $\lim_{\|E\|_\infty\to 0}$. 

\begin{lemma}\label{lemma:J_p_lip}
$\|J_p-J_q\|_{\infty\to 1}= \sup_{\|f\|_\infty \le 1} \|(J_p-J_q) f\|_1 
\le 3\|p-q\|_1$ for any $p, q\in \Delta^\infty$
\end{lemma}
\begin{proof}
For any $f\in L^\infty$, writing $J_p f= p\odot f - p \cdot \langle p, f\rangle$, 
\begin{align*}
(J_p-J_q)f
&=
(p-q)\odot f
-
(p-q)\cdot \langle p, f\rangle - 
q \cdot \langle p-q, f\rangle
\end{align*}
Using H\"{o}lder's inequality with $\|p\|_1=\|q\|_1 =1$, one can show that the $L^1$ norm of each term on RHS is bounded by $\|p-q\|_1\|f\|_\infty$. Thus, by the triangle inequality, complete the proof. 
\end{proof}

\subsubsection{MCHF}

\begin{lemma}\label{lemma:derivative_lipschitz_continuous_mc}
  The Fr\'echet derivative $D p_\mc(U)$ is locally Lipschitz in the sense that for all $U, V\in L^\infty$, 
  $$
    \sup_{\|E\|_\infty \le 1}\|D p_\mc(U)[E] - D p_\mc(V)[E]\|_1 \le 5 \cdot \exp(2\|U\|_\oplus + 2\|V\|_\oplus)\cdot \|U-V\|_\infty
  $$
\end{lemma}
\begin{proof}
Let $A_U$ and $h_U$ as 
$$
A_U = I-P_U^\star, \quad h_{U, E}(y) = \int \mu_\rf(dx)
p_\mc(U)(x) [J_{P_U(x, \cdot)} E(x, \cdot)](y)
$$
so that $Dp_\mc(U)[E]
= \bigl(A_U\mid_{L_0^1}\bigr)^{-1} h_{U, E}$. 
We use the decomposition
\begin{align*}
Dp(U)[E]-Dp(V)[E]
=
\bigl((A_U|_{L_0^1})^{-1}-(A_V|_{L_0^1})^{-1}\bigr)h_{V,E}
+
(A_U|_{L_0^1})^{-1}(h_{U,E}-h_{V,E})
\end{align*}
First, we note 
\begin{align*}
(A_U|_{L_0^1})^{-1}-(A_V|_{L_0^1})^{-1}
=
(A_U|_{L_0^1})^{-1}(A_V-A_U)(A_V|_{L_0^1})^{-1}.
\end{align*}
Recall $\|(A_U\mid_{L_0^1})^{-1}\|_{L_0^1\to L_0^1}\le e^{2\|U\|_\oplus}$ from \eqref{eq:I_P_inverse_operator_norm}, and $\|A_U-A_V\|_{1\to 1}
=\|P_V^\star-P_U^\star\|_{1\to 1}$. Here, for all $h\in L^1$, 
\begin{align*}
  \|(P_V^\star -  P_U^\star) h\|_1  &= \|\int \mu_\rf(dx) h(x)(P_V(x, \cdot)-P_U(x, \cdot))\|_1\le \|h\|_1 \esssup_x \|P_V(x, \cdot)-P_U(x, \cdot)\|_1.
\end{align*}
Then, using Lemma \ref{lemma:lipschitz_softmax}, we have 
$$
\esssup_x \|P_V(x, \cdot)-P_U(x, \cdot)\|_1 \le \esssup_x \|V(x, \cdot)-U(x, \cdot)\|_\infty = \|U-V\|_\infty
$$
and hence $\|A_U-A_V\|_{1\to 1}
=
\|P_V^\star-P_U^\star\|_{1\to 1} 
\le
\|U-V\|_\infty$. 
Therefore, 
we obtain
\begin{align*}
\Bigl\|
\bigl((A_U|_{L_0^1})^{-1}-(A_V|_{L_0^1})^{-1}\bigr)\mid_{L_0^1}
\Bigr\|_{1\to 1}
\le
\exp(2\|U\|_\oplus)
\exp(2\|V\|_\oplus)
\|U-V\|_\infty
\end{align*}

Next, we bound $h_{V,E}$. Using $\|J_p\|_{\infty\to 1}\le 1$ from Lemma \ref{lemma:LJ_p_infty_1_norm}, whenever $\|E\|_\infty \le 1$, 
\begin{align*}
\|h_{V,E}\|_1
\le
\int \mu_\rf(dx)p_\mc(V)(x)
\|J_{P_V(x,\cdot)}E(x,\cdot)\|_1 \le \esssup_x \|J_{P_V(x,\cdot)}E(x,\cdot)\|_1
\le 1.
\end{align*}
Hence, 
\begin{align*}
\bigl\|
\bigl((A_U|_{L_0^1})^{-1}-(A_V|_{L_0^1})^{-1}\bigr)h_{V,E}
\bigr\|_1 \le
\exp(2\|U\|_\oplus+2\|V\|_\oplus)
\|U-V\|_\infty .
\end{align*}

It remains to bound $h_{U,E}-h_{V,E}$. We write
\begin{align*}
h_{U,E}-h_{V,E}
&=
\int \mu_\rf(dx)
\{p_\mc(U)(x)-p_\mc(V)(x)\}
J_{P_U(x,\cdot)}E(x,\cdot) \\
&+
\int \mu_\rf(dx)
p_\mc(V)(x)
\{J_{P_U(x,\cdot)}-J_{P_V(x,\cdot)}\}E(x,\cdot).
\end{align*}
For the first term, using again $\esssup_x \|J_{P_U(x,\cdot)}E(x,\cdot)\|_1\le \|E\|_\infty\le 1$, 
we get
\begin{align*}
\Bigl\|
\int \mu_\rf(dx)
\{p_\mc(U)(x)-p_\mc(V)(x)\}
J_{P_U(x,\cdot)}E(x,\cdot)
\Bigr\|_1
\le
\|p_\mc(U)-p_\mc(V)\|_1.
\end{align*}
By Corollary \ref{corollary:stability_mc},
\begin{align*}
\|p_\mc(U)-p_\mc(V)\|_1
\le
\exp\{2(\|U\|_\oplus\wedge\|V\|_\oplus)\}
\|U-V\|_\infty.
\end{align*}
Thus, the first term is bounded by $\exp\{2(\|U\|_\oplus\wedge\|V\|_\oplus)\}
\|U-V\|_\infty$.
For the second term, using 
$
\|J_p-J_q\|_{\infty\to 1}
\le
3\|p-q\|_1$ in Lemma \ref{lemma:J_p_lip}, 
\begin{align*}
\Bigl\|
\int \mu_\rf(dx)
p_\mc(V)(x)
\{J_{P_U(x,\cdot)}-J_{P_V(x,\cdot)}\}E(x,\cdot)
\Bigr\|_1 
&\le 3 \esssup_x \|P_U(x,\cdot)-P_V(x,\cdot)\|_1
\end{align*}
Using again $\esssup_x \|P_U(x,\cdot)-P_V(x,\cdot)\|_1\le \|U-V\|_\infty$, the second term is bounded by $3\|U-V\|_\infty$. 
Combining the two estimates gives
\begin{align*}
\|h_{U,E}-h_{V,E}\|_1
\le
\exp\{2(\|U\|_\oplus\wedge\|V\|_\oplus)\}
\|U-V\|_\infty
+
3\|U-V\|_\infty. 
\end{align*}
Absorbing constants, we may write
\begin{align*}
\|h_{U,E}-h_{V,E}\|_1
\le
4\exp\{2(\|U\|_\oplus\wedge\|V\|_\oplus)\}
\|U-V\|_\infty.
\end{align*}
Therefore,
\begin{align*}
\|(A_U|_{L_0^1})^{-1}(h_{U,E}-h_{V,E})\|_1
&\le
\exp(2\|U\|_\oplus)
\|h_{U,E}-h_{V,E}\|_1 \le
4\exp(2\|U\|_\oplus+2\|V\|_\oplus)
\|U-V\|_\infty.
\end{align*}

Putting the two bounds together, for every $\|E\|_\infty\le 1$,
\begin{align*}
\|Dp_\mc(U)[E]-Dp_\mc(V)[E]\|_1
\le
(1+4) \exp(2\|U\|_\oplus+2\|V\|_\oplus)
\|U-V\|_\infty.
\end{align*}
Taking the supremum over $\|E\|_\infty\le1$ completes the proof.
\end{proof}

\begin{corollary}[Uniform Fr\'echet differentiability for MCHF]\label{corollary:uniform_frechet_differentiability_mc}
For any constant $R\ge 0$,
\begin{align*}
\lim_{\|E\|_\infty\to 0}
\sup_{\|U\|_\oplus\le R}
\frac{
\|p_\mc(U+E)-p_\mc(U)-Dp_\mc(U)[E]\|_1
}{
\|E\|_\infty
}
=0.
\end{align*}
\end{corollary}

\begin{proof}
Fix $U\in L^\infty(\mu_\rf\otimes\mu_\rf)$. Since $t \mapsto p_\mc(U+tE)$ is $C^1$ from $\R$ to $L^1$ by Theorem \ref{theorem:frechet_derivative_mc},
the fundamental theorem of calculus in Banach spaces gives
\begin{align*}
p_\mc(U+E)-p_\mc(U)-Dp_\mc(U)[E]
=
\int_0^1
\{Dp_\mc(U+tE)-Dp_\mc(U)\}[E]dt .
\end{align*}
Taking the $L^1$-norm yields
\begin{align*}
\|p_\mc(U+E)-p_\mc(U)-Dp_\mc(U)[E]\|_1 \le 
\int_0^1
\|Dp_\mc(U+tE)[E]-Dp_\mc(U)[E]\|_1dt.
\end{align*}
Applying
Lemma \ref{lemma:derivative_lipschitz_continuous_mc} with $V=U+tE$, we obtain
\begin{align*}
\|Dp_\mc(U+tE)[E]-Dp_\mc(U)[E]\|_1
&\le
5 \|E\|_\infty
\exp\bigl(2\|U+tE\|_\oplus+2\|U\|_\oplus\bigr)
\|tE\|_\infty\\
&\le 5 t
\exp\bigl(4\|U\|_\oplus +2\|E\|_\infty\bigr)
\|E\|_\infty^2, 
\end{align*}
for all $t\in [0,1]$. Here, the second inequality follows from the triangle inequality for the seminorm $\|\cdot\|_\oplus$ and the property $\|\cdot\|_\oplus \le \|\cdot\|_\infty$. 
Hence, taking the integration $\int_0^1 dt$, using $\int_0^1 t dt = 1/2$, we get 
\begin{align*}
\|p_\mc(U+E)-p_\mc(U)-Dp_\mc(U)[E]\|_1
&\le(5/2)\cdot 
\exp\bigl(4\|U\|_\oplus+2\|E\|_\infty\bigr)
\|E\|_\infty^2 .
\end{align*}
Therefore, 
\begin{align*}
\sup_{\|U\|_\oplus\le R}
\frac{
\|p_\mc(U+E)-p_\mc(U)-Dp_\mc(U)[E]\|_1
}{
\|E\|_\infty
}
\le
\frac{5}{2}
\exp\bigl(4R+2\|E\|_\infty\bigr)
\|E\|_\infty .
\end{align*}
Letting $\|E\|_\infty\to0$ proves the claim.
\end{proof}

\subsubsection{NLHF}
\begin{lemma}\label{lemma:derivative_lipschitz_continuous_nl}
There exists a universal constant $c > 0$ such that for any $U, V\in L^\infty(\mu_\rf\otimes\mu_\rf)$, 
\begin{align*}
&\sup_{\|E\|_\infty\le 1}
\|Dp_\nl(U)[E]-Dp_\nl(V)[E]\|_1\le 
c (1 +\|U\|_\oplus^c + \|V\|_\oplus^c)\|U-V\|_\infty.
\end{align*}
\end{lemma}
\begin{proof}
The proof follows the same structure as that of Lemma \ref{lemma:derivative_lipschitz_continuous_mc}. 
Write
$p_U=p_{\nl}(U)$, $q_U=q_{\nl}(U)$ 
and define
$$
A_U
=
I+J_{p_U}U^\star J_{q_U}U,
\qquad
h_{U,E}
=
J_{p_U}E^\star q_U
-
J_{p_U}U^\star J_{q_U}E p_U .
$$
By Theorem \ref{theorem:frechet_derivative_nl}, 
$
Dp_{\nl}(U)[E]=A_U^{-1}h_{U,E},
$
so we decompose
\begin{align*}
Dp_{\nl}(U)[E]-Dp_{\nl}(V)[E]
&=
A_U^{-1}(A_V-A_U)A_V^{-1}h_{V,E}
+
A_U^{-1}(h_{U,E}-h_{V,E}),
\end{align*}
where we used the identity
$
A_U^{-1}-A_V^{-1}
=
A_U^{-1}(A_V-A_U)A_V^{-1}.
$

We now collect the bounds needed for the two terms. First, by Lemma \ref{lemma:LJ_p_infty_1_norm} and Lemma \ref{lemma:J_p_lip}, for any $K\in L^\infty(\mu_\rf\otimes \mu_\rf)$ and any $p, q\in \Delta^\infty$, we have 
\begin{align*}
  \forall h\in L^1, \quad \|J_p K h\|_1 &\le \|K h\|_\infty \le \|K\|_\infty \|h\|_1 \\
  \forall h\in L^1, \quad \|(J_p-J_q) K h\|_1 &\le 3 \|p-q\|_1 \|Kh\|_\infty \le 3 \|p-q\|_1  \|K\|_\infty \|h\|_1. 
\end{align*}
Moreover, if $h\in L_0^1$, with additive parts removed
as in the proof of Theorem \ref{theorem:stability_nl}, the above bounds hold with $\|K\|_\infty$ replaced by $\|K\|_\oplus$:
\begin{align*}
  \forall h\in L_0^1, \quad \|J_p K h\|_1 \le \|K\|_\oplus \|h\|_1, \quad \|(J_p-J_q) K h\|_1 \le 3 \|p-q\|_1 \|K\|_\oplus \|h\|_1,
\end{align*}
Second, by Theorem \ref{theorem:stability_nl}, noting that the same bound applied to the dual variable $\mu_\nl$ by replacing $U$ by $-U^\star$ in the proof, 
there exists a universal constant $c>0$ such that
$$
\|p_U-p_V\|_1+\|q_U-q_V\|_1
\le
c\bigl(1+\|U\|_\oplus^6+\|V\|_\oplus^6\bigr)\|U-V\|_\infty.
$$
Third, by the $L^1$-bound of the inverse operator from Lemma \ref{lemma:Schur_complement_l1_norm}, with additive parts removed
as in the proof of Theorem \ref{theorem:stability_nl}, 
$$
\forall h\in L_0^1, \quad
\|A_U^{-1}h\|_1
\le
c(1+\|U\|_\oplus^5)\|h\|_1,
\qquad
\|A_V^{-1}h\|_1
\le
c(1+\|V\|_\oplus^5)\|h\|_1. 
$$
Notice that $A_U$ maps $L^1_0$ into itself
because $J_{p_U}$ always outputs a mean-zero function. Also,
$h_{U,E}$ and $h_{V,E}$ belong to $L^1_0$, since they are in
the ranges of $J_{p_U}$ and $J_{p_V}$, respectively.

Given these bounds, we next claim that
$$
\|(A_U-A_V)\mid_{L_0^1}\|_{1\to 1}
\le
c\bigl(1+\|U\|_\oplus^c+\|V\|_\oplus^c\bigr)\|U-V\|_\infty .
$$
Indeed, for any $h\in L^1_0$,
$$
(A_U-A_V)h
=
J_{p_U}U^\star J_{q_U}Uh
-
J_{p_V}V^\star J_{q_V}Vh .
$$
Expanding this difference into the four natural terms and using the bounds mentioned above, we get
$$
\forall h\in L_0^1, \quad 
\|(A_U-A_V)h\|_1
\le
c\bigl(1+\|U\|_\oplus^c+\|V\|_\oplus^c\bigr)
\|U-V\|_\infty\|h\|_1. 
$$
Similarly,
$$
\|h_{V,E}\|_1
\le \|J_{p_V}E^\star q_V\|_1 +
\|J_{p_V}V^\star J_{q_V}E p_V\|_1\le  
(1+\|V\|_\oplus)\|E\|_\infty ,
$$
and
$$
\|h_{U,E}-h_{V,E}\|_1
\le
c\bigl(1+\|U\|_\oplus^c+\|V\|_\oplus^c\bigr)
\|U-V\|_\infty\|E\|_\infty .
$$
The second bound follows by expanding
\begin{align*}
h_{U,E}-h_{V,E}
&=
\{J_{p_U}E^\star q_U-J_{p_V}E^\star q_V\}
-
\{J_{p_U}U^\star J_{q_U}E p_U
-
J_{p_V}V^\star J_{q_V}E p_V\},
\end{align*}
and applying the same ingredients used to prove the bound on $\|(A_U-A_V)\mid_{L_0^1}\|_{1\to 1}$.

Combining these estimates, we obtain
\begin{align*}
&\|Dp_{\nl}(U)[E]-Dp_{\nl}(V)[E]\|_1\\
&\le
\|(A_U\mid_{L_0^1})^{-1}\|_{1\to 1} \Bigl(
\|(A_U-A_V)\mid_{L_0^1}\|_{1\to 1}
\|(A_V\mid_{L_0^1})^{-1}\|_{1\to 1} \|h_{V,E}\|_1  
+
\|h_{U,E}-h_{V,E}\|_1 \Bigr)\\
&\le
c\bigl(1+\|U\|_\oplus^c+\|V\|_\oplus^c\bigr)
\|U-V\|_\infty\|E\|_\infty.
\end{align*}
Taking the supremum over $\|E\|_\infty\le 1$ completes the proof for $D p_\nl$. 
\end{proof}

\begin{corollary}[Uniform Fr\'echet differentiability for NLHF]\label{corollary:uniform_frechet_differentiability_nl}
  For any fixed constant $R\ge 0$, 
$$
\lim_{\|E\|_\infty\to 0} \sup_{\|U\|_\oplus \le R}
\frac{
\|p_\nl(U+E)-p_\nl(U)-Dp_\nl(U)[E]\|_1}{\|E\|_\infty} 
= 0
$$
\end{corollary}
\begin{proof}
  This corollary follows by the exactly same argument as for the MCHF; apply the fundamental theorem of calculus with Lemma \ref{lemma:derivative_lipschitz_continuous_nl}. Thus, we omit the proof. 
\end{proof}

\section{Alignment dynamics of MCHF and NLHF}\label{sec:asymptotics_general_U}
For $*\in \{\mc, \nl, \rl\}$, let $p_{*} = \frac{d\mu_*}{d\mu_{\rf}}$ be the density with respect to $\mu_\rf$. 

\subsection{Proof of Proposition \ref{proposition:derivative_additive_antisymmetric}}
We prove a general version of Proposition \ref{proposition:derivative_additive_antisymmetric} for general base point and perturbation (not necessarily antisymmetric):
\begin{proposition}\label{proposition:derivative_additive}
  For any $g, f\in L^\infty$ and $E\in L^\infty(\mu_\rf\otimes \mu_\rf)$,
  \begin{align*}
    D p_\mc(g\oplus f)[E] &= J_{p_\rl(f)} E^\star p_{\rl}(f)\\
    D p_\nl(g\oplus f)[E] &= J_{p_\rl(f)} E^\star p_{\rl}(-g)
  \end{align*}
  where $J_p h = p\odot h - p \langle p, h \rangle$. 
\end{proposition}
Note that  Proposition \ref{proposition:derivative_additive_antisymmetric} immediately follows from Proposition \ref{proposition:derivative_additive}. Indeed, if $E$ is antisymmetric, we have that for any density $p\in \Delta^\infty$, 
$$
J_p E^\star p = p \odot E^\star p - p \langle p, E^\star p\rangle =  p \odot E^\star p 
$$
since $\langle p, E^\star p\rangle=0$ by the antisymmetry of $E$. Therefore, assuming the antisymmetry of $E$ and substituting $g=-f$ in Proposition \ref{proposition:derivative_additive}, we obtain 
$D p_*((-f)\oplus f) = p_\rf(f) \odot E^\star p_\rl(f)$ for each $*\in\{\mc, \nl\}$, 
thereby completing the proof of Proposition \ref{proposition:derivative_additive_antisymmetric}.

Below we prove Proposition \ref{proposition:derivative_additive}. 
Recall the derivative formula by Theorem \ref{theorem:frechet_derivative_mc}:
\begin{align*}
  &Dp_\mc(U)[E]\\
  &=
\Bigl((I-P_U^\star)\mid_{L_0^1}\Bigr)^{-1}
\biggl(
y\mapsto
\int \mu_\rf(dx)
p_\mc(U)(x)P_U(x,y)
\biggl[
E(x,y)
-
\int \mu_\rf(dy')P_U(x,y')E(x,y')
\biggr]
\biggr).
\end{align*}
When $U=g\oplus f$, 
\begin{align*}
  P_{g\oplus f}(x, y) = \frac{\exp(g(x) +  f(y))}{\int \mu_\rf(dy') \exp(g(x) + f(y'))} = \frac{\exp( f(y))}{\int \mu_\rf(dy') \exp( f(y'))} = p_\rl(f)(y)
\end{align*}
and $p_\mc(g\oplus f)=p_\rl(f)$.
Note that $P_{g\oplus f}^\star h= 0$ for any $h\in L_0^1$ since 
$$
P_{g\oplus f}^\star h(y) = \int \mu_\rf(dx) h(x) P_{g\oplus f}(x, y) = \Bigl(\int \mu_\rf(dx) h(x) \Bigr) p_\rl(f)(y) = 0. 
$$
Thus, we have
$$((I - P_{g\oplus f}^\star)\mid_{L_0^1})^{-1} = I.$$
On the other hand, 
\begin{align*}
&\int \mu_\rf(dx)
p_\mc(g\oplus f)(x)P_{g\oplus f}(x,y)
\biggl[
E(x,y)
-
\int \mu_\rf(dy')P_{g\oplus f}(x,y')E(x,y')
\biggr]\\
&=\int \mu_\rf(dx)
p_\rl(f)(x) p_\rl(f)(y)
\biggl[
E(x,y)
-
\int \mu_\rf(dy') p_\rl(f)(y')E(x,y')
\biggr]\\
 &= p_\rl(f)(y) \Bigl([E^\star p_\rl(f)](y) - \int\mu_\rf(dy')[E^\star p_\rl(f)](y') \Bigr)\\
 &= [J_{p_\rl(f)} E^\star p_\rl(f)](y)
\end{align*}
Therefore, we get 
$Dp_\mc(U)[E] = J_{p_\rl(f)} E^\star p_\rl(f)$, 
and the proof is complete for $\mc$. 

For the NLHF, recall the derivative formula from Theorem \ref{theorem:frechet_derivative_nl}:
$$
Dp_\nl(U)[E]
=
\bigl(I+J_{p(U)}U^\star J_{q(U)}U\bigr)^{-1}
\bigl(
J_{p(U)}E^\star q(U)
-
J_{p(U)}U^\star J_{q(U)}E p(U)
\bigr)
$$
where $p(U)=p_\nl(U)$ and $q(U)=q_\nl(U)$ on the RHS. 
Note $p_\nl(g\oplus f)=p_\rl(f)$ and $q_\nl(g\oplus f) =p_\rl(-g)$.  

Now, for any $h\in L_0^2 = \{h\in L^2: \int\mu_\rf(dy) h(y)=0\}$, 
\begin{align*}
((f\oplus g) h )(x) &= \int \mu_\rf(dy) (f\oplus g)(x, y) h(y)
= \int \mu_\rf(dy) g(y) h(y).
\end{align*}
Thus, $(f\oplus g) h$ is a constant. Since $J_{p(U)}$ is a mean-subtracting operator, we get $J_{p_0} (f\oplus g) h = 0$ and hence
\begin{align*}
  J_{p(U)} (f\oplus g) \mid_{L_0^2} = 0. 
\end{align*}
By the same argument, we also get 
$
  J_{q(U)} (g\oplus f) \mid_{L_0^2} = 0.
$
Since the image space of $J_{p(U)}$ and $J_{q(U)}$ are both $L_0^2$, 
when $U=g\oplus f$, we get 
\begin{align*}
\bigl(I+J_{p(U)}U^\star J_{q(U)}U\bigr)^{-1} = I, \qquad
  J_{p(U)}E^\star q(U)
-
J_{p(U)}U^\star J_{q(U)}E p(U) &=  J_{p(U)} E^\star q(U).
\end{align*}
Finally,  substituting $p_\nl(g\oplus f)=p_\rl(f)$ and $q_\nl(g\oplus f) =p_\rl(-g)$, we obtain
$Dp_\nl(g\oplus f)[E] = J_{p_\rl(f)} E^\star p_\rl(-g)$ and the proof is complete. 

\begin{remark}
Proposition \ref{proposition:derivative_additive} yields a generalization of Corollary \ref{corollary:asymptotics_equilibrium} to a general $U\in L^\infty(\mu_\rf\otimes \mu_\rf)$, not necessarily antisymmetric. Indeed, for any $L^2$ solution $(\hat{g}, \hat{f})\in \argmin_{g,f}\|U-g\oplus f\|_2$, noting that $\|U-\hat{g}\oplus \hat{f}\|_\infty \asymp \|U\|_\oplus$ from Proposition \ref{proposition:L2_estimate}, applying the uniform Fr\'echet differentiability results in
 Corollary \ref{corollary:uniform_frechet_differentiability_mc} and \ref{corollary:uniform_frechet_differentiability_nl}, together with the derivative formula in Proposition \ref{proposition:derivative_additive}, we obtain
\begin{align*}
        \|p_{\mc}(U) - p_{\rl}(\hat f) - J_{p_\rl(\hat f)} (U-\hat g\oplus \hat f)^\star p_{\rl}(\hat f)\|_{L^1(\mu_{\rf})} &= o(\|U\|_\oplus)\\
        \|p_{\nl}(U) - p_{\rl}(\hat f) - J_{p_\rl(\hat f)} (U-\hat g\oplus \hat f)^\star p_{\rl}(-\hat g)\|_{L^1(\mu_{\rf})} &= o(\|U\|_\oplus)
\end{align*}
for any sequence of $U\in L^\infty(\mu_\rf\otimes\mu_\rf)$ such that $\|U\|_\oplus \to 0$. 

Thus, MCHF and NLHF differ in their first order term by $p_\rl(\hat{f})$ and $p_\rl(-\hat{g})$. Notice that they coincide when $U$ is antisymmetric, since $\hat{g}=-\hat{f}$. This recovers Corollary \ref{corollary:asymptotics_equilibrium} in the antisymmetric case. 
\end{remark}

\subsection{Proof of Theorem \ref{theorem:asymptotics_iterations}}
Assume $U$ is antisymmetric. 
Let us show the claim for $t=1$ first. Noting that $d\mu_\rf/d\mu_\rf = \bm{1}$, using the notation in \Cref{app:frechet_differentiability}, 
  we can write $p^1_\mc= \frac{d\mu_\mc^1}{d\mu_\rf}$ and $p^1_\nl=\frac{d \mu_\nl^1}{d\mu_\rf}$ as 
  \begin{align*}
 p_\mc^1(U) = P_U^\star \bm{1} = \int \mu_\rf(dx)
\softmax(U(x,\cdot)), \quad p_\nl^1(U) = \softmax(U^\star \bm{1}).
  \end{align*}
  Then, by the exactly same argument we used in \Cref{app:frechet_differentiability} and \ref{app:uniform_frechet_differentiability}, we can show that for any constant $R\ge 0$, 
  \begin{align}\label{eq:unifrom_differentiability_t1}
         \lim_{\|E\|_\infty\to 0}
\sup_{\|U\|_\oplus\le R}
\frac{
\|p_*^1(U+E)-p_*^1(U)-Dp_*^1(U)[E]\|_1
}{
\|E\|_\infty
}
=0,
  \end{align}
where 
$$
D p_\mc^1(U)[E] = \int \mu_\rf(dx) J_{\softmax(U(x, \cdot))} E(x, \cdot), \quad D p_\nl^1(U)[E] = J_{\softmax(U^\star 1)} E^\star \bm1.
$$
Here, when $U=g\oplus f$, we have  $\softmax(U(x, \cdot))=\softmax(U^\star \bm1) = p_\rl(f)$ and hence 
 the Fr\'echet derivatives collapse into the following form for each $*\in \{\mc, \nl\}$:
$$
D p_*^1 (g\oplus f)[E] = J_{p_\rl(f)} E^\star \bm{1}.
$$
Now, substituting $E=U-(-\hat{f})\oplus \hat{f}$ and $(g, f)=(-\hat{f}, \hat{f})$ with $\hat{f}(\cdot)=\int \mu_\rf(dx) U(x, \cdot) = U^\star \bm{1}$, we have
$$
(U-(-\hat{f})\oplus \hat{f})^\star \bm1 = U^\star \bm1 - ((-\hat{f})\oplus \hat{f})^\star \bm1 = \hat{f} - \hat{f} = 0 
$$
where the second equation follows from 
$((-\hat{f})\oplus \hat{f})^\star \bm1 = \hat{f}(\cdot) - \langle\bm1, \hat{f}\rangle$ and $ \langle\bm1, \hat{f}\rangle = \langle\bm1, U^\star 1\rangle=0$ by the antisymmetry of $U$. Thus, when $U$ is antisymmetric, the derivative at $(-\hat{f})\oplus\hat{f}$ to the direction $U-(-\hat{f})\oplus \hat{f}$
is $0$:
$$
D p_*^1 ((-\hat{f})\oplus \hat{f})[U-(-\hat{f}\oplus \hat{f})] = J_{p_\rl(\hat{f})} (U-(-\hat{f})\oplus \hat{f})^\star \bm1  = J_{p_\rl(\hat{f})} 0 = 0. 
$$
Substituting this to \eqref{eq:unifrom_differentiability_t1}, using $\|U-(-\hat{f})\oplus \hat{f}\|_\infty \asymp \|U\|_\oplus$ from Proposition \ref{proposition:L2_estimate_antisymmetric} and $p_*^1((-\hat{f})\oplus \hat{f})=p_\rl(\hat{f})$, we obtain 
$$
\|p_*^1(U) - p_\rl(\hat{f}) - 0\|_1 = o(\|U\|_\oplus). 
$$
This completes the proof for $t=1$. 

For $\sup_{t\ge 2}$, we compare the iterates with their equilibrium. By Corollary \ref{corollary:stationary_dist} and Theorem \ref{theorem:convergence_nl_tv}, for each $*\in\{\mc,\nl\}$, we have
$$
\forall t\ge 1,\qquad
\|p_*^t(U)-p_*(U)\|_1
\le
\|U\|_\oplus^t \|\bm 1-p_*(U)\|_1
\le
2\|U\|_\oplus^t .
$$
Here the second inequality follows from the triangle inequality,
$
\|\bm 1-p_*(U)\|_1
\le
\|\bm 1\|_1+\|p_*(U)\|_1
=
2$. 
For MCHF, this estimate follows directly from Corollary \ref{corollary:stationary_dist}, since the contraction coefficient satisfies
$c(U)\le \|U\|_\oplus$. For NLHF, since $U$ is antisymmetric, the two equilibrium marginals coincide, that is, $\mu_\nl=\nu_\nl$. Moreover, our single iterate $p_\nl^t$ can be identified with the two-stage iterates $(\mu_t,\nu_t)$ in Theorem \ref{theorem:convergence_nl_tv} through
$
p_\nl^{2t}
=
\frac{d\mu_t}{d\mu_\rf}$ and $
p_\nl^{2t+1}
=
\frac{d\nu_t}{d\mu_\rf}.
$
Hence, the contraction estimate of Theorem \ref{theorem:convergence_nl_tv} applies to the present iteration as well.

Therefore, for any sequence of antisymmetric utilities \(U\) such that \(\|U\|_\oplus\to0\), 
$$
\sup_{t\ge 2}
\|p_*^t(U)-p_*(U)\|_1
=
O(\|U\|_\oplus^2)
=
o(\|U\|_\oplus).
$$
Combining this estimate with Corollary \ref{corollary:asymptotics_equilibrium} and applying the triangle inequality completes the proof for $\sup_{t\ge 2}$.

\section{Analysis of inference-time refinement via hitting times}\label{sec:inference_time_refinement}
We formulate inference-time refinement as the expected stopping time until the user is satisfied. 
Let us fix an antisymmetric utility function $U$ and a reference distribution $\mu_\rf$. For each $\gamma\ge 0$, define the Markov kernel
  $$
  \MP_\gamma(x, dy)
  =
  \frac{\exp(\gamma U(x, y)) \mu_\rf(dy)}
  {\int \exp(\gamma U(x, y'))\mu_\rf(dy')}.
  $$
  Now, for each measurable set $A$ such that $\mu_\rf(A)>0$, 
  let $T_\gamma(x, A)$ denote the expected hitting time from state $x$ to the set $A$ under the Markov kernel $\MP_\gamma$, with $T_\gamma(x, A)=0$ whenever $x\in A$. We further define
  $$
      T_\gamma(A) = \int \mu_\rf(dx) (1+T_\gamma(x, A)) = 1 + \int_{A^c} \mu_\rf(dx) T_\gamma(x, A)
  $$
  Thus, $T_\gamma(A)$ is the expected time to hit $A$ when the initial state is drawn from $\mu_\rf$, with one additional initial sampling step included.

  When $\gamma=0$, the Markov chain samples independently from $\mu_\rf$ at every step. Hence,  we have the explicit formula
  $$
      T_0(x, A)
      =
      \begin{cases}
        0, & x\in A,\\
        \frac{1}{\mu_\rf(A)}, & x\notin A,
      \end{cases}
      \qquad
      T_0(A)=\frac{1}{\mu_\rf(A)}. 
  $$
  Our goal is to characterize conditions on $U$, for any measurable set $A$ such that $\mu_\rf(A)>0$, under which there exists small $\gamma>0$ such that  
  $$
    T_\gamma(A) < T_0(A). 
  $$
  We address this question by calculating the derivatives of $T_\gamma$ at $\gamma=0$.

  \begin{theorem}\label{theorem:hitting_time}
    Assume that $U\in L^\infty(\mu_\rf\otimes \mu_\rf)$ is antisymmetric. 
    Fix a measurable set $A$ such that $\mu_\rf(A)>0$. 
    Then, $\gamma\mapsto T_\gamma(A)$ is $C^2$ on $\R$. In particular, the first derivative at $\gamma=0$ is given by
    $$
        \dot{T}_0(A)
        =
        -\frac{\mu_\rf(A^c)}{\mu_\rf(A)} \int \mu_\rf(dx) s(x, A), 
    $$
    where $s(x, A)$ is the average preference from $x$ to $A$:
    $$
      s(x, A) = \frac{1}{\mu_\rf(A)}\int_A \mu_\rf(dy) U(x, y).
    $$
    We further assume that $\int\mu_\rf(dy) U(x, y)=0$ for all $x$ (in which case $\dot{T}_0(A)=0$ for any $A$ such that $\mu_\rf(A)>0$). 
      Then, the second derivative at $\gamma=0$ can be written as 
      \begin{align*}
        \ddot{T}_0(A) = \frac{1}{\mu_\rf(A)} \int_{A^c} \mu_\rf(dx)\bigl(
       e(x, \mathcal{X}) -e(x, A) - 2 \mu_\rf(A)s(x, A)^2
        \bigr)
      \end{align*}
      where $e(x, B)$ for $B\in \{A, \mathcal{X}\}$ is the average energy defined as 
      $$
         e(x, B) = \frac{1}{\mu_\rf(B)}\int_B \mu_\rf(dy) U(x, y)^2.
      $$
  \end{theorem}
By the first order derivative formula, we can decrease the hitting time to $A$ whenever $\int \mu_\rf(dx) s(x, A)$ is positive. 
    This is intuitive: a positive value of this statistic means that the set $A$ is attractive on average. 

    For the second case, where $\int\mu_\rf(dy) U(x, y)=0$ for all $x$,  we have $\dot{T}_0(A)=0$ by the first order condition. In this case, NLHF equilibrium $\mu_\nl$ is reduced to $\mu_\rf$, meaning that marginal distributional alignment by NLHF alone cannot change hitting time to any set. In contrast, the Markov chain can still improve inference-time refinement through its transient dynamics. 
    
    Let us simplicity the second order condition. By Jensen's inequality, we have $\mu_\rf(A)s(x, A)^2\le e(x, A)$, and hence  $\ddot{T}_0(A)$ can be estimated as 
    \begin{align*}
      \frac{1}{\mu_\rf(A)} \int_{A^c}\mu_\rf(dx) \Bigl(-3 e(x, A) + e(x, \mathcal{X})\Bigr) \le \ddot{T}_0(A) \le       \frac{1}{\mu_\rf(A)} \int_{A^c}\mu_\rf(dx) \Bigl(-e(x, A) + e(x, \mathcal{X})\Bigr). 
    \end{align*}
    Therefore, the ratio 
        $$
        r(A) \equiv \frac{\int_{A^c} \mu_\rf(dx) e(x, A)}{\int_{A^c}\mu_\rf(dx) e(x, \mathcal{X})},
        $$
        which measures the \emph{relative energy} from $A^c$ to $A$, 
        characterizes the second order behavior in the sense that:
        \begin{align*}
       r(A)> 1 \Rightarrow \ddot{T}_0(A) < 0, \qquad  r(A)< 3^{-1}\Rightarrow \ddot{T}_0(A)>0
      \end{align*}

    Note that the second derivative formula can be further simplified when $\mu_\rf$ is a discrete measure and the target set $A$ is a singleton. 
    Let us consider the discrete case and take $\mu_\rf=\text{Unif}[n]$, and assume $\sum_y U(x, y)=0$ for all $x\in [n]$. Consider the hitting time to a single state, $A=\{y\}$. In this case, $\dot{T}_0(\{y\})=0$. For the second derivative, noting $s(x, \{y\}) = U(x, y)$ and $e(x, \{y\}) = U(x, y)^2$, we have 
    \begin{align*}
              \ddot{T}_0(\{y\}) &= \frac{1}{1/n}\sum_{x\ne y} \frac1n \Bigl(
          -U(x, y)^2 + \frac1n \sum_{y'}U(x, y')^2 - \frac2n U(x, y)^2
        \Bigr)\\
        &= -\Bigl(1+\frac3n\Bigr) \sum_x U(x, y)^2 + \frac1n\sum_{x,y'}U(x,y')^2 && \text{by $U(y, y)=0$ and $U(x, y)^2=U(y, x)^2$}
    \end{align*}
    Now that the insight is much clearer; we can decrease the hitting time to $y$ in second order whenever $\sum_x U(x, y)^2 \ge (n+3)^{-1}\|U\|_F^2$, that is, the energy sum at column $y$ is sufficiently larger than the total energy (the Frobenius norm of $U$). 

    The example of rock-paper-scissors game in \Cref{subsec:coupling_hitting_time}, corresponds to this case where $n=3$,  $\sum_x U(x, y)=0$ and $\sum_x U_{x,y}^2=2$ for all $y$. In this case, $\ddot{T}_0(\{y\}) = - (1+1) 2 + 2 < 0$, meaning that we can decrease the hitting time to any state for a finite $\gamma$.

\begin{remark}
In Theorem \ref{theorem:hitting_time}, we only state the second-derivative formula under the simplifying assumption that the row sum vanishes. This is mainly for readability: one can derive an explicit formula of the second derivative for general $U$. More generally, one can compute higher-order derivatives; the $r$-th derivative is an $r$-th order polynomial in $U$. However, the resulting expressions may be less interpretable for $r\ge 3$.
The current analysis also focuses on the local behavior around $\gamma=0$. At this point, we do not yet have a complete picture of the global behavior over $\gamma\in[0,\infty)$ (see \Cref{fig:global_local}). 
    \begin{figure}[htpb]
    \centering
    \includegraphics[width=0.68\linewidth]{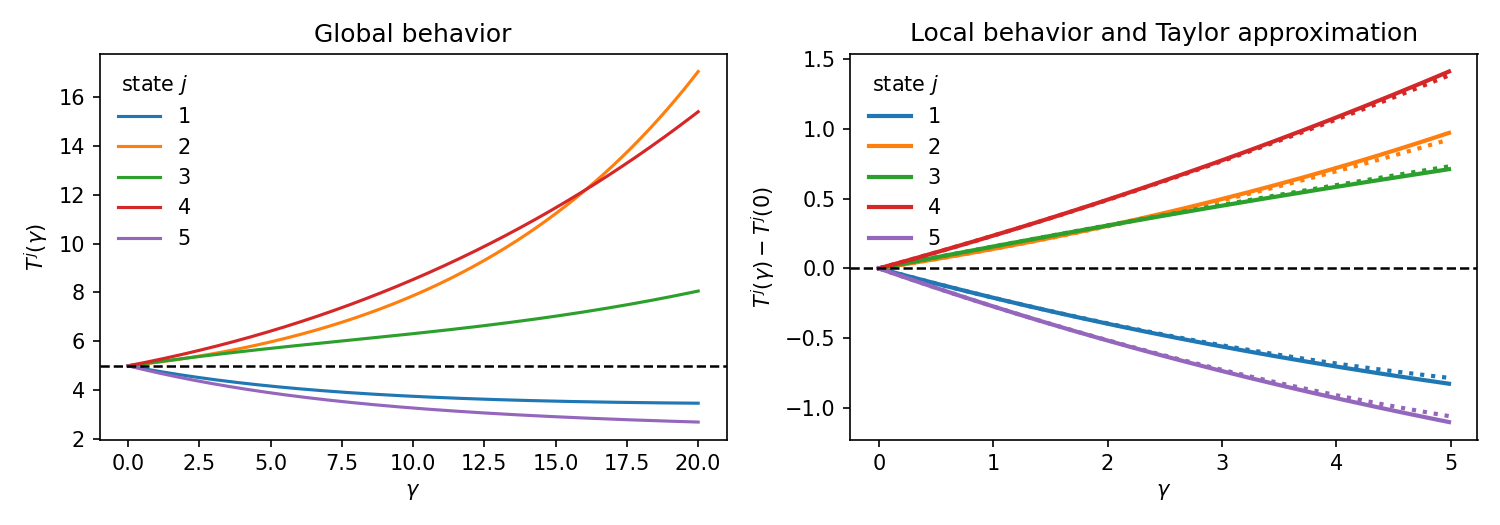}
    \includegraphics[width=0.3\linewidth]{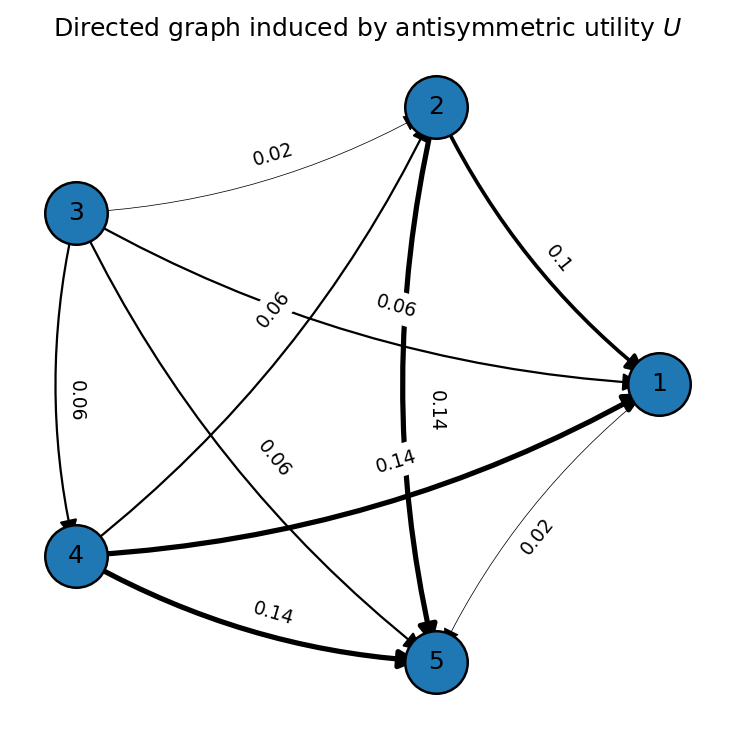}
    \caption{
    Numerical simulation of hitting time. 
    The left panel shows the global behavior of $T_\gamma(\{j\})$, while the middle panel shows the local difference $T_\gamma(\{j\})-T_0(\{j\})$ together with the second-order Taylor approximations at $\gamma=0$ shown as dotted curves. The right panel visualizes the antisymmetric utility $U$ as a directed weighted graph.
    }
    \label{fig:global_local}
  \end{figure}
  \end{remark}

\subsection{Proof of Theorem \ref{theorem:hitting_time}}
Fix measurable set $A$ such that $\mu_\rf(A)>0$. Write $T_\gamma(x)=T_\gamma(x,A)$ and $t_\gamma=T_\gamma(A)$ so that
$$
    t_\gamma=1+\int_{A^c}T_\gamma(x)\mu_\rf(dx).
$$
By the Markov property, for $x\in A^c$, $T_\gamma$ satisfies 
$$
    T_\gamma(x)
    =
    1+\int_{A^c}P_\gamma(x,y)T_\gamma(y)\mu_\rf(dy)\quad \text{where} \quad  P_\gamma(x,y)
    =
    \frac{\exp(\gamma U(x,y))}
    {\int \exp(\gamma U(x,z))\mu_\rf(dz)}.
$$
Define the subset $L^\infty(A^c, \mu_\rf)$ of $L^\infty(\mu_\rf)$ as 
$$L^\infty(A^c, \mu_\rf)=\{f\in L^\infty(\mu_\rf): f(x)=0\text{ for all $x\in A$}\}$$
and define the linear operator $P_\gamma$ on $L^\infty(A^c,\mu_\rf)$ as 
$
P_\gamma h(x)=\int_{A^c}P_\gamma(x,y)h(y)\mu_\rf(dy).
$
Then, 
$$
    T_\gamma-P_\gamma T_\gamma-\mathbf 1=0  \quad \text{on } L^\infty(A^c, \mu_\rf).
$$
Now we claim that the operator norm of $P_\gamma$ on $L^\infty(A^c, \mu_\rf)$ is strictly less than $1$. For all $h\in L^\infty(A^c, \mu_\rf)$, 
\begin{align*}
 \frac{\|P_\gamma h\|_\infty}{\|h\|_\infty}
    &\le
 \esssup_x \int_{A^c}P_\gamma(x,y)\mu_\rf(dy)= \esssup_x \frac{\int_{A^c} e^{\gamma U(x, y)}\mu_\rf(dy)}{\int_{A^c} e^{\gamma U(x, y)}\mu_\rf(dy) + \int_{A} e^{\gamma U(x, y)}\mu_\rf(dy)}.
\end{align*}
Combined with the minorization/maximization argument, 
\begin{align*}
  \essinf_x  \int_{A} e^{\gamma U(x, y)}\mu_\rf(dy) \ge \mu(A)  e^{-\gamma \|U\|_\infty}, \qquad 
  \esssup_x \int_{A^c} \exp^{\gamma U(x, y)} \mu_\rf(dy) &\le \mu(A^c)e^{\gamma \|U\|_\infty}, 
\end{align*}
and using the assumption $\mu(A)>0$, we obtain
\begin{align*}
     \sup_{h\in L^\infty(A^c, \mu_\rf)} \frac{\|P_\gamma h\|_\infty}{\|h\|_\infty} \le \frac{\mu(A^c)  e^{\gamma \|U\|_\infty}}{\mu(A^c)  e^{\gamma \|U\|_\infty} + \mu(A^c)e^{-\gamma \|U\|_\infty}} = \frac{\mu(A^c)}{\mu(A^c) + \mu(A) e^{-2\gamma\|U\|_\infty}} < 1, 
\end{align*}
Therefore, the linear operator $I-P_\gamma$ has the bounded inverse $\sum_{t=0}^\infty P_\gamma^t$ on $L^\infty(A^c, \mu_\rf)$. 

Since $\gamma\mapsto P_\gamma$ is $C^2$ as a bounded operator on $L^\infty(A^c,\mu_\rf)$, the implicit function theorem on the Banach space $L^\infty(A^c, \mu_\rf)$ implies that $\gamma\mapsto T_\gamma$ is $C^2$ from $\R$ to $L^\infty(A^c, \mu_\rf)$. Hence, $t_\gamma=1+\int_{A^c}T_\gamma(x)\mu_\rf(dx)$ is also $C^2$. 
Furthermore, differentiating
$T_\gamma-P_\gamma T_\gamma-\mathbf 1=0$
once and twice gives
$$
    (I-P_\gamma)\dot T_\gamma=\dot P_\gamma T_\gamma,
    \qquad
    (I-P_\gamma)\ddot T_\gamma
    =
    \ddot P_\gamma T_\gamma+2\dot P_\gamma\dot T_\gamma. 
$$

We first compute $\dot t_0=\int_{A^c}\dot{T}_0(x)\mu_\rf(dx)$. Note  $T_0(x)=\mu_\rf(A)^{-1}$ on $A^c$ and 
$$
P_0(x, y)=1, \quad  \dot P_0(x,y)
    =
    U(x,y)-\int U(x,z)\mu_\rf(dz).
$$
Thus, we have $P_0 \dot T_0(x) = \int_{A^c} \mu_\rf(dy) \dot{T}_0(y) = \dot{t}_0$ (constant). 
Therefore, integrating the first derivative equation at $\gamma=0$ over $A^c$ gives
$$
   (1-\mu_\rf(A^c)) \dot t_0
    =
    \int_{A^c}\dot P_0T_0(x)\mu_\rf(dx) =
    \frac{1}{\mu_\rf(A)}
    \int_{A^c}\int_{A^c}
    \biggl(
        U(x,y)-\int U(x,z)\mu_\rf(dz)
    \biggr)
    \mu_\rf(dy)\mu_\rf(dx).
$$
By antisymmetry,
$\int_{A^c}\int_{A^c}U(x,y)\mu_\rf(dy)\mu_\rf(dx)=0$ and hence 
\begin{align*}
    \mu_\rf(A) \dot t_0
    &=
    -\frac{\mu_\rf(A^c)}{\mu_\rf(A)}
    \int_{A^c}\int U(x,z)\mu_\rf(dz)\mu_\rf(dx) \\
    &=
    \frac{\mu_\rf(A^c)}{\mu_\rf(A)}
    \int_A\int U(x,z)\mu_\rf(dz)\mu_\rf(dx) \\
    &=
    -\mu_\rf(A^c) \int s(z,A)\mu_\rf(dz).
\end{align*}
Dividing by $\mu_\rf(A)>0$, we complete the proof of the first derivative. 

Next, we compute $\ddot t_0=\int_{A^c}\ddot{T}_0(x)\mu_\rf(dx)$ under the assumption that the row mean is zero, that is, 
$\int U(x,y)\mu_\rf(dy)=0$ for all $x$. 
Then $\dot P_0(x,y)=U(x,y)$ and $\dot t_0=0$. Hence, for $x\in A^c$,
$$
    \dot T_0(x)
    =
    \dot P_0T_0(x) 
    =
    \frac{1}{\mu_\rf(A)}
    \int_{A^c}U(x,y)\mu_\rf(dy)
    =
    -s(x,A).
$$
Integrating the second derivative equation at $\gamma=0$ over $A^c$ gives
$$
    \mu_\rf(A)\ddot t_0
    =
    \int_{A^c}\ddot P_0T_0(x)\mu_\rf(dx)
    +
    2\int_{A^c}\dot P_0\dot T_0(x)\mu_\rf(dx).
$$
The second term is
$$
    \int_{A^c}\dot P_0\dot T_0(x)\mu_\rf(dx)
    =
    -\int_{A^c}\int_{A^c}
    U(x,y)s(y,A)\mu_\rf(dy)\mu_\rf(dx) 
    =
    -\mu_\rf(A)\int_{A^c}s(y,A)^2\mu_\rf(dy),
$$
where we used antisymmetry and the row-mean-zero assumption. Also,
$$
    \ddot P_0(x,y)
    =
    U(x,y)^2-\int U(x,z)^2\mu_\rf(dz) =  U(x,y)^2-e(x,\mathcal{X})
$$
so, with $\int \mu_\rf(dy) \ddot P_0(x,y) = 0$, 
$$
    \ddot P_0T_0(x)
    =
    \frac{1}{\mu_\rf(A)}
    \int_{A^c}\ddot P_0(x,y)\mu_\rf(dy)  =  -\frac{1}{\mu_\rf(A)}
    \int_{A}\ddot P_0(x,y)\mu_\rf(dy) =
  -e(x,A) +  e(x,\mathcal{X}).
$$
Therefore, 
$$
    \mu_\rf(A)\ddot t_0
    =
    \int_{A^c}
    \bigl(e(x,\mathcal X)-e(x,A)\bigr)
    \mu_\rf(dx)
    -
    2\mu_\rf(A)\int_{A^c}s(x,A)^2\mu_\rf(dx). 
$$
Dividing by $\mu_\rf(A)$, we complete the proof. 

\end{document}